\documentclass[10pt,journal,compsoc]{IEEEtran}
\ifCLASSINFOpdf
  \usepackage[pdftex]{graphicx}
  \graphicspath{{./figures/}}
\else
\fi
\usepackage{amsmath}
\usepackage{amssymb}
\usepackage{amsthm}
\usepackage{amssymb}
\usepackage{subfigure}
\usepackage{wrapfig}
\usepackage{algorithm}
\usepackage{mathtools}
\usepackage{booktabs}
\usepackage{bm}
\usepackage{bbm}
\usepackage{multirow}
\usepackage[numbers]{natbib}

\usepackage{color}

\theoremstyle{plain}
\newtheorem{definition}{Definition}
\newtheorem{lemma}{Lemma}
\newtheorem{theorem}{Theorem}

\newtheorem{assumption}{Assumption}

\usepackage{algorithmic}
\usepackage{url}

\begin{document}
%
\title{Adjacency Constraint for Efficient Hierarchical Reinforcement Learning}
%
%
%
%

\author{Tianren~Zhang$^\dagger$,
        Shangqi~Guo$^{\dagger}$,
        Tian~Tan,
        Xiaolin~Hu,~\IEEEmembership{Senior Member,~IEEE},
        and Feng~Chen,~\IEEEmembership{Member,~IEEE}
\thanks{$\dagger$ indicates equal contribution.}
\IEEEcompsocitemizethanks{\IEEEcompsocthanksitem T. Zhang and F. Chen are with the Department of Automation, Tsinghua University, Beijing 100086, China, with the Beijing Innovation Center for Future Chip, Beijing 100086, China, and with the LSBDPA Beijing Key Laboratory, Beijing 100084, China
(e-mail: zhang-tr19@mails.tsinghua.edu.cn;  chenfeng@mail.tsinghua.edu.cn).\protect\\
\IEEEcompsocthanksitem S. Guo is with the Department of Automation, Tsinghua University, Beijing 100086, China (e-mail: shangqi\_guo@foxmail.com).\protect\\
\IEEEcompsocthanksitem T. Tan is with the Department of Civil and Environmental Engineering, Stanford University, Stanford CA 94305, USA (e-mail: tiantan@stanford.edu). \protect\\
\IEEEcompsocthanksitem X. Hu is with the Department of Computer Science and Technology,
Institute for Artificial Intelligence, Beijing National Research Center for
Information Science and Technology, State Key Laboratory of Intelligent
Technology and Systems, Tsinghua University, Beijing 100084, China (e-mail: xlhu@mail.tsinghua.edu.cn).}
\thanks{This work was supported in part by the National Natural Science Foundation of China under Grant 62176133 and 61836004, and in part by the Tsinghua-Guoqiang Research Program under Grant 2019GQG0006, and in part by the National Key Research and Development Program of China under Grant 2021ZD0200300.\protect\\
Corresponding authors: Shangqi Guo and Feng Chen.\vspace{1em}}\protect\\
\thanks{\textbf{© © 2022 IEEE. Personal use of this material is permitted. Permission from IEEE must be obtained for all other uses, in any current or future media, including reprinting/republishing this material for advertising or promotional purposes, creating new collective works, for resale or redistribution to servers or lists, or reuse of any copyrighted component of this work in other works.}}}

%
%

\markboth{IEEE TRANSACTIONS ON PATTERN ANALYSIS AND MACHINE INTELLIGENCE}%
{Shell \MakeLowercase{\textit{et al.}}: Bare Demo of IEEEtran.cls for Computer Society Journals}
%



\IEEEtitleabstractindextext{%
\begin{abstract}
Goal-conditioned Hierarchical Reinforcement Learning (HRL) is a promising approach for scaling up reinforcement learning (RL) techniques. However, it often suffers from training inefficiency as the action space of the high-level, i.e., the goal space, is large. Searching in a large goal space poses difficulty for both high-level subgoal generation and low-level policy learning. In this paper, we show that this problem can be effectively alleviated by restricting the high-level action space from the whole goal space to a $k$-step adjacent region of the current state using an adjacency constraint. We theoretically prove that in a deterministic Markov Decision Process (MDP), the proposed adjacency constraint preserves the optimal hierarchical policy, while in a stochastic MDP the adjacency constraint induces a bounded state-value suboptimality determined by the MDP's transition structure. We further show that this constraint can be practically implemented by training an adjacency network that can discriminate between adjacent and non-adjacent subgoals. Experimental results on discrete and continuous control tasks including challenging simulated robot locomotion and manipulation tasks show that incorporating the adjacency constraint significantly boosts the performance of state-of-the-art goal-conditioned HRL approaches.
\end{abstract}

}

\maketitle

\IEEEdisplaynontitleabstractindextext

%
\IEEEpeerreviewmaketitle

\IEEEraisesectionheading{\section{Introduction}\label{sec:intro}}

%
%
%
%

\IEEEPARstart{H}{ierarchical} reinforcement learning (HRL) has shown great potential in scaling up reinforcement learning (RL) methods to tackle large, temporally extended problems with long-term credit assignment and sparse rewards~\cite{sutton_between_1999,precup_temporal_2000,barto_recent_2003}. As one of the prevailing HRL paradigms, goal-conditioned HRL framework~\cite{dayan_feudal_1992,schmidhuber_planning_1993,kulkarni_hierarchical_2016,vezhnevets_feudal_2017,nachum_data-efficient_2018,levy_learning_2019}, which comprises a high-level policy that breaks the original task into a series of subgoals and a low-level policy that aims to reach those subgoals, has recently achieved significant success in video games~\cite{kulkarni_hierarchical_2016,vezhnevets_feudal_2017} and robotics~\cite{nachum_data-efficient_2018,nachum_near-optimal_2019,levy_learning_2019}. However, the effectiveness of goal-conditioned HRL relies on the acquisition of effective and semantically meaningful subgoals, which remains a key challenge.

As subgoals can be interpreted as high-level actions, it is feasible to directly train the high-level policy to generate subgoals using external rewards as supervision, which has been widely adopted in previous research~\cite{nachum_data-efficient_2018,nachum_near-optimal_2019,levy_learning_2019,kulkarni_hierarchical_2016,vezhnevets_feudal_2017}. Although these methods require little task-specific design, they often suffer from training inefficiency. This is because the action space of the high-level, i.e., the goal space, is often as large as the state space~\cite{nachum_near-optimal_2019,huang_mapping_2019,guo_state-temporal_2021}. Such a large action space leads to inefficient high-level exploration and a non-ignorable burden of both high-level and low-level value function approximation, thus often resulting in inefficient learning.

One effective way of handling large action spaces is action space reduction. However, it is difficult to perform action space reduction in general scenarios without additional information, since a restricted action set may not be expressive enough to express the optimal policy. There has been limited literature~\cite{zahavy_learn_2018,van_de_wiele_q-learning_2020,khetarpal_what_2020} studying action space reduction in RL, and to our knowledge, there is no prior work studying action space reduction in HRL, since the information loss in the goal space can lead to severe performance degradation~\cite{nachum_near-optimal_2019}.

In this paper, we present an optimality-preserving high-level action space reduction method for goal-conditioned HRL. Concretely, we show that the high-level action space can be restricted from the whole goal space to a $k$-step adjacent region centered at the current state. Our main intuition is depicted in Fig.~\ref{fig:motivation}: distant subgoals can be substituted by closer subgoals, as long as they drive the low-level to move towards the same ``direction''. Therefore, given the current state $s$ and the subgoal generation frequency $k$, the high-level only needs to explore in a subset of subgoals covering states that the low-level can possibly reach within $k$ steps. By reducing the action space of the high-level, the learning efficiency of both the high-level and the low-level can be improved: a considerably smaller action space compared with the raw state space relieves the burden of high-level exploration and both high-level and low-level value function approximation; in addition, adjacent subgoals provide a stronger learning signal for the low-level than non-adjacent subgoals when intrinsic rewards are sparse, as the agent can be intrinsically rewarded with a higher frequency for reaching these subgoals. Formally, we introduce a $k$-step adjacency constraint for high-level action space reduction, and further show that this constraint can be practically implemented by training an adjacency network that enables succinct judgment of the $k$-step adjacency between all states and subgoals. Theoretically, we prove that in a deterministic Markov Decision Process (MDP), the proposed adjacency constraint preserves the optimal hierarchical policy, while in a stochastic MDP the adjacency constraint induces a bounded state-value suboptimality determined by the MDP's transition structure.

\begin{figure}[t]
\centering
\includegraphics[width=0.75\linewidth]{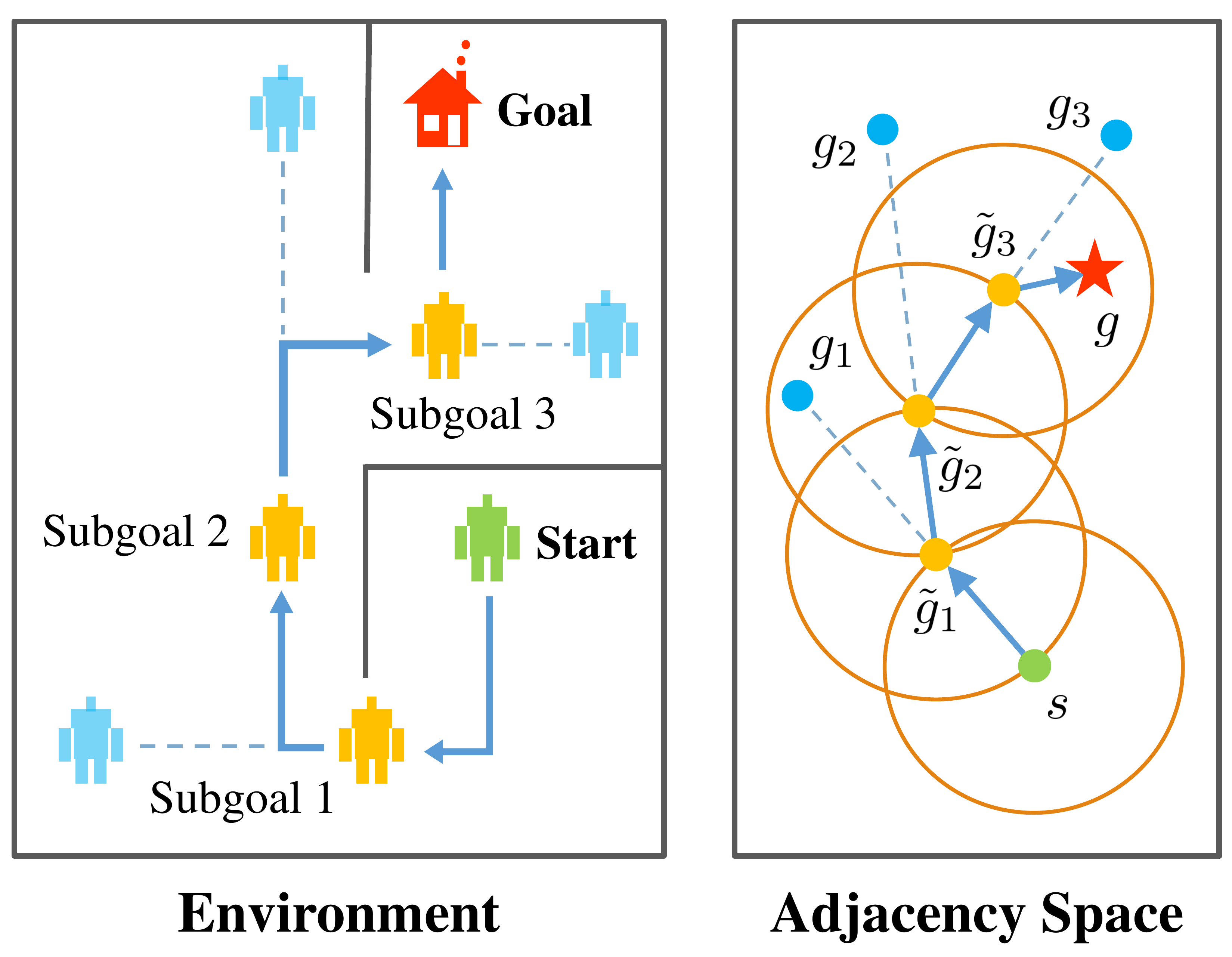}
\caption{A high-level illustration of our motivation: distant subgoals $g_1,\,g_2,\,g_3$ (blue) can be surrogated by closer subgoals $\tilde{g}_1,\,\tilde{g}_2,\,\tilde{g}_3$ (yellow) that fall into the $k$-step adjacent regions.
}
\label{fig:motivation}
\end{figure}

We benchmark our method on various kinds of tasks, including discrete control and planning tasks on grid worlds and challenging continuous control tasks based on the MuJoCo simulator~\cite{todorov_mujoco_2012}, which has been widely used in HRL literature~\cite{nachum_data-efficient_2018,levy_learning_2019,nachum_near-optimal_2019,florensa_stochastic_2017}. Experimental results exhibit the superiority of our method on both sample efficiency and asymptotic performance compared with state-of-the-art goal-conditioned HRL approaches, demonstrating the effectiveness of the proposed adjacency constraint.

A preliminary version of this manuscript was previously published at the conference of NeurIPS 2020~\cite{zhang_generating_2020}. Compared with the previous paper, in this work we further make the following contributions:
\begin{itemize}
  \item Theoretically, we analyze the impact on the suboptimality of the adjacency constraint in the context of \emph{stochastic} MDPs (in our previous work~\cite{zhang_generating_2020} a similar result is derived only in \emph{deterministic} MDPs), showing that in the general case the adjacency constraint induces a bounded suboptimality determined by the MDP's transition structure in terms of the state value of the optimal policy.
  \item Our analysis on stochastic MDPs shows that the distribution mismatch between the subgoal (state) distribution proposed by the high-level policy and the state distribution yielded by the subgoal-conditioned low-level policy plays a key role in upper-bounding the suboptimality induced by the hierarchical policy with goal-conditioning, which we believe is of independent interest.
  \item Empirically, we apply our method to additional robot control tasks, including a quadrupedal robot locomotion task that simultaneously involves locomotion and object manipulation and two robot arm manipulation tasks with sparse rewards. Experimental results on these tasks further exhibit the efficacy of our method. Also, we provide additional subgoal generation visualization to illustrate the effect of the adjacency constraint.
\end{itemize}

The rest of this paper is organized as follows: in Section~\ref{sec:pre}, we introduce the preliminaries on MDPs and HRL; in Section~\ref{sec:theory}, we formalize the proposed adjacency constraint and analyze the suboptimality induced by the constraint, respectively in deterministic and stochastic settings; in Section~\ref{sec:method}, we detail our practical implementation of the adjacency constraint and the overall HRL algorithm; in Section~\ref{sec:experiment}, we show our main experimental results and empirical analyses; in Section~\ref{sec:rel}, we introduce the related work of this paper; in Section~\ref{sec:con}, we discuss potential implications and conclude the paper.

\section{Preliminaries}
\label{sec:pre}

We consider a finite-horizon, goal-conditioned MDP defined as a tuple $\langle\mathcal{S},\mathcal{G},\mathcal{A},P,R,\gamma\rangle$, where $\mathcal{S}$ is a state space, $\mathcal{G}$ is a goal space, $\mathcal{A}$ is an action set, $P(s'\,|\, s, a)$ is the one-step state transition probability, $R(r\,|\, s, a, s')$ is the one-step reward, and $\gamma \in [0,1)$ is a discount factor. We define the expected reward function as $r(s,a,s') \vcentcolon= \mathbb{E} [R(\cdot\,|\,s, a, s')]$. We assume that the reward function $R$ is bounded within the range $[0,R_\mathrm{max}]$ with $R_\mathrm{max}\in\mathbb{R}_+$, which is a common assumption in the theoretical RL literature. We assume communicating MDPs~\cite{puterman_markov_1994}, i.e., all MDP states are strongly connected, which is widely used in the literature~\cite{jaksch_near-optimal_2010,agrawal_optimistic_2017,fruit_exploration-exploitation_2017,abbasi-yadkori_politex_2019,wei_learning_2021} and is natural in many reinforcement learning applications such as video games and robotics~\cite{lillicrap_continuous_2016,chen_understanding_2022}. This assumption is also necessary in our context for ensuring that all subgoals in the goal space are reachable.

Following prior works~\cite{kulkarni_hierarchical_2016,vezhnevets_feudal_2017,nachum_data-efficient_2018}, we consider a hierarchical framework comprising two levels of policies: a high-level policy $\pi_\mathrm{hi}(g\,|\, s)$ and a low-level policy $\pi_{\mathrm{lo}}(a\,|\, s,g)$, as shown in Fig.~\ref{fig:method}. We assume that the high-level and the low-level policies are parameterized by two function approximators, e.g., neural networks, with parameters $\theta_\mathrm{hi}$ and $\theta_\mathrm{lo}$ respectively. The high-level policy aims to maximize the external reward and generates a high-level action, i.e., a subgoal $g_t\sim\pi_\mathrm{hi}(g\,|\, s_t)\in\mathcal{G}$ every $k$ time steps when $t\equiv 0\,(\mathrm{mod}\ k)$, where $k>1$ is a pre-determined hyperparameter. The high-level policy modulates the behavior of the low-level policy by intrinsically rewarding the low-level for reaching these subgoals. The low-level aims to maximize the intrinsic reward provided by the high-level, and performs a primary action $a_t\sim\pi_\mathrm{lo}(a\,|\, s_t,g_t)\in\mathcal{A}$ at every time step. Following prior methods~\cite{nachum_data-efficient_2018,andrychowicz_hindsight_2017}, we consider a pre-defined goal space $\mathcal{G}$ with a mapping function $\varphi:\mathcal{S}\rightarrow\mathcal{G}$ and an inverse mapping function $\varphi^{-1}:\mathcal{G}\rightarrow\mathcal{S}$. When $t\not\equiv 0\,(\mathrm{mod}\ k)$, a pre-defined goal transition process $g_t = h(g_{t-1},s_{t-1},s_t)$ is utilized. We adopt the setting of directional subgoal that represents the difference between the desired state and the current state~\cite{vezhnevets_feudal_2017,nachum_data-efficient_2018}, i.e., the low-level agent is supposed to perform atomic actions to reach the state $s_{t+k}$ that is similar to $s_t + \varphi^{-1}(g_t)$, corresponding to the goal transition function $h(g_{t-1},s_{t-1},s_t) = g_{t-1} + \varphi(s_{t-1}) - \varphi(s_t)$. The reward of the high-level is given by
\begin{equation}
\begin{aligned}
r_{kt}^\mathrm{hi} \sim&\; R_\mathrm{hi}(s_{kt}, g_{kt}, \pi_\mathrm{lo})\\
\vcentcolon=&\; \sum_{i=kt}^{kt+k-1} R(r_{i}\,|\,s_i,a_i, s_{i+1}), \quad t = 0,\,1,\,\cdots,
\label{equ:r_high}
\end{aligned}
\end{equation}
which is the accumulation of the external reward in the time interval $[kt,kt+k-1]$. Note that the state\vspace{0.2em} value function of the high-level policy is then bounded by $\left[0, \frac{k R_\mathrm{max}}{1-\gamma}\right]$\vspace{0.2em} for high-level action frequency $k$ and discounting $\gamma\in[0,1)$, where the maximum is reached when the agent receives the maximal one-step reward $R_\mathrm{max}$ at every step.

While the high-level controller is motivated by the environmental reward, the low-level controller has no direct access to this external reward. Instead, the low-level is supervised by the intrinsic reward that describes subgoal-reaching performance, defined as $r_t^\mathrm{lo} \vcentcolon= -D\left(g_t, \varphi(s_{t+1})\right)$, where $D$ is a binary or continuous distance function. In practice, we employ Euclidean distance as $D$.

The goal-conditioned HRL framework enables us to train high-level and low-level policies concurrently in an end-to-end fashion. However, it often suffers from training inefficiency due to the unconstrained subgoal generation process, as we have mentioned in Section~\ref{sec:intro}. In the following section, we will introduce the $k$-step adjacency constraint to mitigate this problem.

\begin{figure}[t]
\centering
\includegraphics[width=0.8\linewidth]{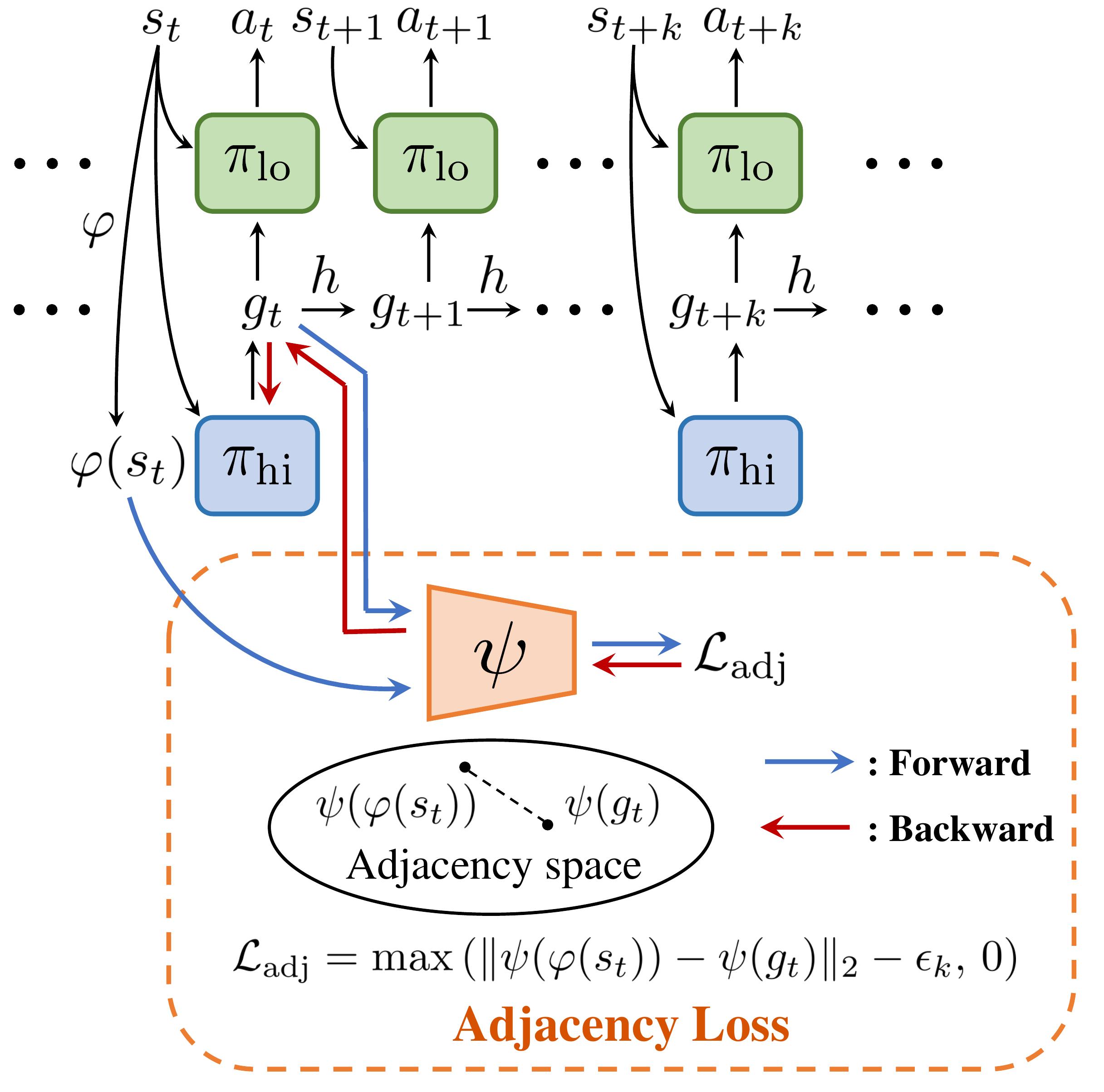}
\caption{The goal-conditioned HRL framework combined with the proposed $k$-step adjacency constraint, which is implemented by the adjacency network $\psi$ (dashed orange box).}
\label{fig:method}
\end{figure}

\section{Formulation and Theoretical Analysis}
\label{sec:theory}

In this section, we formalize the $k$-step adjacency constraint which has been intuitively explained in Section~\ref{sec:intro}. We will also provide our theoretical results, showing that the optimality can be fully (deterministic MDPs) or approximately preserved (stochastic MDPs) when learning a high-level policy with the adjacency constraint. We begin by introducing a distance measure that can decide whether a state is ``close'' to another state. In this regard, common distance functions such as Euclidean distance are not suitable, as they cannot reveal the transition structure of the MDP. Instead, we introduce \emph{shortest transition distance}, which equals the minimum number of steps required to reach a target state from a start state. In stochastic MDPs, the number of steps required is not a fixed number, but a random variable of which the distribution is conditioned on a specific policy. In this case, we resort to the notion of \emph{first hit time} from stochastic processes and define the shortest transition distance by minimizing the expected first hit time over all possible policies. 

\begin{definition}[Shortest transition distance]
Let $s,\,s'\in\mathcal{S}$. The shortest transition distance from $s$ to $s'$ is defined as:
\begin{equation}
\begin{aligned}
d_{\mathrm{st}}(s,s') \vcentcolon=&\, \inf_{\pi\in\Pi}\mathbb{E}[\mathcal{T}_{s,s'} \,|\, \pi] \\
=&\, \inf_{\pi\in\Pi} \sum_{t=0}^\infty\, t P(\mathcal{T}_{s,s'}=t\,|\,\pi),
\label{equ:dist}
\end{aligned}
\end{equation}
where $\Pi$ is set of all possible policies and $\mathcal{T}_{s,s'}$ denotes the first hit time from $s$ to $s'$, which is a random variable defined by
\begin{equation}
\mathcal{T}_{s,s'} \vcentcolon=\inf{\{t\in\mathbb{N}\mid s_t = s', s_0 = s\}}.
\end{equation}
Note that $\mathcal{T}_{s,s'}$ also depends on the policy.
\label{def:distance}
\end{definition}
The shortest transition distance is determined by a policy that connects states $s_1$ and $s_2$ in the most efficient way, which has also been studied by several prior works~\cite{florensa_self-supervised_2019,eysenbach_search_2019}. This policy is optimal in the sense that it requires the minimum number of steps to reach state $s_2$ from state $s_1$. Compared with the dynamical distance~\cite{hartikainen_dynamical_2020}, our definition here does not rely on a specific non-optimal policy. Also, we do not assume that the environment is reversible, i.e., $d_{\mathrm{st}}(s_1, s_2) = d_{\mathrm{st}}(s_2, s_1)$ may not hold for all pairs of states. Therefore, the shortest transition distance is a quasi-metric as it does not satisfy the symmetry condition. However, this limitation does not affect the following analysis as we only need to consider the transition from the start state to the goal state without the reversed transition. 

Given the definition of the shortest transition distance, we now formulate the property of an optimal (deterministic) goal-conditioned policy $\pi^*:\mathcal{S}\times\mathcal{G}\rightarrow\mathcal{A}$~\cite{schaul_universal_2015}. For every $s\in\mathcal{S}$ and $g\in\mathcal{G}$, we have:

\begin{equation}
\pi^*(s,g) \in \underset{a\in\mathcal{A}}{\arg\min} \sum_{s'\in\mathcal{S}}P(s'\,|\, s,a)\,d_{\mathrm{st}}\left(s',\varphi^{-1}(g)\right).
\label{equ:goal_conditioned}
\end{equation}

We then consider a goal-conditioned HRL framework with high-level action frequency $k$. Different from a flat goal-conditioned policy, in this setting, the low-level policy is required to reach the subgoals with $k$ limited steps. As a result, only a subset of the original states can be reliably reached even with an optimal goal-conditioned policy. We introduce the notion of \emph{$k$-step adjacent region} to describe the set of subgoals mapped from this ``adjacent'' subset of states. 

\begin{definition}[Average $k$-step adjacent region]
Let $s\in\mathcal{S}$. The average $k$-step adjacent region of $s$ is defined as:
\begin{equation}
\mathcal{G}_\mathrm{A}(s,k) \vcentcolon= \{ g\in \mathcal{G} \,|\, d_{\mathrm{st}}\left(s, \varphi^{-1}(g)\right)\le k \}.
\label{equ:adjacent_region}
\end{equation}
\end{definition}

In deterministic MDPs, average $k$-step adjacent region describes the goal subspace that incorporates all subgoals that has a shortest transition distance upper-bounded by $k$ from the starting state. Since the transitions between states and subgoals are deterministic, sampling subgoals from this subspace suffices for the high-level policy because all other subgoals can only be reached with probability zero due to the step limit. However, in stochastic MDPs where the transitions between states are stochastic, this region may exclude some distant subgoals that can also be reached in $k$ steps with a probability larger than zero, albeit reaching them is a rather rare event. Hence, for general cases we extend the notion of average $k$-step adjacent region as follows: 

\begin{definition}[Maximal $k$-step adjacent region]
\label{def:adjacent_region_max}
Let $s\in\mathcal{S}$. The maximal $k$-step adjacent region of $s$ is defined as:
\begin{equation}
\mathcal{G}_\mathrm{AM}(s,k) \vcentcolon= \left\{ g\in \mathcal{G} \;\middle|\; \sup_{\pi\in\Pi}P\left(\mathcal{T}_{s,\varphi^{-1}(g)} \le k \,|\, \pi \right) > 0 \right\}.
\label{equ:adjacent_region_max}
\end{equation}
\end{definition}
Note that the definition of maximal $k$-step adjacent region $\mathcal{G}_\mathrm{AM}$ generalizes the notion of average $k$-step adjacent region $\mathcal{G}_\mathrm{A}$: in stochastic MDPs, the maximal $k$-step adjacent region $\mathcal{G}_\mathrm{AM}$ includes all subgoals that can be possibly reached from $s$ within $k$ steps, while some of them may not be reached within an average number of steps $k$, thus falling out of the average $k$-step adjacent region $\mathcal{G}_\mathrm{A}$. Also, it is easy to verify that in deterministic MDPs, for all $s\in\mathcal{S}$ and $k\in\mathbb{N}_+$, we have $\mathcal{G}_\mathrm{AM}(s,k)=\mathcal{G}_\mathrm{A}(s,k)$.

Leveraging the tools defined above, we formulate the high-level objective incorporating this $k$-step adjacency constraint as:
\begin{equation}
  {\begin{aligned}
  \underset{\theta_\mathrm{hi}}{\mathrm{max}}\quad &\mathbb{E}_{\pi_\mathrm{hi}^{\theta_\mathrm{hi}}}\sum_{t=0}^{T-1}\gamma^{t} r_{kt}^\mathrm{hi}  \\
  \mathrm{subject\ to}\quad & g_{kt} \in \mathcal{G}_\mathrm{AM}(s_{kt}, k),\quad t = 0,\,1,\,\cdots,\,T-1
  \end{aligned}
  },
\label{equ:formulation}
\end{equation}
where $r_{kt}^\mathrm{hi}$ is the high-level reward defined by Equation~\eqref{equ:r_high} and $g_{kt}\sim\pi_\mathrm{hi}(g\,|\, s_{kt})$.

In the sequel, we will analyze the impact on the optimality of the $k$-step adjacency constraint formalized above. For ease of exposition, we will first focus on the deterministic case and then the general stochastic case. We hope that by first analyzing the deterministic case separately, we can help our readers build more intuition on our method before we delve deeper into the more complicated stochastic case.

\subsection{Analysis on Deterministic MDPs}
\label{sec:theory_det}

We begin by considering a rather simple case of deterministic MDPs, i.e., MDPs with deterministic transition functions. In this case, given a subgoal and a deterministic low-level policy, the agent's state trajectory is fixed. Then, by harnessing the property of $\pi^*$, we can show that in deterministic MDPs, given an optimal low-level policy $\pi^{*}_\mathrm{lo}=\pi^*$, subgoals that fall in the $k$-step adjacent region of the current state can ``represent'' all optimal subgoals in the whole goal space in terms of the induced $k$-step low-level action sequence. We summarize this result in the following lemma.

\begin{lemma}
Let $s\in\mathcal{S},\,g\in\mathcal{G}$ and $\pi^*$ an optimal goal-conditioned policy defined by Equation~\eqref{equ:goal_conditioned}. Under the assumption that the MDP is deterministic, for all $k\in\mathbb{N}_+$ satisfying $k \le d_{\mathrm{st}}(s,\varphi^{-1}(g))$, there exists a surrogate goal $\tilde{g}$ such that:
\label{theo:low}
\begin{equation}
\begin{aligned}
&\tilde{g} \in \mathcal{G}_\mathrm{A}(s,k), \\
&\pi^*(s_i, \tilde{g}) = \pi^*(s_i, g),\quad \forall s_i\in \tau\, (i\ne k), 
\end{aligned}
\end{equation}
where $\tau \vcentcolon= (s_0,s_1,\cdots,s_k)$ is the $k$-step state trajectory starting from state $s_0 = s$ under $\pi^*$ and $g$. 
\end{lemma}
\begin{proof}
See Appendix~\ref{app:proof1}.
\end{proof}

Lemma~\ref{theo:low} suggests that the $k$-step low-level action sequence generated by an optimal low-level policy conditioned on a distant subgoal can be induced using a subgoal that is closer. Naturally, we can generalize this result to a two-level goal-conditioned HRL framework, where the low-level is actuated not by a single subgoal, but by a subgoal sequence produced by the high-level policy.

\begin{theorem}
Given high-level action frequency $k$ and high-level planning horizon $T$, for $s\in\mathcal{S}$, let $\rho^* = (g_0,g_k,\cdots,g_{(T-1)k})$ be the high-level subgoal trajectory starting from state $s_0=s$ under an optimal high-level policy $\pi^{*}_\mathrm{hi}$. Also, let $\tau^* = (s_0,s_k,s_{2k},\cdots,s_{Tk})$ be the high-level state trajectory under $\rho^*$ and an optimal low-level policy $\pi^{*}_\mathrm{lo}$ as defined by Equation~\eqref{equ:goal_conditioned}. Then, under the assumption that the MDP is deterministic, there exists a surrogate subgoal trajectory $\tilde{\rho}^* = (\tilde{g}_0,\tilde{g}_k,\cdots,\tilde{g}_{(T-1)k})$ such that:
\begin{equation}
\begin{aligned}
&\tilde{g}_{kt} \in \mathcal{G}_\mathrm{A}({s_{kt},k}), \\
&Q^*(s_{kt},\tilde{g}_{kt}) = Q^*(s_{kt},g_{kt}),\quad t=0,\,1,\,\cdots,\,T-1,
\end{aligned}
\end{equation}
where $Q^*$ is the optimal high-level state-action value function. 
\label{theo:high}
\end{theorem}
\begin{proof}
See Appendix~\ref{app:proof2}.
\end{proof}

Theorem~\ref{theo:high} shows that when the MDP is deterministic, we can constrain the high-level action space to state-wise $k$-step adjacent regions without the loss of optimality, which matches our intuition as presented in Section~\ref{sec:intro}.

\subsection{Analysis on Stochastic MDPs}
\label{sec:theory_sto}

In this section, we extend our theoretical results from deterministic MDPs to stochastic MDPs. Before the concrete analysis, we first provide some explanations to help build the intuition on the reason that the result in deterministic MDPs does not apply to stochastic MDPs directly. In Lemma~\ref{theo:low}, we have shown that for each non-adjacent subgoal, there exists a surrogate adjacent subgoal given that the low-level goal-conditioned policy is optimal. Intuitively, this is because an optimal low-level policy can only reach an intermediate state in the trajectory to the non-adjacent subgoal due to the step limit. By the construction of setting the surrogate subgoal as the subgoal defined by the reached state, it is clear that both the original subgoal and the surrogate subgoal induce the same $k$-step state trajectory since the dynamics of the MDP is deterministic. However, in the sequel, we will show that such construction does not preserve the optimality in stochastic MDPs since an adjacent subgoal may still not be reached within $k$ steps when state transitions are non-deterministic, even with an optimal low-level policy. Therefore, unlike the deterministic case, in the stochastic case the adjacency constraint does not fully preserve the optimality of the original high-level action space (though we will show that the suboptimality can be upper-bounded).

We now begin our analysis. Our general idea is to compare the state value induced by an optimal hierarchical policy without the $k$-step adjacency constraint and the optimal policy after the constraint is imposed. We will show that if we allow the high-level policy to generate non-deterministic subgoals (i.e., subgoal distributions), then there exists a subgoal distribution that induces a bounded state value sub-optimality while satisfying the $k$-step adjacency constraint. First, we note that for every state $s\in\mathcal{S}$, given a high-level policy $\pi_\mathrm{hi}$ and a low-level policy $\pi_\mathrm{lo}$, the high-level state value function under $\pi_\mathrm{hi}$ can be written as:
\begin{equation}
\begin{aligned}
V^{\pi_\mathrm{hi}, \pi_\mathrm{lo}}(s) = \sum_{g\in\mathcal{G}} \;&\pi_\mathrm{hi}(g\,|\, s)\bigg(r_\mathrm{hi}(s,g, \pi_\mathrm{lo}) \,+ \\
& \gamma\sum_{s'\in\mathcal{S}} P^k(s'\,|\, s,g, \pi_\mathrm{lo}) V^{\pi_\mathrm{hi}}(s')\bigg),
\end{aligned}
\end{equation}
where $r_\mathrm{hi}(s,g, \pi_\mathrm{lo}) \vcentcolon= \mathbb{E}[R_\mathrm{hi}(s, g, \pi_\mathrm{lo})]$ defines the expected $k$-step high-level reward, and $P^k(s'\,|\, s,g, \pi_\mathrm{lo})$ denotes the $k$-step transition probability. Similar to the deterministic case, here we are interested in analyzing the suboptimality when the low-level policy $\pi_\mathrm{lo}$ is an optimal goal-conditioned policy $\pi_\mathrm{lo}^*$ as defined by Equation~\eqref{equ:goal_conditioned}. For brevity, in what follows we use $r_\mathrm{hi}(s, g)$ and $P^k(s'\,|\,s,g)$ as shorthand for $r_\mathrm{hi}(s, g, \pi_\mathrm{lo}^*)$ and $P^k(s'\,|\,s, g, \pi_\mathrm{lo}^*)$ respectively, and write the high-level state-value function for $\pi_\mathrm{hi}$ and an optimal low-level policy $\pi_\mathrm{lo}^*$ as:
\begin{equation}
\begin{aligned}
V^{\pi_\mathrm{hi}}(s) \vcentcolon=&\, V^{\pi_\mathrm{hi}, \pi_\mathrm{lo}^*}(s)\\
=&\,\sum_{g\in\mathcal{G}} \;\pi_\mathrm{hi}(g\,|\, s)\bigg(r_\mathrm{hi}(s,g) \,+\\
&\quad\gamma\sum_{s'\in\mathcal{S}} P^k(s'\,|\, s,g) V^{\pi_\mathrm{hi}}(s')\bigg).
\label{eq:value}
\end{aligned}
\end{equation}
Note that for the high-level policy, the subgoal $g$ represents the desired state that it would like to reach, while the intermediate low-level state and action details are inaccessible. Therefore, given a fixed low-level policy, it is natural to assume that the high-level reward only depends on the state where the agent starts and the state where the agent arrives, which we formalize as follows:
\begin{assumption}
For all $s\in\mathcal{S}$ and $g\in\mathcal{G}$, the expected $k$-step high-level reward can be written as the following form:
\begin{equation}
r_\mathrm{hi}(s, g) = \sum_{s'\in\mathcal{S}}P^k(s'\,|\, s,g)\,\widetilde{r}_\mathrm{hi}(s,s'),
\label{eq:high_r}
\end{equation}
where $\widetilde{r}_\mathrm{hi}:\mathcal{S}\times\mathcal{S}\rightarrow[0, kR_\mathrm{max}]$ is a $k$-step accumulated reward function defined only on starting and ending states of the $k$-step transition.
\end{assumption}
As we have mentioned above, this assumption is natural in the setting of goal-conditioned hierarchical reinforcement learning, since it is well-aligned with the spirit of subgoal-based task decomposition. Meanwhile, this assumption benefits further analysis on the state-value functions of different high-level policies by transforming the difference between immediate rewards to the difference between state transitions, as we will detail in the following. First, plugging Equation~\eqref{eq:high_r} into Equation~\eqref{eq:value} gives
\begin{equation}
\begin{aligned}
&V^{\pi_\mathrm{hi}}(s) \\
&= \sum_{g\in\mathcal{G}} \pi_\mathrm{hi}(g\,|\, s) \sum_{s'\in\mathcal{S}} P^k(s'\,|\, s,g) \left(\widetilde{r}_\mathrm{hi}(s,s') + \gamma V^{\pi_\mathrm{hi}}(s')\right).
\label{eq:value_simple}
\end{aligned}
\end{equation}
We then consider an optimal high-level policy $\pi_\mathrm{hi}^*$, which generates a subgoal distribution for state $s$. Each subgoal $g\in\mathcal{G}$ in this distribution can be mapped to a $k$-step transition $P^k(s'\,|\, s,g)$. According to Definition~\ref{def:adjacent_region_max}, this transition will result in a conditional state distribution defined on the maximal $k$-step adjacent region $\mathcal{G}_\mathrm{AM}$ of $s$, since the probability of arriving at other states from $s$ within $k$ steps is zero. Our key insight is that this state distribution \emph{itself} can be used as the subgoal distribution generated by an adjacency-constrained high-level policy. That is, we construct our adjacency-constrained high-level policy $\pi_\mathrm{hi}^\mathrm{adj}$ by setting
\begin{equation}
\pi_\mathrm{hi}^\mathrm{adj}(g\,|\, s) \vcentcolon= \sum_{g'\in\mathcal{G}}P^k\left(\varphi^{-1}(g)\,|\, s,g'\right)\pi_\mathrm{hi}^*(g'\,|\, s).
\label{eq:construction}
\end{equation}
Intuitively, we construct an adjacency-constrained high-level policy from an optimal high-level policy without the constraint, by substituting each raw subgoal generated by the policy with its induced state distribution. The state value function of $\pi_\mathrm{hi}^\mathrm{adj}$ is then:
\begin{equation}
\begin{aligned}
&V^{\pi_\mathrm{hi}^\mathrm{adj}}(s) \\
&= \sum_{g\in\mathcal{G}} \pi_\mathrm{hi}^\mathrm{adj}(g\,|\, s)
\sum_{s'\in\mathcal{S}} P^k(s'\,|\, s,g) \left(\widetilde{r}_\mathrm{hi}(s,s') + \gamma V^{\pi_\mathrm{hi}^\mathrm{adj}}(s')\right) \\
&= \sum_{g\in\mathcal{G}} \sum_{g'\in\mathcal{G}}P^k\left(\varphi^{-1}(g)\,|\, s,g'\right)\pi_\mathrm{hi}^*(g'\,|\, s) \\
&\qquad\sum_{s'\in\mathcal{S}} P^k(s'\,|\, s,g) \left(\widetilde{r}_\mathrm{hi}(s,s') + \gamma V^{\pi_\mathrm{hi}^\mathrm{adj}}(s')\right).
\end{aligned}
\end{equation}
Since the state distribution $P^k(s'\,|\, s,g)$ is defined on the maximal $k$-step adjacent region $\mathcal{G}_\mathrm{AM}$ of the state, the constructed high-level policy must satisfy the adjacency constraint. Therefore, to compute the expected state value under $\pi_\mathrm{hi}^\mathrm{adj}$ we only need to consider the subgoals falling into $\mathcal{G}_\mathrm{AM}$. This gives
\begin{equation}
\begin{aligned}
V^{\pi_\mathrm{hi}^\mathrm{adj}}(s) &= \sum_{g\in\mathcal{G}_\mathrm{AM}(s,k)}\sum_{g'\in\mathcal{G}}P^k\left(\varphi^{-1}(g)\,|\, s,g'\right)\pi_\mathrm{hi}^*(g'\,|\, s) \\
&\qquad \sum_{s'\in\mathcal{S}} P^k(s'\,|\, s,g) \left(\widetilde{r}_\mathrm{hi}(s,s') + \gamma V^{\pi_\mathrm{hi}^\mathrm{adj}}(s')\right).
\end{aligned}
\end{equation}

Next, we would like to bound the state-value suboptimality induced by such construction. We first formalize a critical factor termed \emph{transition mismatch rate} that reflects the transition structure of the MDP and plays an important role in our analysis.
\begin{definition}[Transition mismatch rate]
\label{def:distribution_mismatch}
Let $\mathcal{M}$ be a goal-conditioned MDP as defined in Section~\ref{sec:pre} and $P^k(s'\,|\,s, g)$ its $k$-step transition probability under an optimal goal-conditioned policy as defined by Equation~\eqref{equ:goal_conditioned}. The $k$-step transition mismatch rate of $\mathcal{M}$ is defined as:
\begin{equation}
\mu_k \vcentcolon= \max_{s, g, s'} \left| P^k(s'\,|\,s,g) - \sum_{\widetilde{s}\in\mathcal{S}}P^k(s'\,|\,s,\varphi(\widetilde{s}))P^k(\widetilde{s}\,|\,s, g) \right|.
\label{equ:distribution_mismatch}
\end{equation}
\end{definition}
The above definition captures the transition structure of MDP in terms of ``goal-reaching reliability'': a small $\mu_k$ indicates that for every state, its adjacent goals can be reached with a high probability if the goal-reaching policy is itself optimal. For example, in deterministic MDPs, every subgoal that falls into the maximal $k$-step adjacent region $\mathcal{G}_\mathrm{AM}$ can be reliably reached with probability 1 given an optimal low-level policy since all transitions are deterministic; it is easy to verify that in this case, we have $\mu_k = 0$. Intuitively, this parameter should play an important role in bounding the suboptimality in subgoal conditioning: if all adjacent subgoals can be perfectly reached with probability 1, then our construction~\eqref{eq:construction} should fully preserve the optimality as in deterministic cases; when this is not true, then the ``harder'' adjacent subgoals can be reached, the larger the difference between the state distribution induced by our construction~\eqref{eq:construction} and that induced by the original subgoal. Fig.~\ref{fig:theory} is an illustration of the idea above.

Given Definition~\ref{def:distribution_mismatch}, we have the following theorem which provides a suboptimality upper bound.
\begin{figure}[t]
\centering
\includegraphics[width=0.99\linewidth]{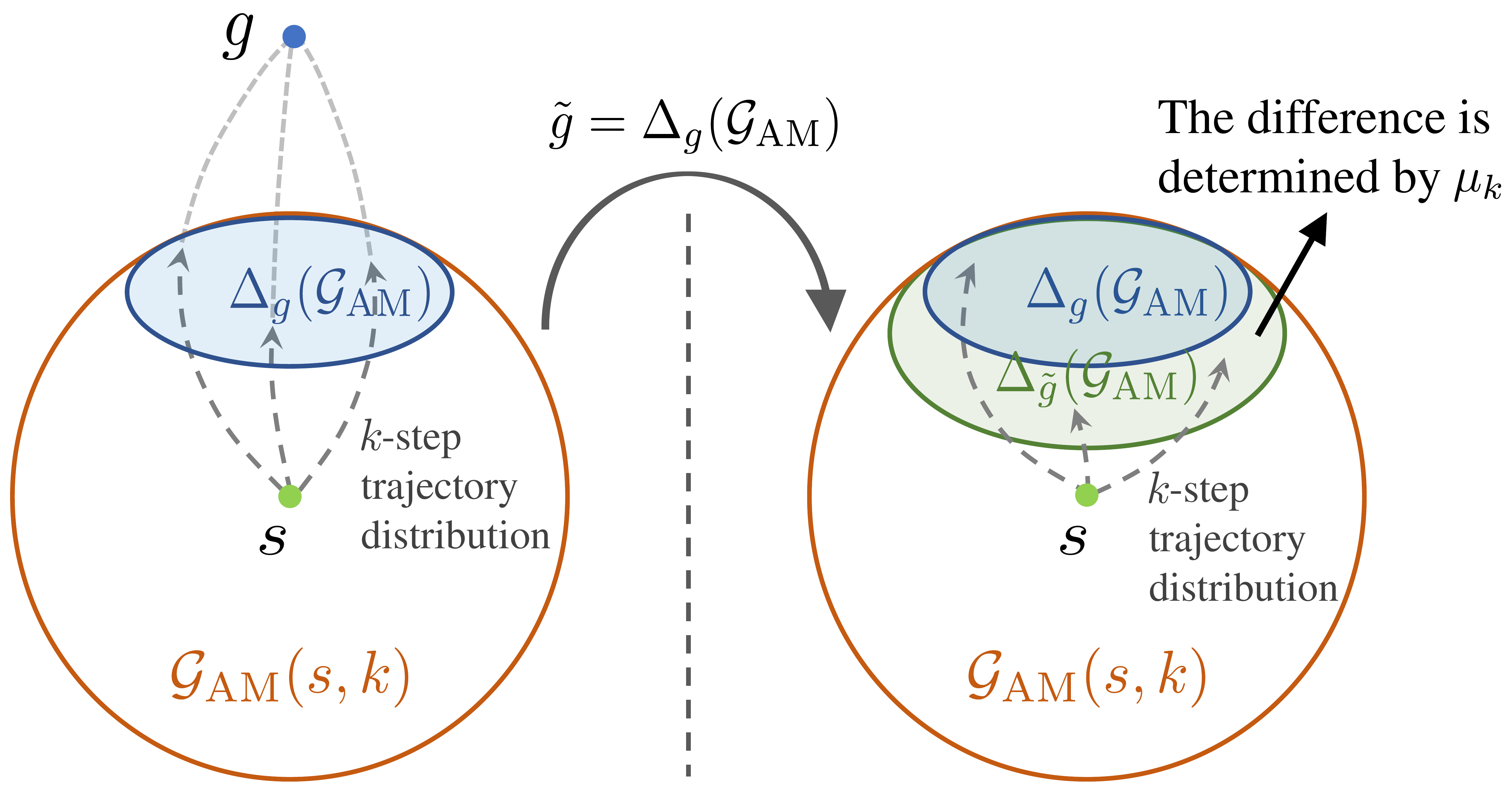}
\caption{A high-level illustration of our theoretical result in stochastic MDPs. $\Delta_g(\mathcal{G}_\mathrm{AM})$ and $\Delta_{\tilde{g}}(\mathcal{G}_\mathrm{AM})$ denote the $k$-step state distribution induced by the subgoal $g$ and the subgoal distribution $\tilde{g}$ respectively.}
\label{fig:theory}
\end{figure}

\begin{theorem} Let $V^{\pi_\mathrm{hi}}$ be the high-level state value function given a high-level policy $\pi_\mathrm{hi}$, under an optimal low-level policy as defined by Equation~\eqref{equ:goal_conditioned}. Let $\pi^*_\mathrm{hi}$ be an optimal high-level policy without the adjacency constraint. Then, there exists a high-level policy $\pi^\mathrm{adj}_\mathrm{hi}$ that satisfies the $k$-step adjacency constraint (that is, given a state $s\in\mathcal{S}$, all subgoals generated by $\pi^\mathrm{adj}_\mathrm{hi}$ fall into the maximal $k$-step adjacent region $\mathcal{G}_\mathrm{AM}(s,k)$), such that:
\begin{equation}
\left\lVert V^{\pi^*_\mathrm{hi}} - V^{\pi^\mathrm{adj}_\mathrm{hi}} \right\rVert_\infty \le \frac{\mu_k k R_\mathrm{max}}{2(1-\gamma)} + \frac{\gamma\mu_k k R_\mathrm{max}}{2(1-\gamma)^2},
\label{eq:main_bound}
\end{equation}
where $\mu_k$ is the $k$-step transition mismatch rate of the MDP.
\label{theo:subopt}
\end{theorem}
\begin{proof}
See Appendix~\ref{app:proof3}.
\end{proof}

Theorem~\ref{theo:subopt} indicates a direct dependence on the transition mismatch rate $\mu_k$ when bounding the suboptimality induced by the adjacency constraint. As mentioned above, in deterministic MDPs we have $\mu_k = 0$, which yields an upper bound of 0 by Equation~\eqref{eq:main_bound}, recovering our theoretical results in deterministic MDPs in Section~\ref{sec:theory_det}.

\section{HRL with Adjacency Constraint}
\label{sec:method}
In this section, we will present our method of Hierarchical Reinforcement learning with $k$-step Adjacency Constraint (HRAC). First, since in practice our original formulation~\eqref{equ:formulation} is hard to optimize due to the strict constraint, we employ the relaxation technique and derive the following unconstrained optimizing objective:
\begin{equation}
  \underset{\theta_\mathrm{hi}}{\mathrm{max}}\quad \mathbb{E}_{\pi_\mathrm{hi}^{\theta_\mathrm{hi}}}\sum_{t=0}^{T-1}\Bigg[\gamma^{t} r_{kt}^\mathrm{hi} - \eta\cdot H\Big(d_{\mathrm{st}}\left(s_{kt},\varphi^{-1}(g_{kt})\right),k\Big)\Bigg] ,
  \label{equ:formulation_unconstrained}
\end{equation}
where $H(x,k) \vcentcolon= \max(\frac{x}{k}-1,0)$ is the hinge loss function and $\eta$ is a balancing coefficient.
However, the exact calculation of the shortest transition distance $d_{\mathrm{st}}(s_1,s_2)$ between two arbitrary states $s_1,s_2\in\mathcal{S}$ remains complex and non-differentiable.
In the sequel, we introduce a simple method to collect and aggregate the adjacency information from the environment interactions. We then train an adjacency network using the aggregated adjacency information to approximate the shortest transition distance in a parameterized form, which enables practical optimization of Equation~\eqref{equ:formulation_unconstrained}.

\subsection{Practical Implementation of Adjacency Constraint}
\label{subsec:adjacency}

As shown in prior research~\cite{pong_temporal_2018,florensa_self-supervised_2019,eysenbach_search_2019,hartikainen_dynamical_2020}, accurately computing the shortest transition distance is not easy and often has the same complexity as learning an optimal low-level goal-conditioned policy. However, from the perspective of goal-conditioned HRL, we do not need a perfect shortest transition distance measure or a low-level policy that can reach any distant subgoals. Instead, only a discriminator of $k$-step adjacency suffices, and it is enough to learn a low-level policy that can reliably reach nearby subgoals (more accurately, subgoals that fall into the $k$-step adjacent region of the current state) rather than all potential subgoals in the goal space. 

Given the analysis above, here we introduce a simple approach to determine whether a subgoal satisfies the $k$-step adjacency constraint. We first note that Equation~\eqref{equ:dist} can be approximated as follows:
\begin{equation}
d_{\mathrm{st}}(s_1,s_2) \approx \min_{\pi\in\{\pi_1,\pi_2,\cdots,\pi_n\}} \sum_{t=0}^\infty \,t P(\mathcal{T}_{s_1s_2}=t\,|\,\pi),
\label{equ:approximation}%
\end{equation}
where $\{\pi_1,\pi_2,\cdots,\pi_n\}$ is a finite policy set containing $n$ different deterministic policies. Obviously, if these policies are diverse enough, we can effectively approximate the shortest transition distance with a sufficiently large $n$. However,
training a set of diverse policies separately is costly,
and using one single policy to approximate the policy set ($n=1$)~\cite{savinov_semi-parametric_2018,savinov_episodic_2019} often leads to non-optimality. To handle this difficulty, we exploit the fact that the low-level policy itself changes over time during the training procedure.
We can thus build a policy set by sampling policies that emerge in different training timesteps.
To aggregate the adjacency information gathered by multiple policies, we propose to explicitly memorize the adjacency information by constructing a binary \emph{$k$-step adjacency matrix} of the explored states. The adjacency matrix has the same size as the number of explored states, and each element represents whether two states are $k$-step adjacent. In practice, we use the agent's trajectories, where the temporal distances between states can indicate their adjacency, to construct and update the adjacency matrix online. The detailed process is as follows.

\textit{Constructing and updating the adjacency matrix:}
the adjacency matrix is initialized to an empty matrix at the beginning of training. Each time when the agent explores a new state that it has never visited before, the adjacency matrix is augmented by a new row and a new column with zero elements, representing the $k$-step adjacent relation between the new state and explored states. When the temporal distance between two states in one trajectory is not larger than $k$, then the corresponding element in the adjacency matrix will be labeled to 1, indicating the adjacency. (The diagonal of the adjacency matrix will always be labeled to 1.) Although the temporal distance between two states based on a single trajectory is often larger than the real shortest transition distance, it can be easily shown that the adjacency matrix with this labeling strategy can converge to the optimal adjacency matrix asymptotically with sufficient trajectories sampled by different policies. In our implementation, we employ a trajectory buffer to store newly-sampled trajectories and update the adjacency matrix online in a fixed frequency using the stored trajectories. The trajectory buffer is cleared after each update.

In practice, using an adjacency matrix is insufficient as this procedure is non-differentiable and cannot generalize to newly-visited states.
Therefore, we further distill the adjacency information stored in a constructed adjacency matrix into an adjacency network $\psi_\phi$ parameterized by $\phi$.
The adjacency network learns a mapping from the goal space to an adjacency space, where the Euclidean distance between the state and the goal is consistent with their shortest transition distance:
\begin{equation}
\tilde{d}_{\mathrm{adj}}(s_1,s_2\,|\,\phi) \vcentcolon= \frac{k}{\epsilon_k}\lVert \psi_\phi(g_1) - \psi_\phi(g_2) \rVert_2,
\label{equ:adjacency_net}
\end{equation}
where $g_1=\varphi(s_1),\,g_2=\varphi(s_2)$ and $\epsilon_k\in\mathbb{R}_+$ is a scaling factor. As we have mentioned above, it is hard to regress the Euclidean distance in the adjacency space to the shortest transition distance accurately, and we only need to ensure a binary relation for implementing the adjacency constraint, i.e., $\lVert \psi_\phi(g_1) - \psi_\phi(g_2) \rVert_2 > \epsilon_k$ for $d_{\mathrm{st}}(s_1,s_2)>k$, and $\lVert \psi_\phi(g_1) - \psi_\phi(g_2) \rVert_2 < \epsilon_k$ for $d_{\mathrm{st}}(s_1,s_2)< k $, as shown in Fig.~\ref{fig:distance_learning}.
Inspired by the works in metric learning research~\cite{hadsell_dimensionality_2006}, we adopt a contrastive-like loss function for this distillation process:
\begin{equation}
\begin{aligned}
    \mathcal{L}_{\mathrm{dis}}(\phi) &= \mathbb{E}_{s_i,s_j\in\mathcal{S}} \left[\,l \cdot \max\left(\lVert \psi_\phi(g_i) - \psi_\phi(g_j) \rVert_2 - \epsilon_k,\,0\right)\right. \\
    +& \left.(1 - l) \cdot \max\left(\epsilon_k + \delta - \lVert \psi_\phi(g_i) - \psi_\phi(g_j) \lVert_2,\,0\right)\right],
\end{aligned}
\label{equ:contrastive}
\end{equation}
where $g_i = \varphi(s_i)$, $g_j = \varphi(s_j)$, and a hyperparameter $\delta>0$ is used to create a gap between the embeddings. $l\in\{0,1\}$ represents the label indicating $k$-step adjacency derived from the $k$-step adjacency matrix. Equation~\eqref{equ:contrastive} penalizes adjacent state embeddings ($l=1$) with large Euclidean distances in the adjacency space and non-adjacent state embeddings ($l=0$) with small Euclidean distances.

In practice, the adjacency network is trained by minimizing the objective defined by Equation~\eqref{equ:contrastive}. We use states uniformly sampled from the adjacency matrix (i.e., from the set of all explored states) to approximate the expectation, and train the adjacency network each time after the adjacency matrix is updated with new trajectories. Note that by explicitly aggregating the adjacency information using an adjacency matrix, we can perform uniform sampling over all explored states and thus achieve a nearly unbiased estimation of the expectation, which cannot be realized when we directly sample state-pairs from the trajectories (see the comparison with the work of Savinov et al.~\cite{savinov_semi-parametric_2018,savinov_episodic_2019} in Appendix~\ref{appsec:comp} for details).
\begin{figure}[t]
\centering
\includegraphics[width=0.7\linewidth]{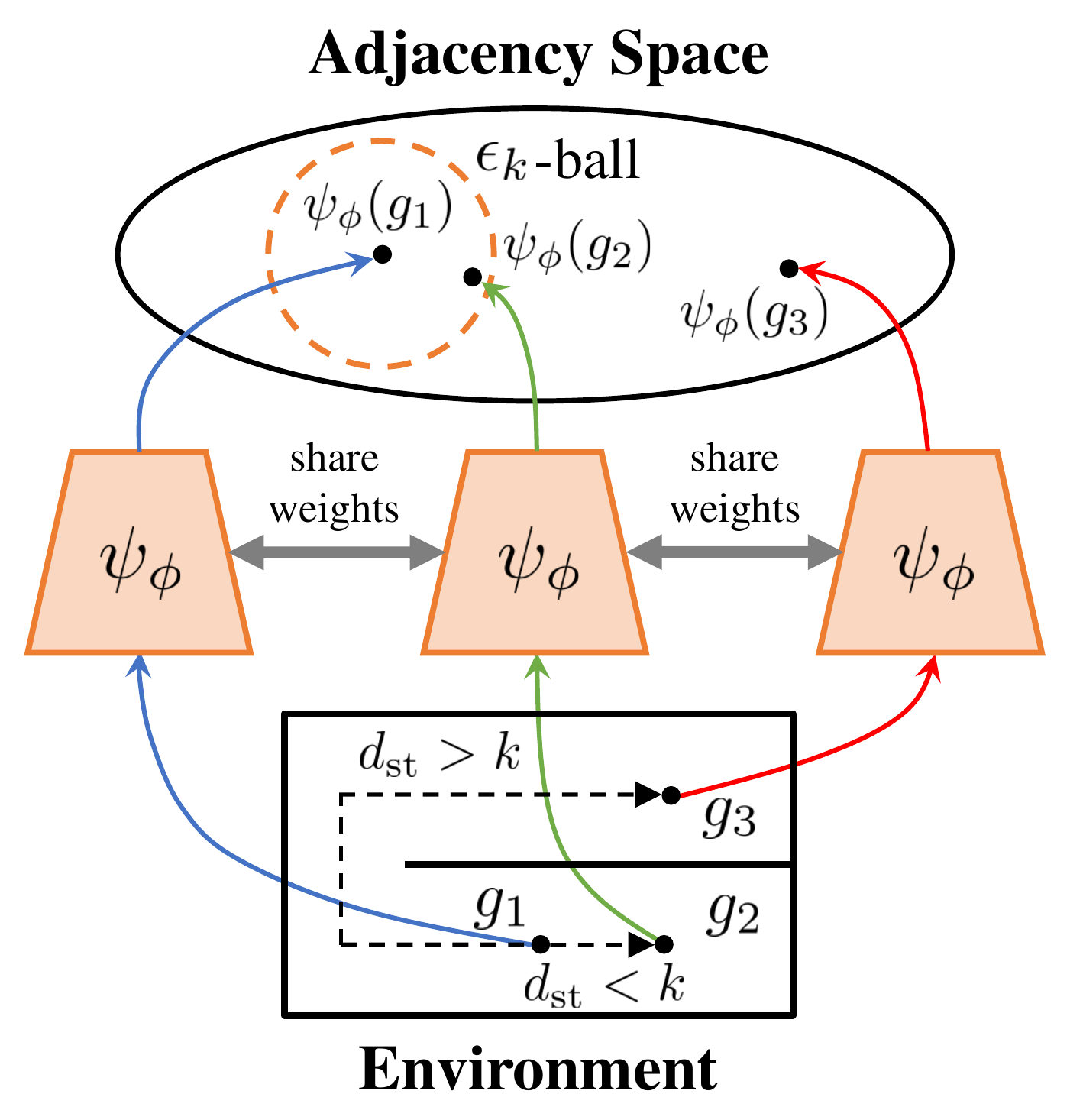}
\caption{The functionality of the adjacency network. The $k$-step adjacent region is mapped to an $\epsilon_k$-ball with Euclidean metric in the adjacency space, where $e_{g_i} = \psi_\phi(g_i),\,i=1,2,3$.}
\label{fig:distance_learning}
\end{figure}

Embedding all subgoals with a single adjacency network is enough to express adjacency when the environment is reversible. However, when this condition is not satisfied, it is insufficient to express directional adjacency using one adjacency network since the parameterized approximation defined by Equation~\eqref{equ:adjacency_net} is symmetric for $s_1$ and $s_2$. In this case, one can use two separate sub-networks to embed $g_1$ and $g_2$ in Equation~\eqref{equ:adjacency_net} respectively using the structure proposed in UVFA~\cite{schaul_universal_2015}.

Although the construction of an adjacency matrix limits our method to tasks with tabular state spaces, we can also handle continuous state spaces using goal space discretization (see our continuous control experiments in Section~\ref{sec:experiment}). For applications with vast state spaces, constructing a complete adjacency matrix in the raw goal space will be intractable, but it is still possible to scale our method to these scenarios using specific feature construction~\cite{ecoffet_first_2021} or dimension reduction methods like VQ-VAE~\cite{oord_neural_2017}, or replacing the distance learning procedure with more accurate distance learning algorithms~\cite{florensa_self-supervised_2019,eysenbach_search_2019} at the cost of some efficiency. We consider possible extensions in this direction as future work.

\begin{algorithm}[tb]
   \caption{HRAC}
   \label{algo:hrac}
\begin{algorithmic}[1]
   \REQUIRE High-level policy $\pi_\mathrm{hi}$ parameterized by $\theta_\mathrm{hi}$, low-level policy $\pi_\mathrm{lo}$ parameterized by $\theta_\mathrm{lo}$, adjacency netork $\psi_\phi$ parameterized by $\phi$, state-goal mapping function $\varphi$, goal transition function $h$, high-level action frequency $k$, number of training episodes $N$, adjacency learning frequency $C$, empty adjacency matrix $\mathcal{M}$, empty trajectory buffer $\mathcal{B}$.
   \STATE Sample and store trajectories in the trajectory buffer $\mathcal{B}$ using a random policy.
   \STATE Construct the adjacency matrix $\mathcal{M}$ using the trajectory buffer $\mathcal{B}$.
   \STATE Pre-train $\psi_\phi$ using $\mathcal{M}$ by minimizing Equation~\eqref{equ:contrastive}.
   \STATE Clear $\mathcal{B}$.
   \FOR {$n=1$ {\bfseries to} $N$}
       \STATE Reset the environment and sample the initial state $s_0$.
       \STATE $t = 0$.
       \REPEAT
       \IF {$t\equiv 0\,(\mathrm{mod}\ k)$}
         \STATE Sample subgoal $g_t \sim \pi_\mathrm{hi}(g\,|\,s_t)$.
       \ELSE
           \STATE Perform subgoal transition $g_t = h(g_{t-1},s_{t-1},s_t)$.
       \ENDIF
       \STATE Sample low-level action $a_t \sim \pi_\mathrm{lo}(a\,|\,s_t,g_t)$.
       \STATE Sample next state $s_{t+1} \sim P(s\,|\,s_t,a_t)$.
       \STATE Sample reward $r_t \sim R(r\,|\,s_t,a_t, s_{t+1})$.
       \STATE Sample episode end signal $done$.
       \STATE $t = t+1$.
     \UNTIL {$done$ is $true$.}
     \STATE Store the sampled trajectory in $\mathcal{B}$.
     \STATE Update the parameters of the high-level policy $\theta_\mathrm{hi}$ according to Equation~\eqref{equ:total_loss} and~\eqref{equ:goal_loss}.
     \STATE Update the parameters of the low-level policy $\theta_\mathrm{lo}$.
     \IF {$n\equiv 0\,(\mathrm{mod}\ C)$}
         \STATE Update the adjacency matrix $\mathcal{M}$ using the trajectory buffer $\mathcal{B}$.
         \STATE Fine-tune $\psi_\phi$ using $\mathcal{M}$ by minimizing Equation~\eqref{equ:contrastive}.
         \STATE Clear $\mathcal{B}$.
     \ENDIF
   \ENDFOR
\end{algorithmic}
\end{algorithm}

\subsection{Combining Adjacency Constraint and HRL}

With a learned adjacency network $\psi_\phi$, we can now incorporate the adjacency constraint into the goal-conditioned HRL framework.
According to Equation~\eqref{equ:formulation_unconstrained}, we introduce an \emph{adjacency loss} $\mathcal{L}_\mathrm{adj}$ to replace the original strict adjacency constraint and minimize the following high-level objective:
\begin{equation}
\mathcal{L}_\mathrm{high}(\theta_\mathrm{hi}) = -\mathbb{E}_{\pi_\mathrm{hi}^{\theta_\mathrm{hi}}}\sum_{t=0}^{T-1}\left(\gamma^{t} r_{kt}^\mathrm{hi} - \eta\cdot\mathcal{L}_\mathrm{adj}\right),
\label{equ:total_loss}
\end{equation}
where $\eta$ is the balancing coefficient, and the adjacency loss $\mathcal{L}_\mathrm{adj}$ is derived by replacing $d_{\mathrm{st}}$ with $\tilde{d}_{\mathrm{st}}$ defined by Equation~\eqref{equ:adjacency_net} in the second term of Equation~\eqref{equ:formulation_unconstrained}:
\begin{equation}
\begin{aligned}
\mathcal{L}_\mathrm{adj}(\theta_\mathrm{hi}) &= H\left(\tilde{d}_{\mathrm{adj}}\left(s_{kt},\varphi^{-1}(g_{kt})\,|\,\phi\right),k\right)\\
&\propto \max\left(\lVert \psi_\phi(\varphi(s_{kt})) - \psi_\phi(g_{kt}) \rVert_2 - \epsilon_k,\,0\right),
\label{equ:goal_loss}
\end{aligned}
\end{equation}
where $g_{kt}\sim \pi_\mathrm{hi}(g\,|\,s_{kt})$. Equation~\eqref{equ:goal_loss} will output a non-zero value when the generated subgoal and the current state have a Euclidean distance larger than $\epsilon_k$ in the adjacency space, indicating non-adjacency. It is thus consistent with the $k$-step adjacency constraint. In practice, we plug $\mathcal{L}_{\mathrm{adj}}$ as an extra loss term into the original policy loss term of a specific high-level RL algorithm, e.g., TD error for temporal-difference learning methods. In the loss backpropagation phase, we keep the adjacency network fixed, and use the gradients only to update policy networks. We provide Algorithm~\ref{algo:hrac} to detail the training procedure of HRAC.

\section{Experimental Evaluation}
\label{sec:experiment}
We have presented the formulation and implementation of HRAC.
Our experiments are designed to answer the following questions: (1)~Can HRAC promote the generation of adjacent subgoals? (2)~Can HRAC improve the sample efficiency and the overall performance of goal-conditioned HRL? (3)~Can HRAC outperform other strategies that may also improve the learning efficiency of hierarchical agents, e.g., low-level hindsight experience replay~\cite{andrychowicz_hindsight_2017}?

\subsection{Environment Setup}
\label{subsec:setup}
We employed two types of tasks with discrete and continuous state and action spaces to evaluate the effectiveness of our method. Discrete control tasks include Key-Chest and Maze, where the agents are spawned in grid worlds with injected stochasticity and need to accomplish tasks that require both low-level control and high-level planning, as shown in Fig.~\ref{fig:discrete_env}. Continuous control tasks include two task suites. The first suite is a collection of quadrupedal robot locomotion tasks including Ant Gather, Ant Maze, Ant Push, and Ant Maze Sparse, where the first three tasks are widely-used benchmarks in the HRL community~\cite{duan_benchmarking_2016,florensa_stochastic_2017,nachum_data-efficient_2018,nachum_near-optimal_2019,levy_learning_2019}, and the last task is a more challenging locomotion task with sparse rewards, as shown in Fig.~\ref{fig:ant_env}. The second suite contains two robot arm manipulation tasks with sparse rewards, including FetchPush and FetchPickAndPlace introduced by~\cite{plappert_multi-goal_2018}, as depicted in Fig.~\ref{fig:fetch_env}. In discrete tasks and quadrupedal robot locomotion tasks, we used a pre-defined 2-dimensional goal space that represents the $(x,y)$ position of the agent; in manipulation tasks, we used a pre-defined 3-dimensional goal space representing the $(x,y,z)$ position of the gripper. The details of each environment are as follows:

\textit{Key-Chest:} this environment is a grid world with size $13\times17$, as shown in Fig.~\ref{subfig:keychest}. In this environment, the agent (A) starts from a random position and needs to pick up the key (K) first, then uses the key to open the chest (C). The environment has a discrete 3-dimensional state space, where the first two dimensions represent the $(x,y)$ position of the agent respectively, and the third dimension represents whether the agent has picked up the key (1 if the agent has the key and 0 otherwise). The action space is discrete with size 4, containing actions moving towards four directions. The agent is provided with sparse rewards of $+1$ and $+5$, respectively for picking up the key and opening the chest. Each episode ends if the agent opens the chest or runs out of the step limit of 500. Environmental stochasticity is introduced by replacing the action of the agent with a random action each step with a probability of 0.25.

\textit{Maze:} this environment is a grid world with size $13\times17$, as shown in Fig.~\ref{subfig:maze}. The agent (A) starts from a fixed position and needs to reach the final goal (G) in the middle of the maze. The environment has a discrete 2-dimensional state space representing the $(x,y)$ position of the agent. The action space is the same as the Key-Chest environment. The agent is provided with dense rewards to facilitate exploration, i.e., $+0.1$ each step if the agent moves closer to the goal, and $-0.1$ each step if the agent moves farther. Each episode has a maximum length of 200. The random action probability of the environment is 0.25.

\textit{Ant Gather:} this environment defines a quadrupedal robot locomotion task, as shown in Fig.~\ref{subfig:antgather}. The ant robot is spawned at the center of the map and needs to gather apples while avoiding bombs. Both apples and bombs are randomly placed in the environment at the beginning of each episode. The environment has a continuous state space including the current position and velocity of the robot, the current time step $t$, and the depth readings defined by the standard Gather environment~\cite{duan_benchmarking_2016}. The depth readings represent the Euclidean distances between the agent and the nearby apples and bombs. Following the settings in prior works~\cite{florensa_stochastic_2017,nachum_data-efficient_2018}, we set the activity range of the sensor to 10 and the sensor span to $2\pi$. We use the ant robot pre-defined by Rllab, with an 8-dimensional continuous action space. The agent receives a positive reward of $+1$ for each apple and a negative reward of $-1$ for each bomb. Each episode terminates at 500 time steps. No random action is applied.

\begin{figure}[t]
    \centering
    \subfigure[]{\includegraphics[width=0.37\linewidth]{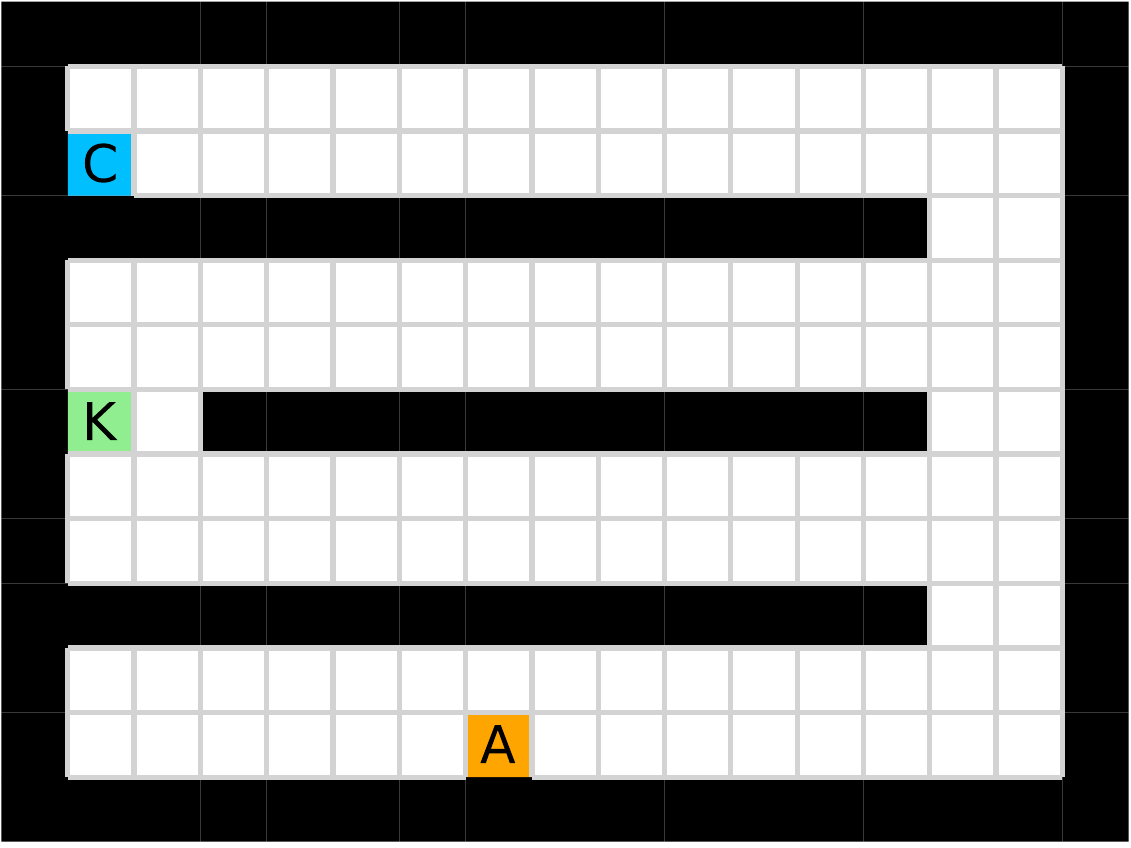}\label{subfig:keychest}}
    \hspace{1em}
    \subfigure[]{\includegraphics[width=0.37\linewidth]{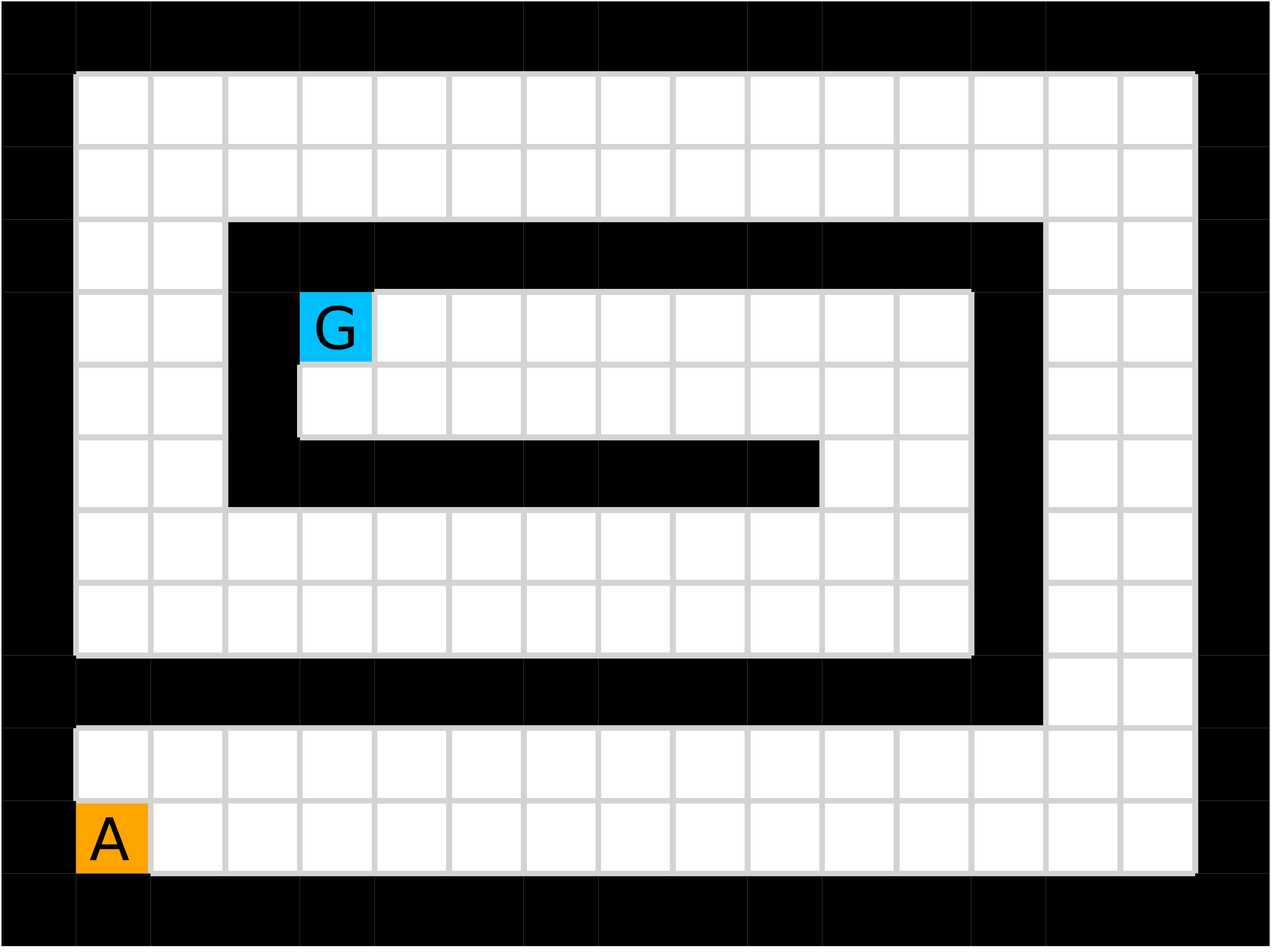}\label{subfig:maze}}
    \caption{Discrete control environments used in our experiments. (a)~Key-Chest. (b)~Maze.}
    \label{fig:discrete_env}
\end{figure}

\begin{figure}[t]
    \centering
    \subfigure[]{\includegraphics[width=0.37\linewidth]{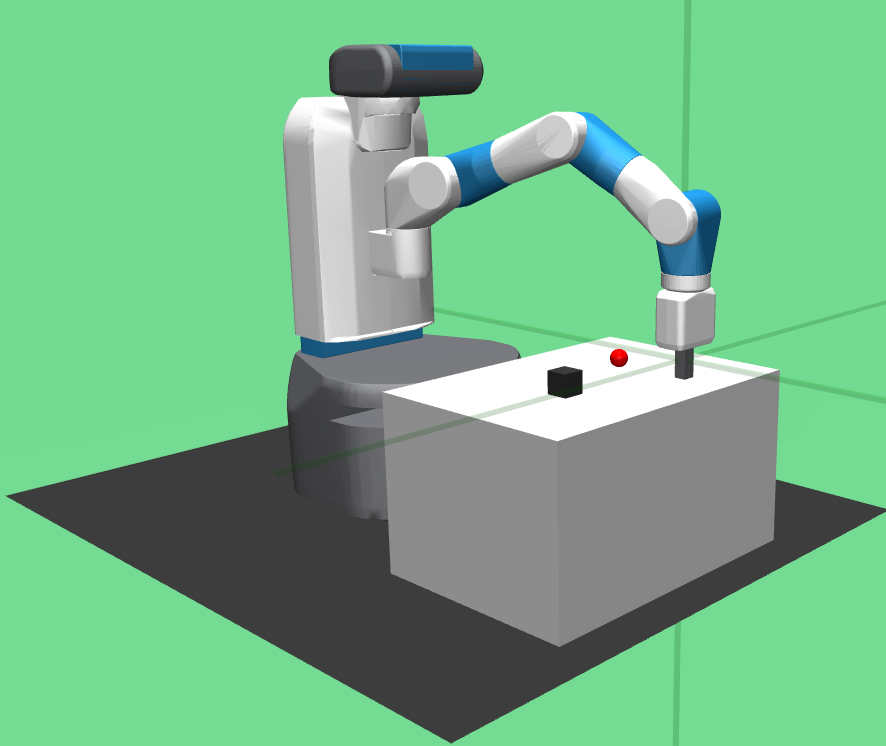}\label{subfig:fetchpush}}
    \hspace{1em}
    \subfigure[]{\includegraphics[width=0.363\linewidth]{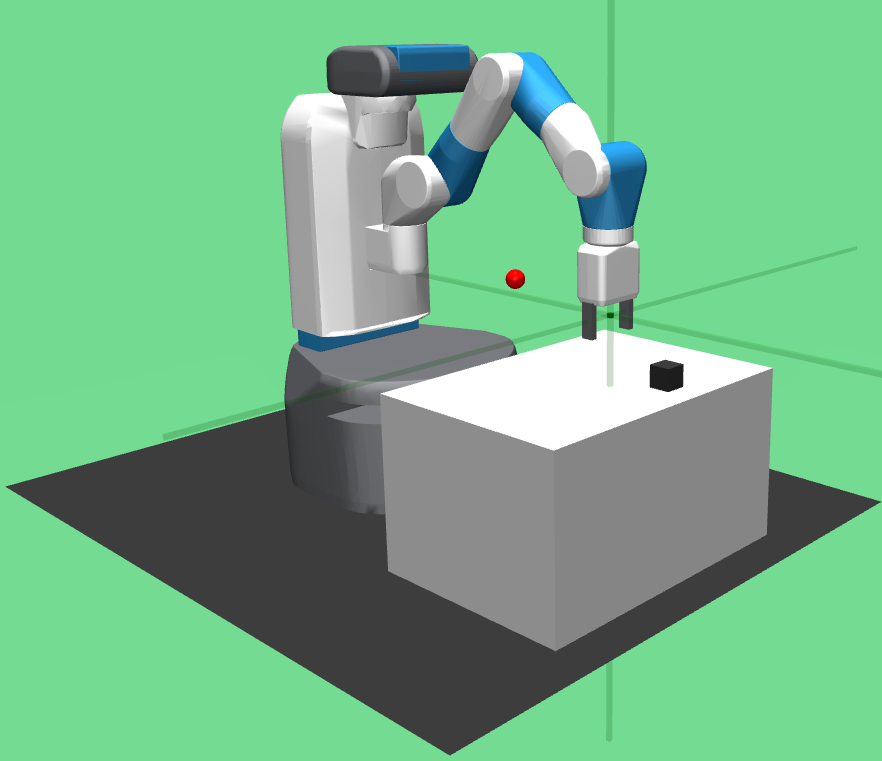}\label{subfig:fetchpick}}
    \caption{Robot arm manipulation environments used in our experiments. (a)~FetchPush. (b)~FetchPickAndPlace.}
    \label{fig:fetch_env}
\end{figure}

\textit{Ant Maze:} this environment defines a quadrupedal robot locomotion task, as shown in Fig.~\ref{subfig:antmaze}. The ant robot starts from the bottom left of a maze and needs to reach a target position. The environment has a continuous state space including the current position and velocity, the target location, and the current time step $t$. In the training stage, the environment randomly samples a target position at the beginning of each episode, and the agent receives a dense reward at each time step according to its negative Euclidean distance from the target position with a scaling of 0.1. At the evaluation stage, the target position is fixed to a hard goal at the top left of the maze, and success is defined as being within a Euclidean distance of 5 from the target. Each episode ends at 500 time steps. No random action is applied.

\textit{Ant Maze Sparse:} this environment defines a quadrupedal robot locomotion task with sparse rewards, as shown in Fig.~\ref{subfig:antmazesparse}. The ant robot starts from a random position in a maze and needs to reach a target position at the center of the maze. The environment has the same state and action spaces as the Ant Maze environment. The agent is rewarded by $+1$ only if it reaches the goal, which is defined as having a Euclidean distance smaller than 1 from the goal. At the beginning of each episode, the agent is randomly placed in the maze except at the goal position. Each episode is terminated if the agent reaches the goal or after 500 steps. No random action is applied.

\textit{Ant Push:} this environment defines a quadrupedal robot locomotion task with object manipulation, as shown in Fig.~\ref{subfig:antpush}. The ant robot starts at the bottom of the maze and needs to reach a target position at the top. Since the goal is initially blocked by a movable block (red), the agent needs to first push the block away before reaching the goal. The environment has the same state and action spaces as the Ant Maze environment. During training, the agent receives a dense reward at each time step according to its negative Euclidean distance from the target position with a scaling of 0.1. At the evaluation stage, success is defined as being within a Euclidean distance of 5 from the target. Each episode ends at 500 time steps. No random action is applied.

\textit{FetchPush:} this environment defines a robot arm manipulation task with sparse rewards. In this environment, an object is placed in front of the robot and the goal is to move it to a target location on the table, as shown in Fig.~\ref{subfig:fetchpush}. The environment has a continuous state space that includes the Cartesian positions and the velocities of the gripper and the object as well as their relative positions and velocities, as detailed in~\cite{plappert_multi-goal_2018}. The action space is 4-dimensional, with the first 3 dimensions specifying the desired gripper movement in Cartesian coordinates and the last dimension representing the opening and closing status of the gripper. The environment has a 3-dimensional goal space that describes the desired position of the object, which is randomized at the beginning of each episode. The agent receives a sparse reward of $0$ if the object is at the target location within a tolerance of 5 cm and $-1$ otherwise. No random action is applied.

\textit{FetchPickAndPlace:} this environment defines a robot arm manipulation task with sparse rewards. In this environment, an object is placed in front of the robot and the goal is to grasp the box and move it to the target location which may be on the table or in the air above the table, as shown in Fig.~\ref{subfig:fetchpick}. The state space, action space, goal space and reward setting are the same as the FetchPush task. Compared with FetchPush, this task is more challenging since the robot often needs to grasp the object first and then move it to the target position before obtaining the final reward, which poses difficulty on exploration. No random action is applied.

\begin{figure}[t]
    \centering
    \subfigure[]{\includegraphics[width=0.33\linewidth]{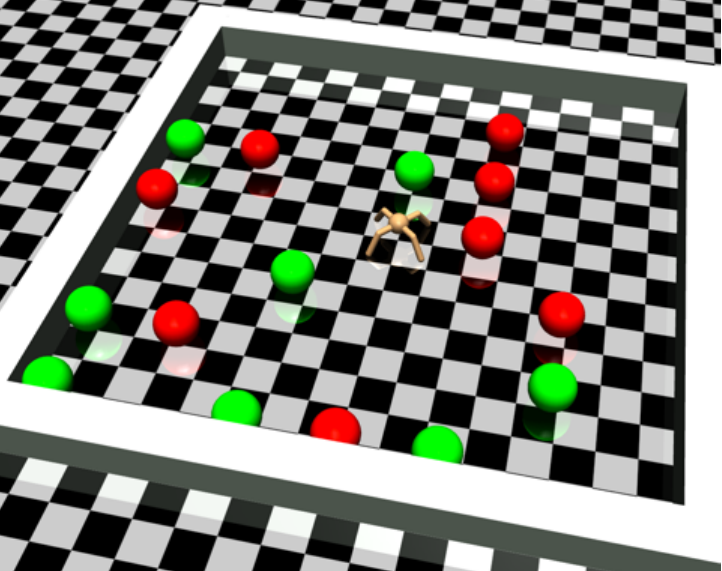}\label{subfig:antgather}}
    \hspace{1em}
    \subfigure[]{\includegraphics[width=0.41\linewidth]{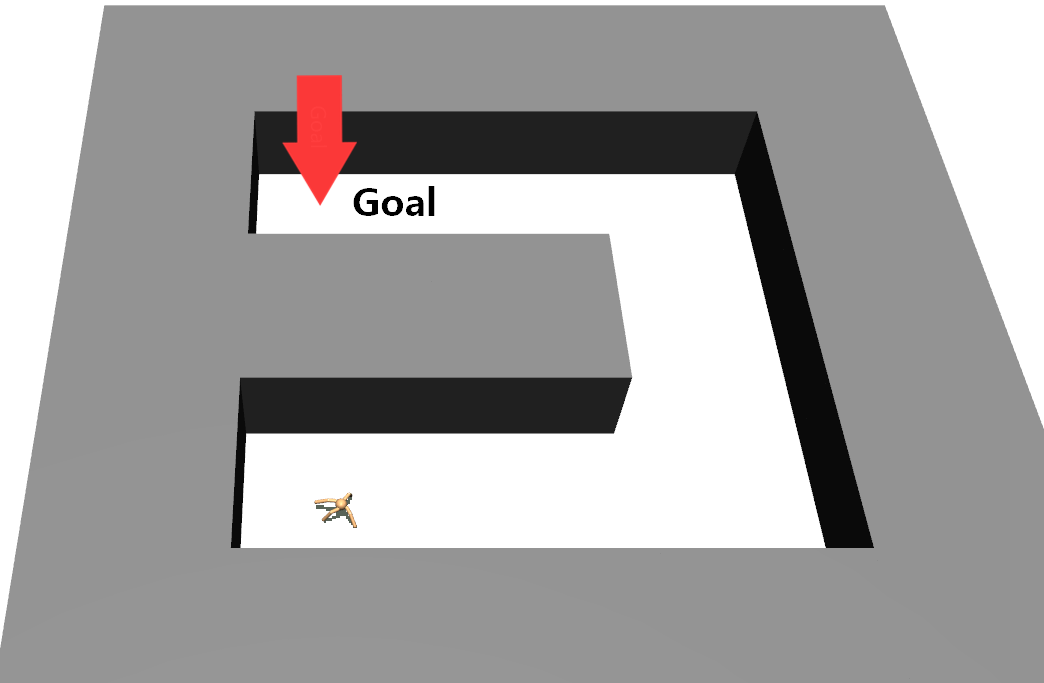}\label{subfig:antmaze}}\\
    \subfigure[]{\includegraphics[width=0.38\linewidth]{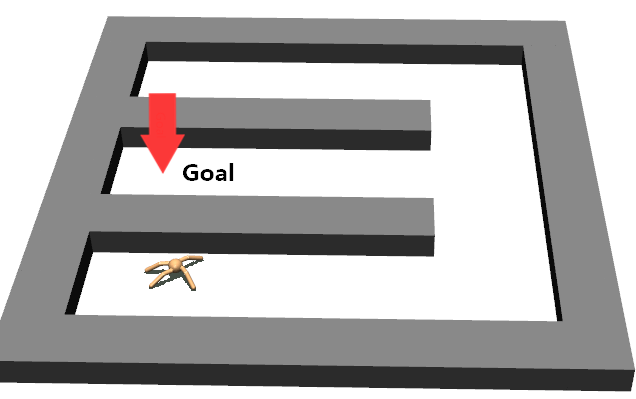}\label{subfig:antmazesparse}}
    \hspace{1em}
    \subfigure[]{\includegraphics[width=0.4\linewidth]{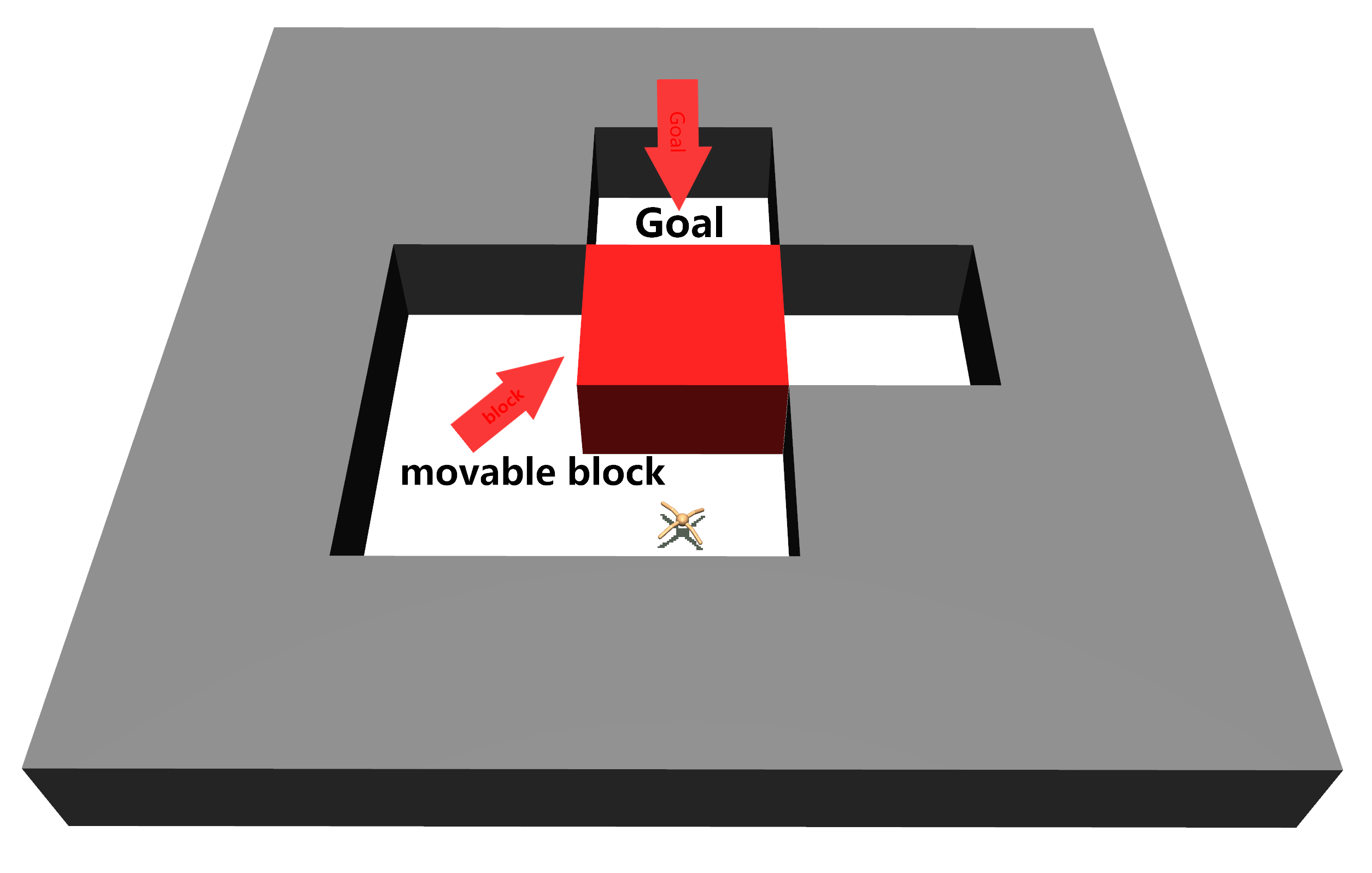}\label{subfig:antpush}}
    \caption{Quadrupedal robot locomotion environments used in our experiments. (a)~Ant Gather (adapted from Duan et al.~\cite{duan_benchmarking_2016}). (b)~Ant Maze. (c)~Ant Maze Sparse. (d)~Ant Push.}
    \label{fig:ant_env}
\end{figure}

\subsection{Comparative Experiments}
\label{subsec:comp_exp}

To comprehensively evaluate the performance of HRAC with different HRL implementations, we employed different HRL instances on different tasks. On discrete tasks, we used off-policy TD3~\cite{fujimoto_addressing_2018} for high-level training and on-policy A2C, the synchronous variant of A3C~\cite{mnih_asynchronous_2016}, for the low-level. On quadrupedal robot locomotion tasks, we used TD3 for both the high-level and the low-level training, following prior work~\cite{nachum_data-efficient_2018}, and discretized the goal space to $1\times1$ grids for adjacency learning. For robot arm manipulation tasks, since the action space of Fetch environments directly represents the desired difference of coordinates (on quadrupedal robot locomotion tasks, the original action space represents the torque of each joint of the ant robot), training a low-level goal-reaching policy is no longer necessary. Instead, we implement our hierarchical agent with a high-level policy network. We repeat the 4-dimension action output by the network for $K$ steps (with subgoal transitions on the first 3 dimensions) to introduce temporal abstraction. The base training algorithm is DDPG~\cite{lillicrap_continuous_2016}, which matches~\cite{plappert_multi-goal_2018}. The precision of goal space discretization is $0.1$. More implementation details, network architectures and hyperparameters are in Appendix~\ref{appsec:imp}.
\begin{figure}
    \centering
    \includegraphics[width=0.4\linewidth]{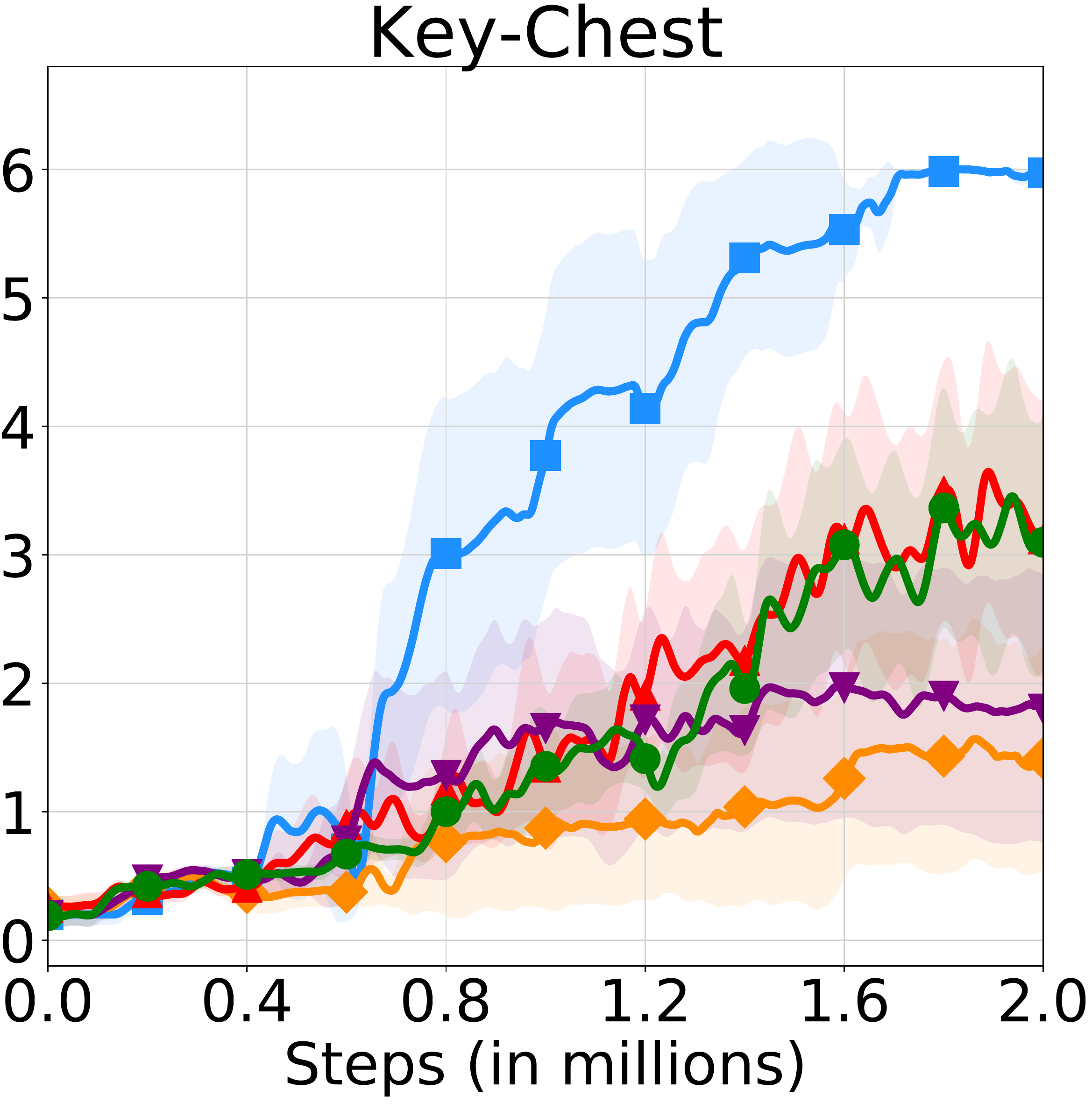}    
    \hspace{0.8em}
    \includegraphics[width=0.4\linewidth]{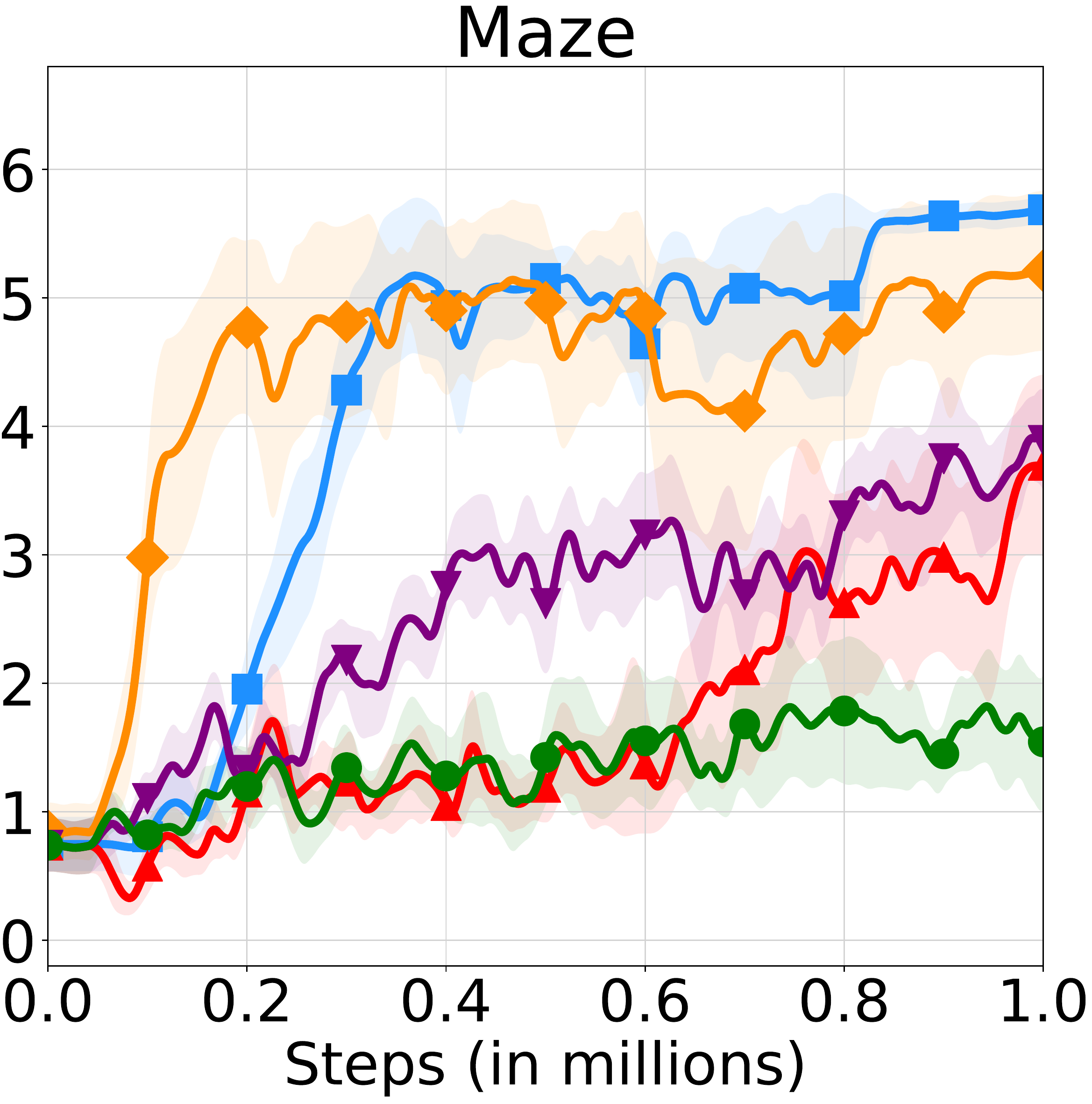} \\
    \vspace{0.8em}
    \includegraphics[width=0.99\linewidth]{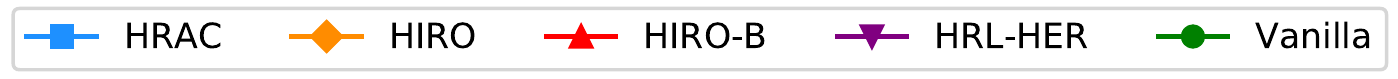}
    \caption{Learning curves of HRAC and baselines on discrete control tasks.}
    \label{fig:discrete_results}
\end{figure}

\begin{figure}
    \centering
    \includegraphics[width=0.40\linewidth]{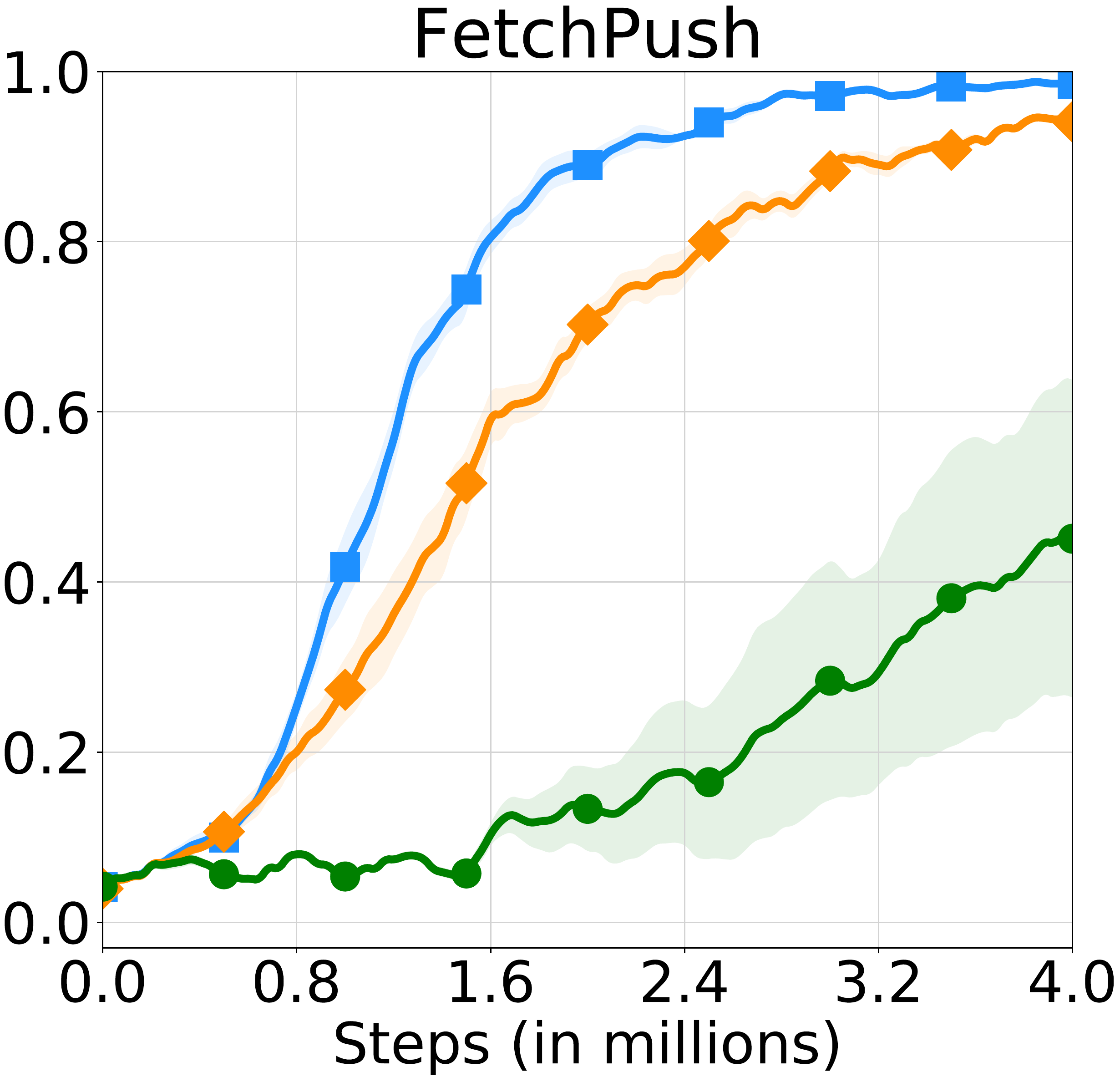}    
    \hspace{0.5em}
    \includegraphics[width=0.40\linewidth]{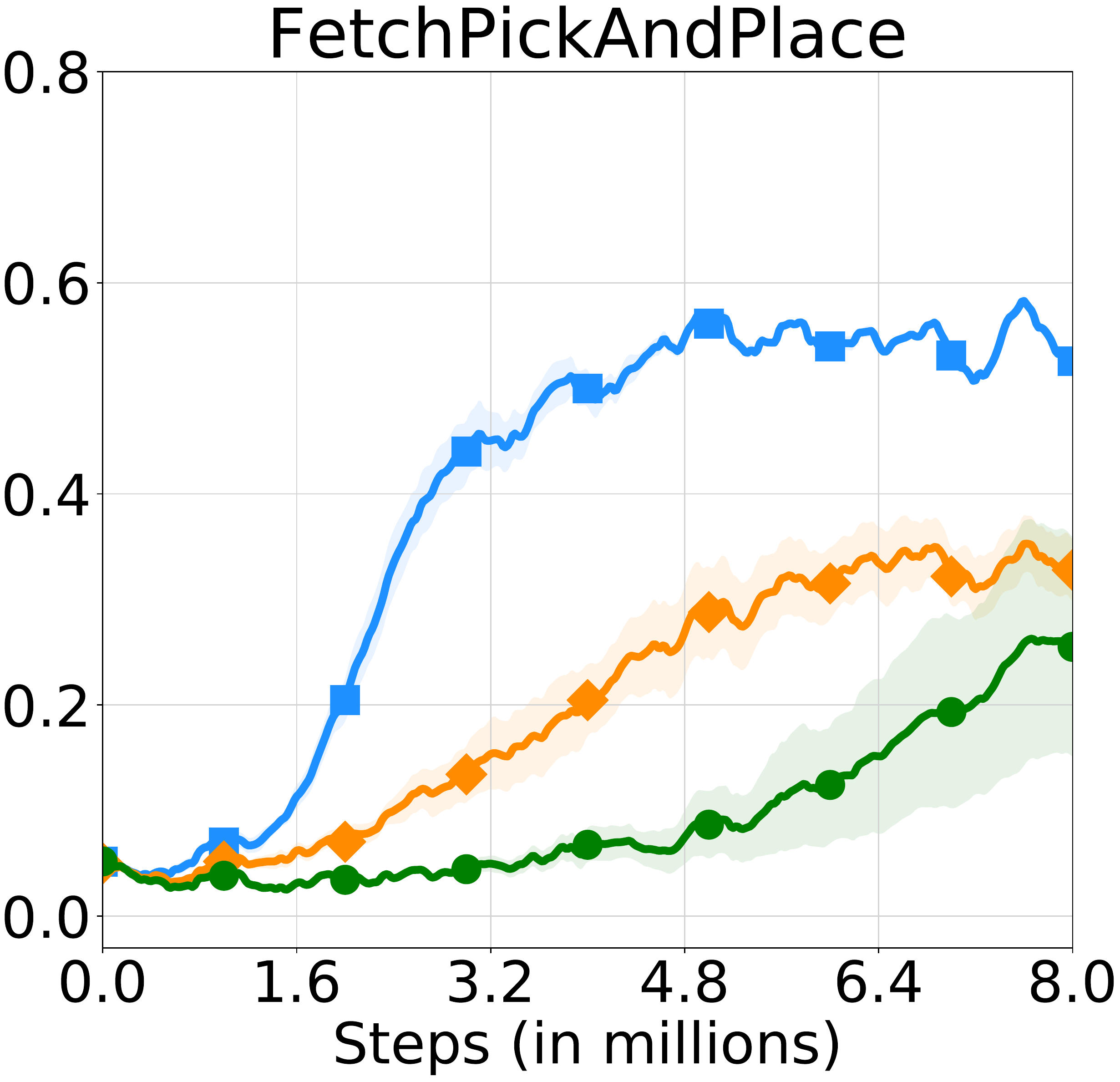}\\
    \vspace{0.5em}
    \includegraphics[width=0.55\linewidth]{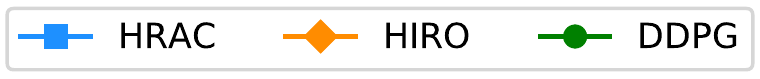}
    \caption{Learning curves of HRAC and baselines on manipulation tasks.}
    \label{fig:fetch_results}
\end{figure}
\subsubsection{Discrete Control Tasks}

We compared HRAC with the following baselines.

\textit{HIRO~\cite{nachum_data-efficient_2018}:} one of the state-of-the-art goal-conditioned HRL approaches.
By limiting the range of directional subgoals generated by the high-level, HIRO can roughly control the Euclidean distance between the absolute subgoal and the current state in the raw goal space rather than the learned adjacency space.

\textit{HIRO-B:} a baseline analogous to HIRO, using binary intrinsic reward for subgoal reaching instead of the shaped reward used by HIRO.

\textit{HRL-HER:} a baseline that employs hindsight experience replay (HER)~\cite{andrychowicz_hindsight_2017} to produce alternative successful subgoal-reaching experiences as complementary low-level learning signals~\cite{levy_learning_2019}.

\textit{Vanilla:} Kulkarni et al.~\cite{kulkarni_hierarchical_2016} used absolute subgoals instead of directional subgoals and adopted a binary intrinsic reward setting.

The learning curves of HRAC and baselines across all tasks are plotted in Fig.~\ref{fig:discrete_results}. Each curve and its shaded region represent mean episode reward and standard error of the mean respectively, averaged over 5 independent trials. All curves have been smoothed equally for visual clarity. In the Maze task with dense rewards, HRAC achieves comparable performance with HIRO and outperforms other baselines, while in the Key-Chest task HRAC achieves better convergency speed and final performance. We can also identify the effect of reward shaping and hindsight: the baseline HIRO that uses reward shaping achieves a very fast convergence rate in the Maze task with dense external rewards, but its performance severely degrades in more difficult Key-Chest and Ant Maze Sparse tasks with sparse reward. Analagous phenomena can be found in the use of the hindsight strategy. Finally, HRAC consistently outperforms the vanilla baseline among all tasks, demonstrating the effectiveness of the introduced $k$-step adjacency constraint.

\subsubsection{Quadrupedal Robot Locomotion Tasks}
We considered the same baselines as discrete control tasks, including HIRO, HIRO-B, HRL-HER, and Vanilla.

The learning curves of HRAC and baselines across all tasks are plotted in Fig.~\ref{fig:ant_results}. Each curve and its shaded region represent mean episode reward (for Ant Maze) or mean success rate (for others) and standard error of the mean respectively, averaged over 5 independent trials. All curves have been smoothed equally for visual clarity. Across all tasks, HRAC consistently surpasses all baselines in terms of both sample efficiency and asymptotic performance. We note that the performance of the baseline HRL-HER matches the results in the previous study~\cite{nachum_data-efficient_2018} where introducing hindsight techniques often degrades the performance of HRL, potentially due to the additional burden introduced on low-level training. 

\begin{figure}
    \centering
    \includegraphics[width=0.4\linewidth]{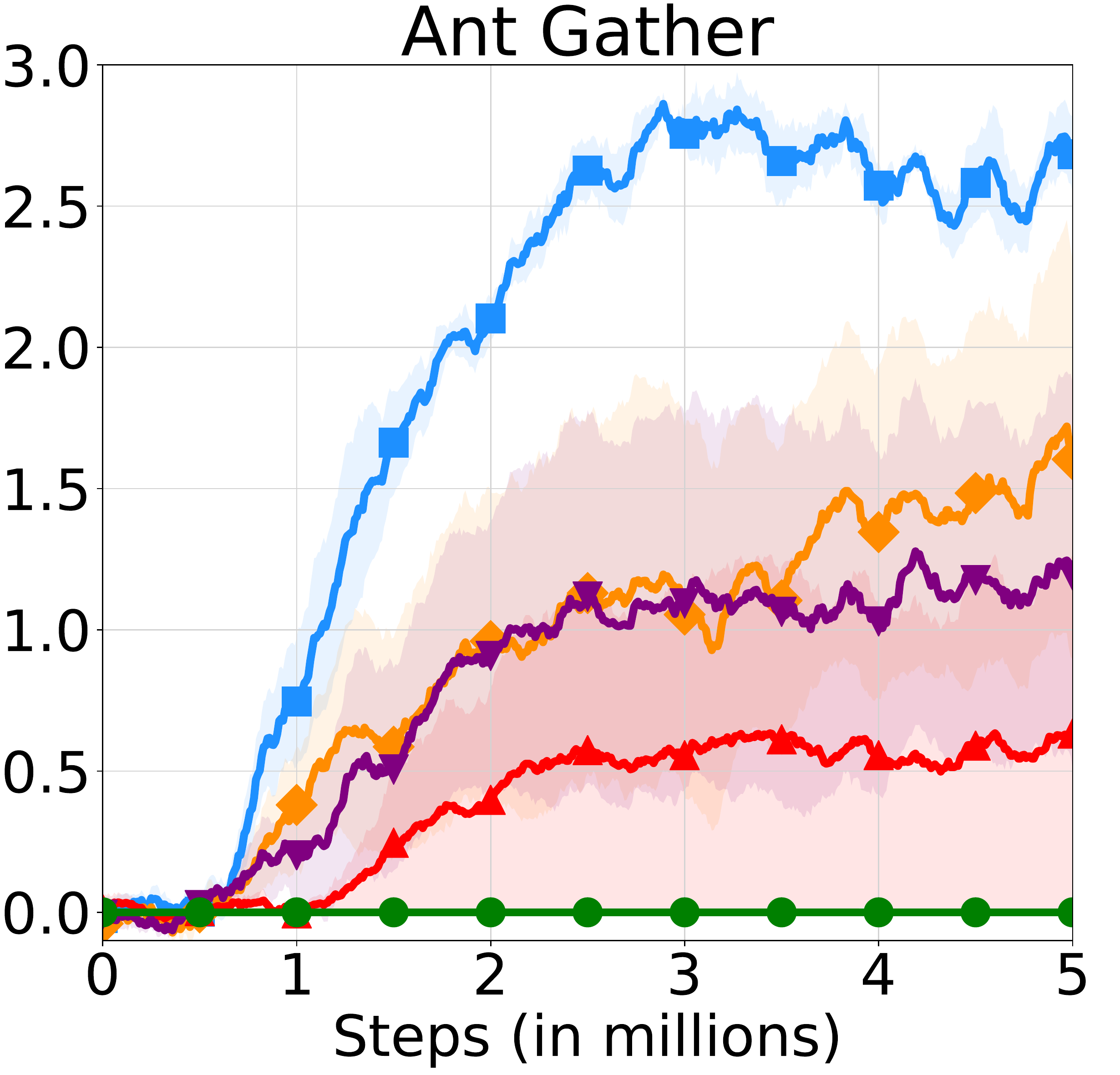}
    \hspace{0.8em}
    \includegraphics[width=0.4\linewidth]{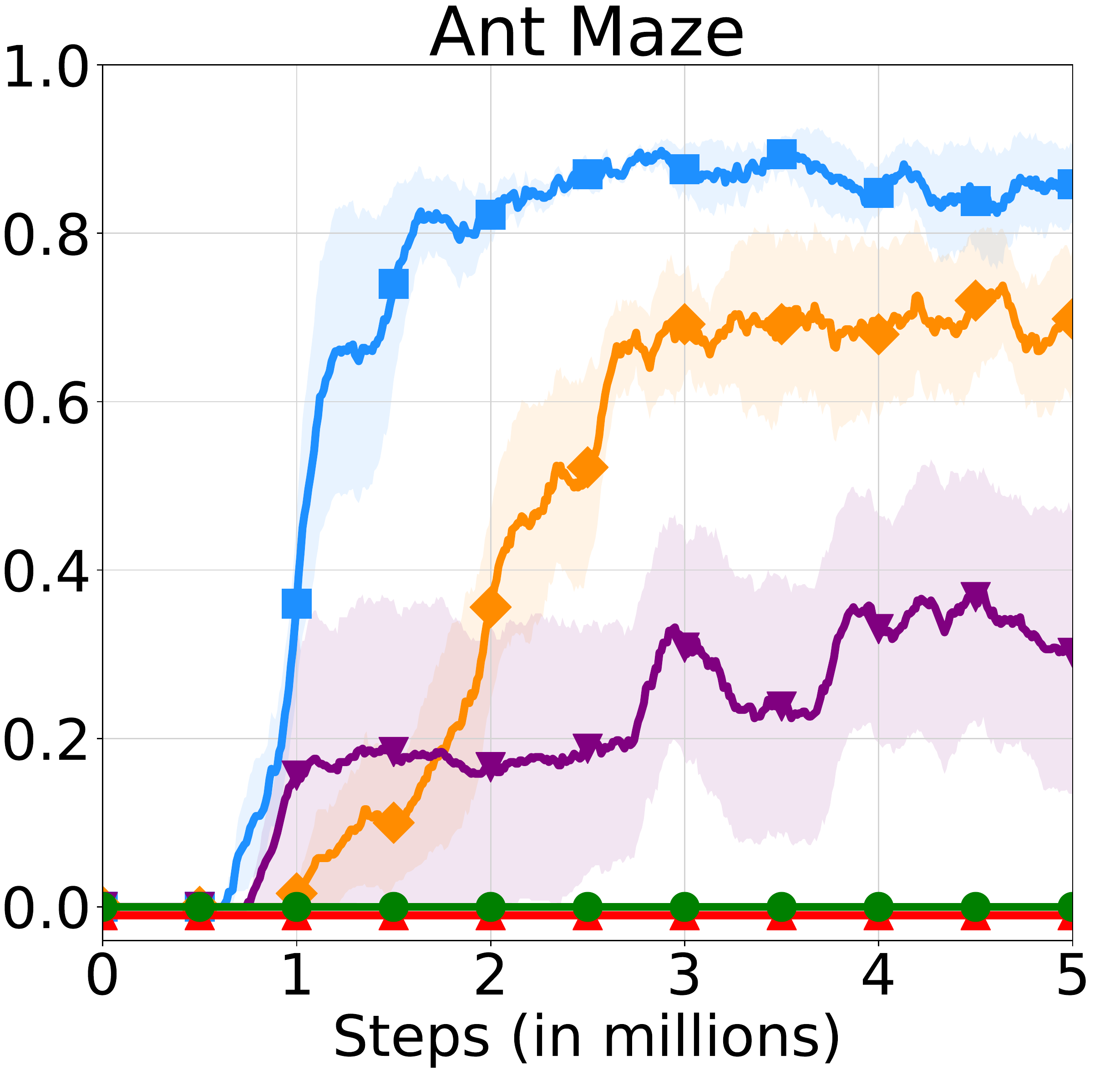}\\
    \includegraphics[width=0.4\linewidth]{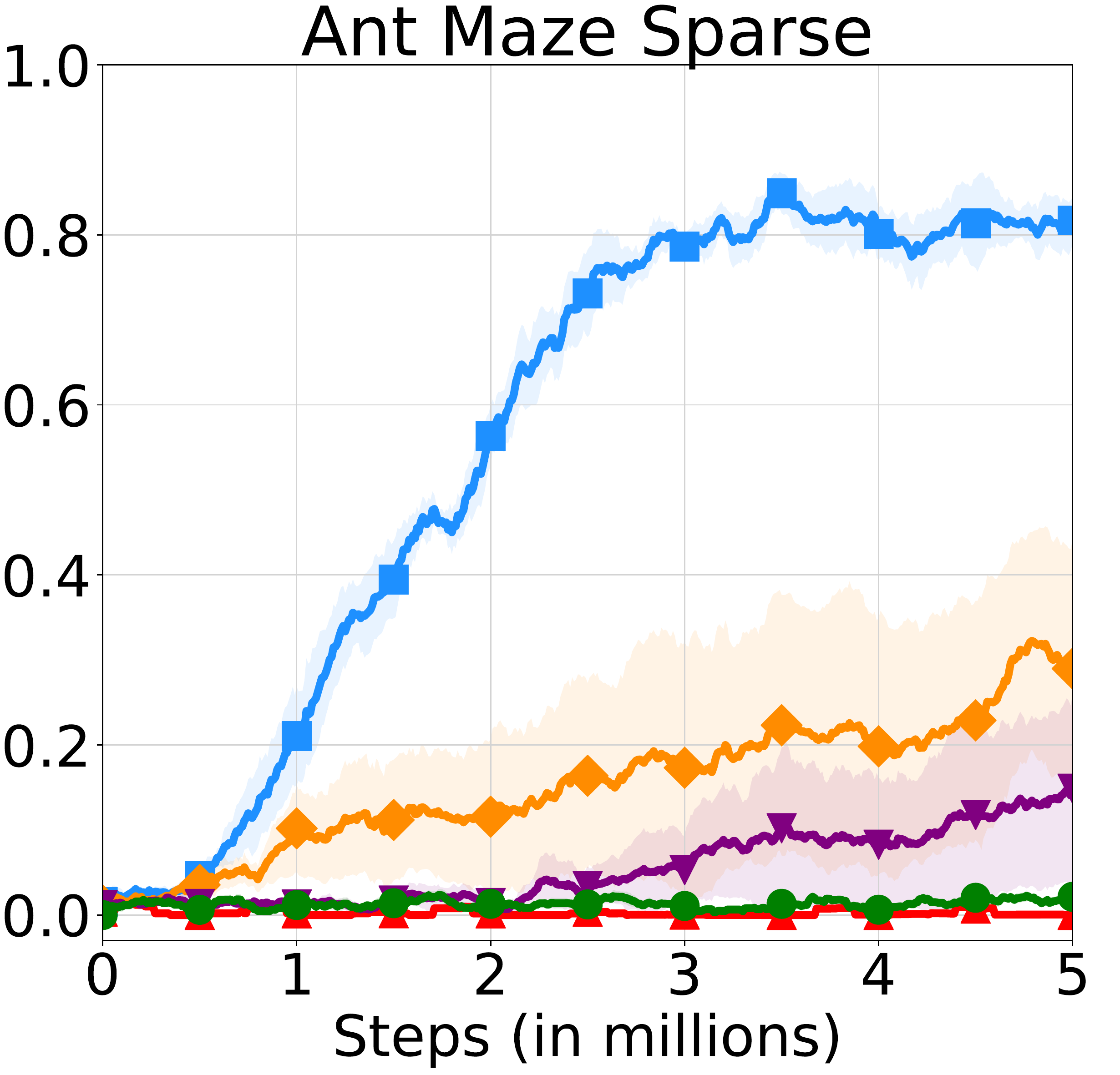}
    \hspace{0.8em}
    \includegraphics[width=0.4\linewidth]{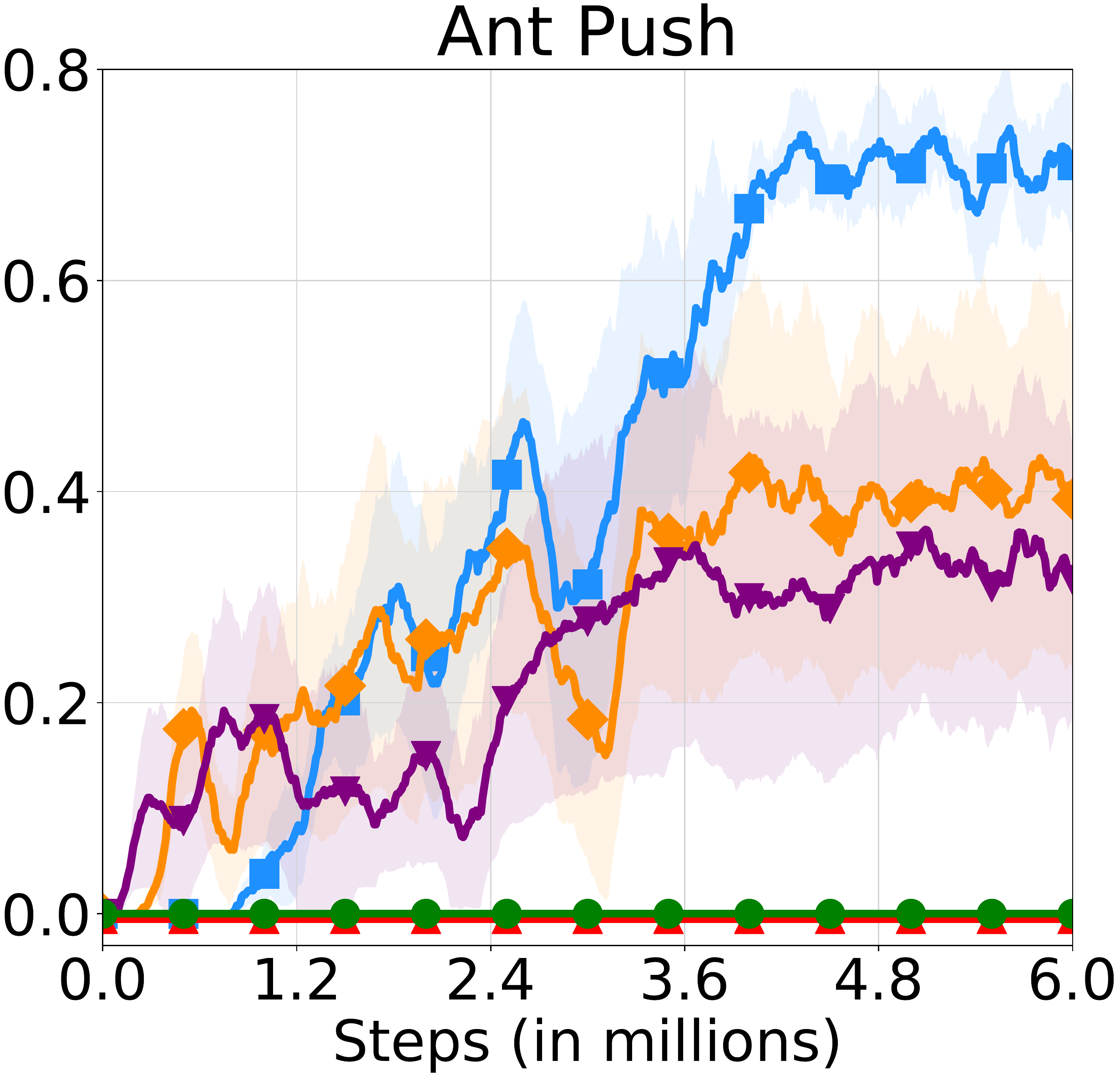}\\
    \vspace{0.8em}
    \includegraphics[width=0.99\linewidth]{legend.pdf}
    \caption{Learning curves of HRAC and baselines on locomotion tasks.}
    \label{fig:ant_results}
\end{figure}

\subsubsection{Robot Arm Manipulation Tasks}
We compare HRAC with HIRO and DDPG~\cite{lillicrap_continuous_2016}, of which the latter is a common, non-hierarchical reinforcement learning method in this task suite. For all methods, we do not use the (high-level) HER technique to better underline the efficacy of policy hierarchy in sparse reward tasks.

\begin{figure}
\centering
\includegraphics[width=0.31\linewidth]{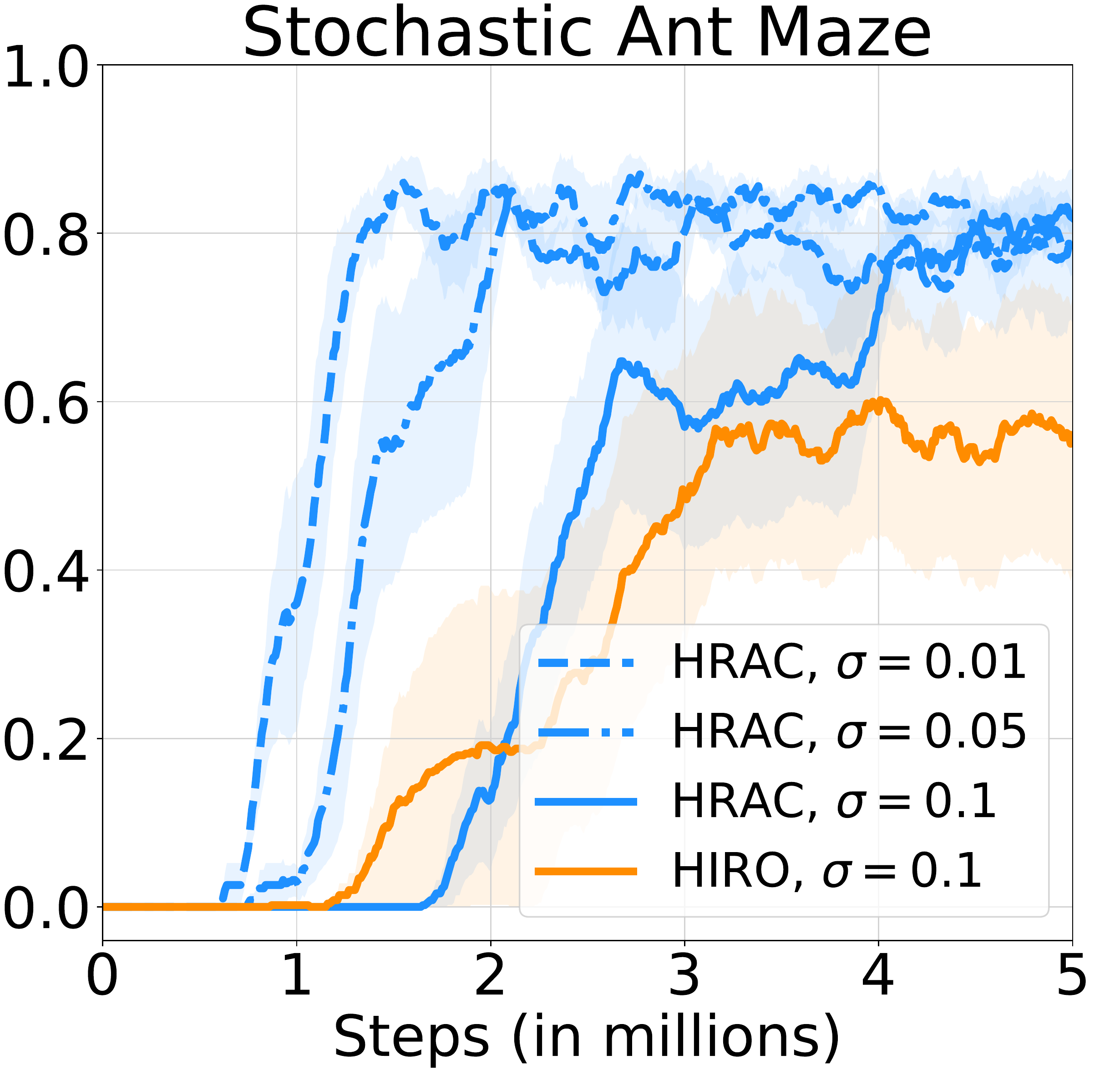}
\hspace{0.4em}
\includegraphics[width=0.31\linewidth]{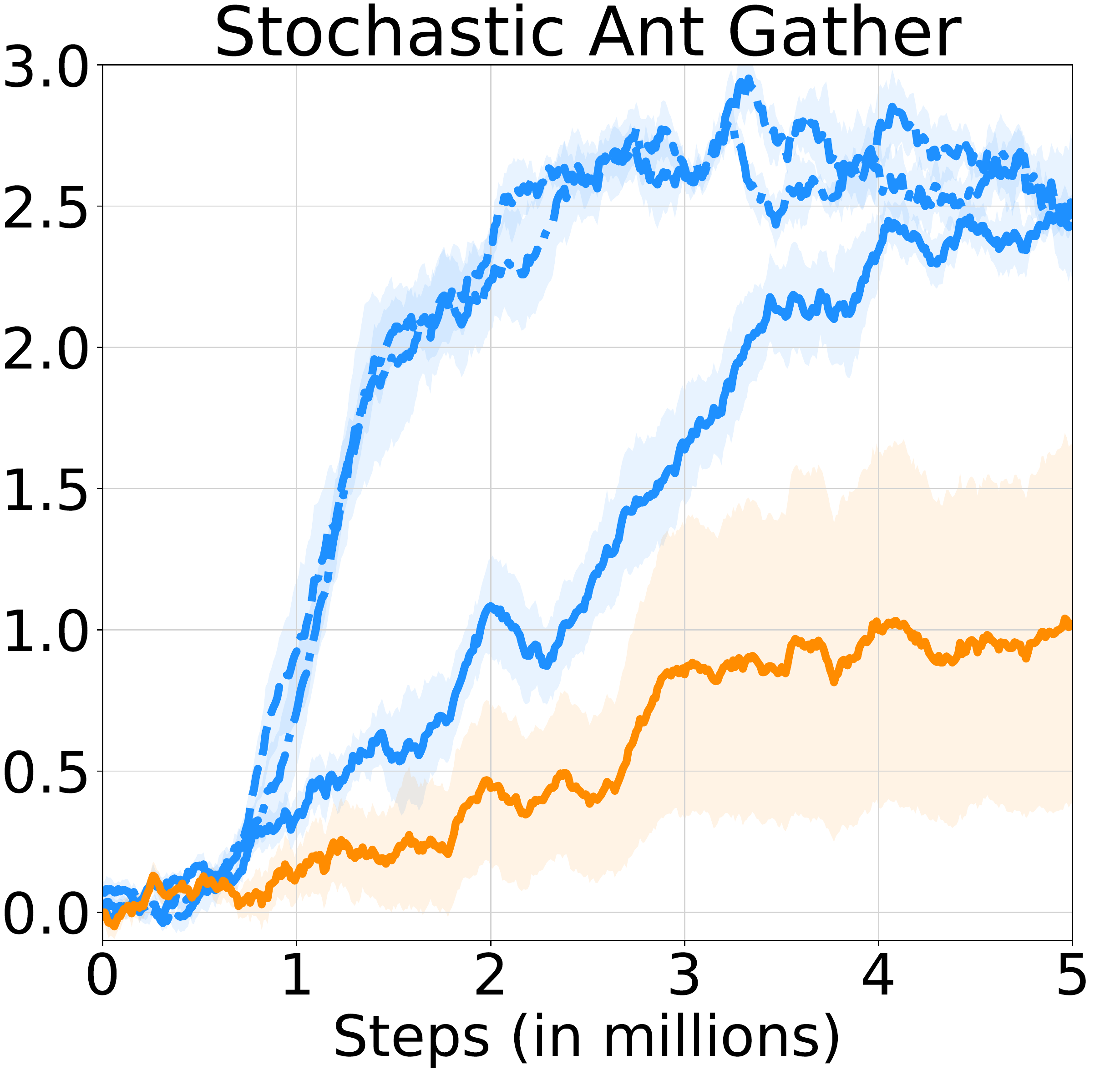}
\hspace{0.4em}
\includegraphics[width=0.31\linewidth]{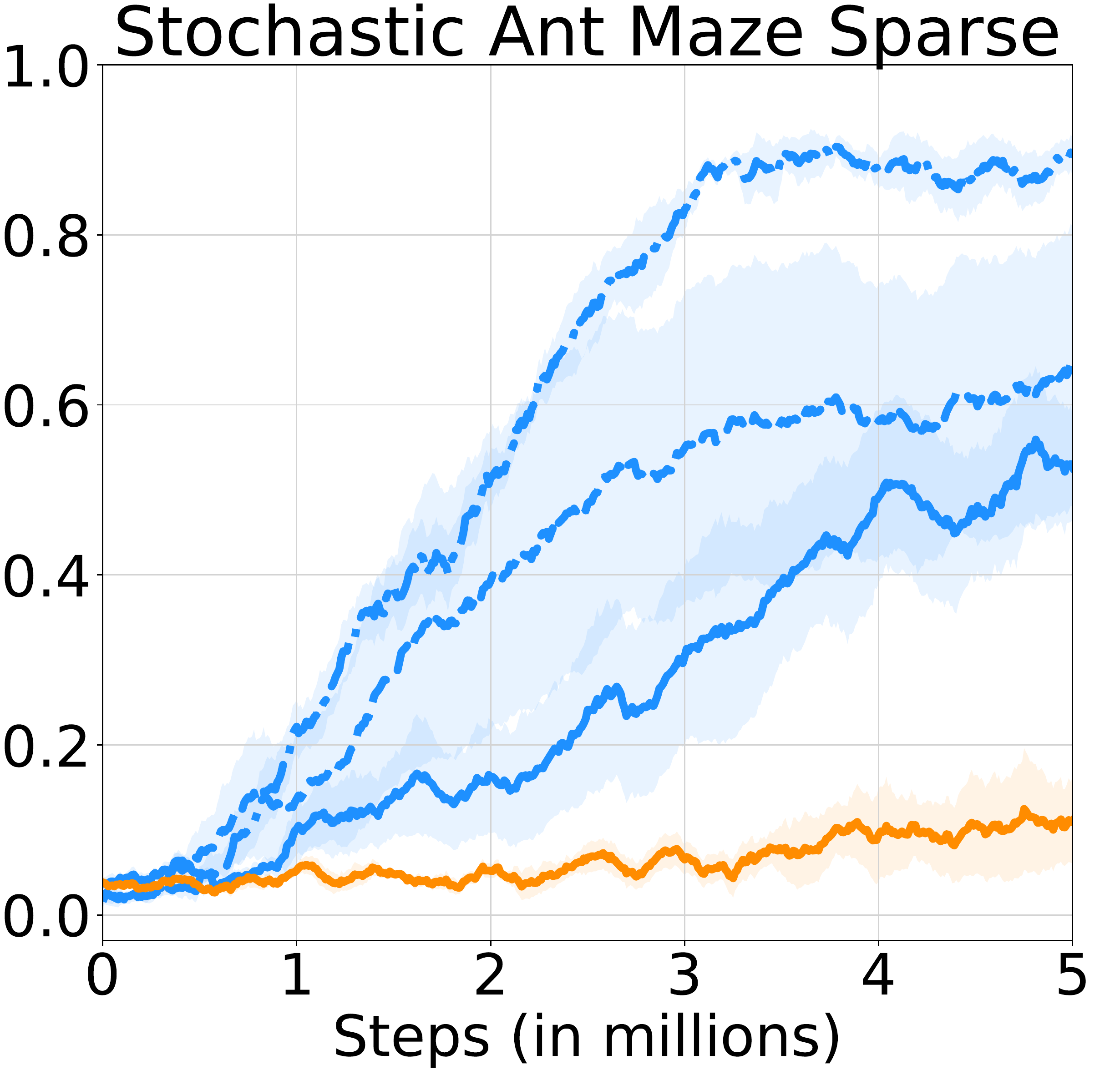}
\caption{Learning curves in stochastic Ant Maze, stochastic Ant Gather, and stochastic Ant Maze Sparse environments.}
\label{fig:stochastic}
\end{figure}

\begin{figure}[t]
\centering
\includegraphics[width=0.3\linewidth]{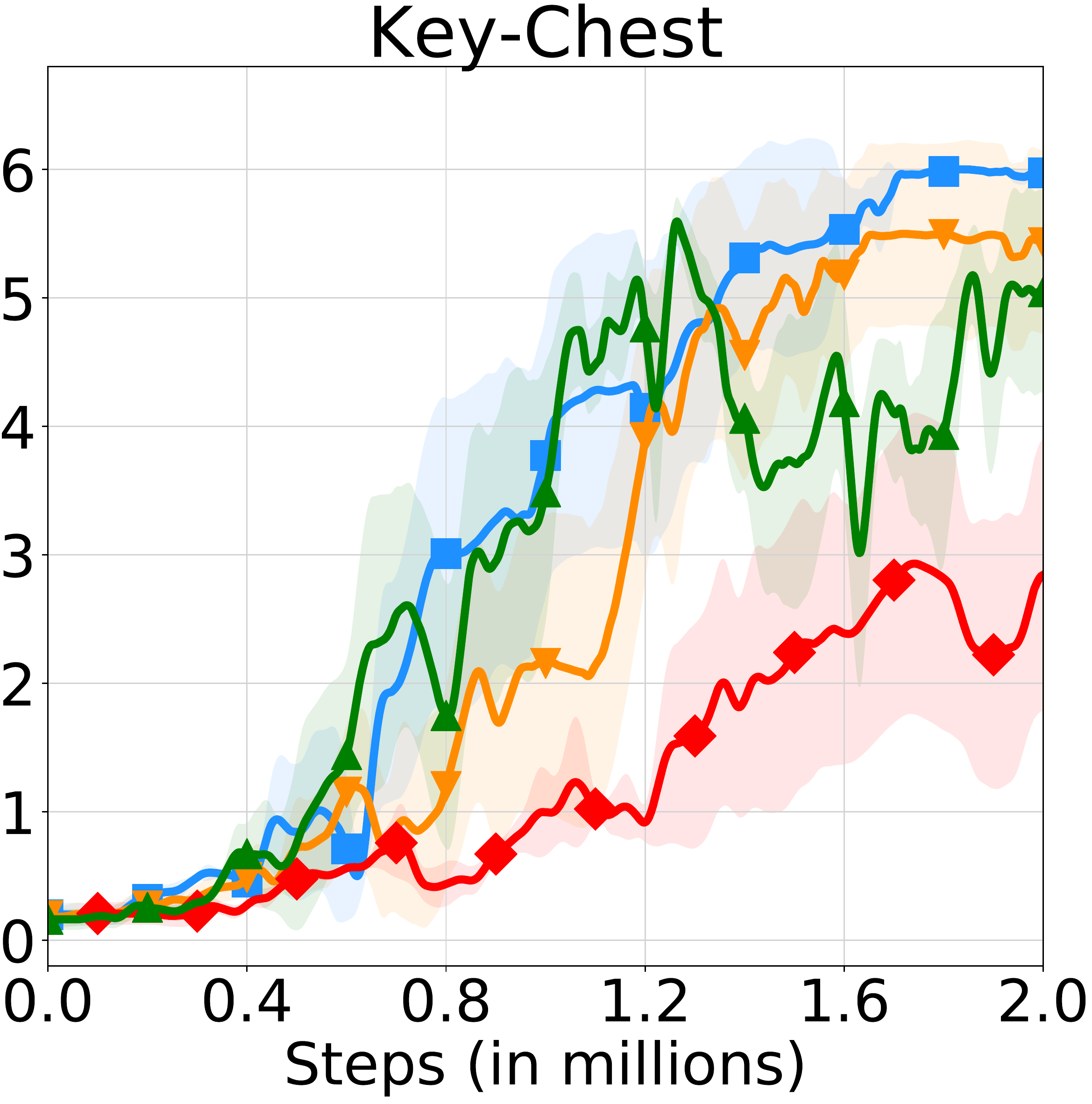}
\hspace{0.1em}
\includegraphics[width=0.3\linewidth]{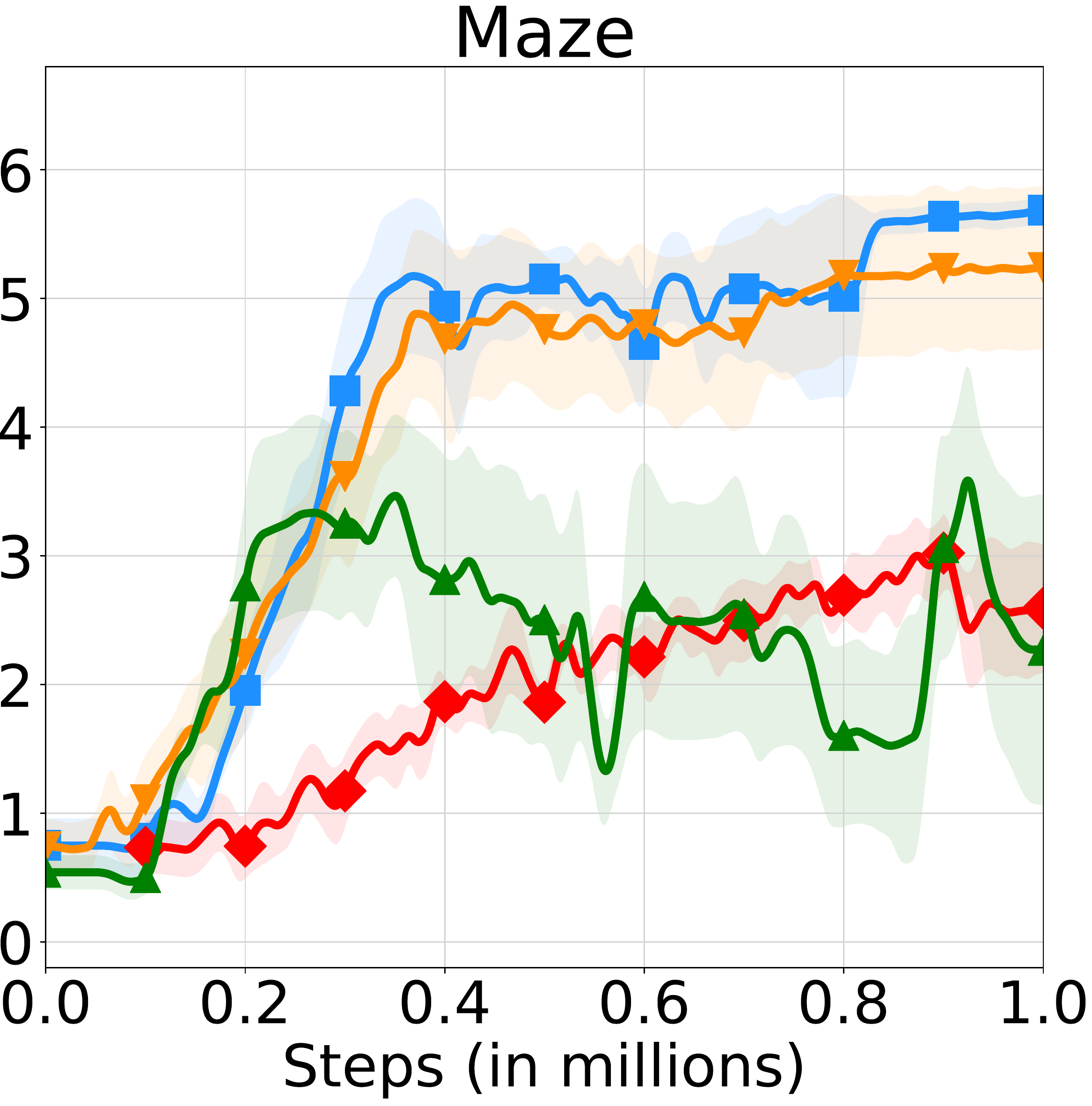}
\hspace{0.1em}
\includegraphics[width=0.31\linewidth]{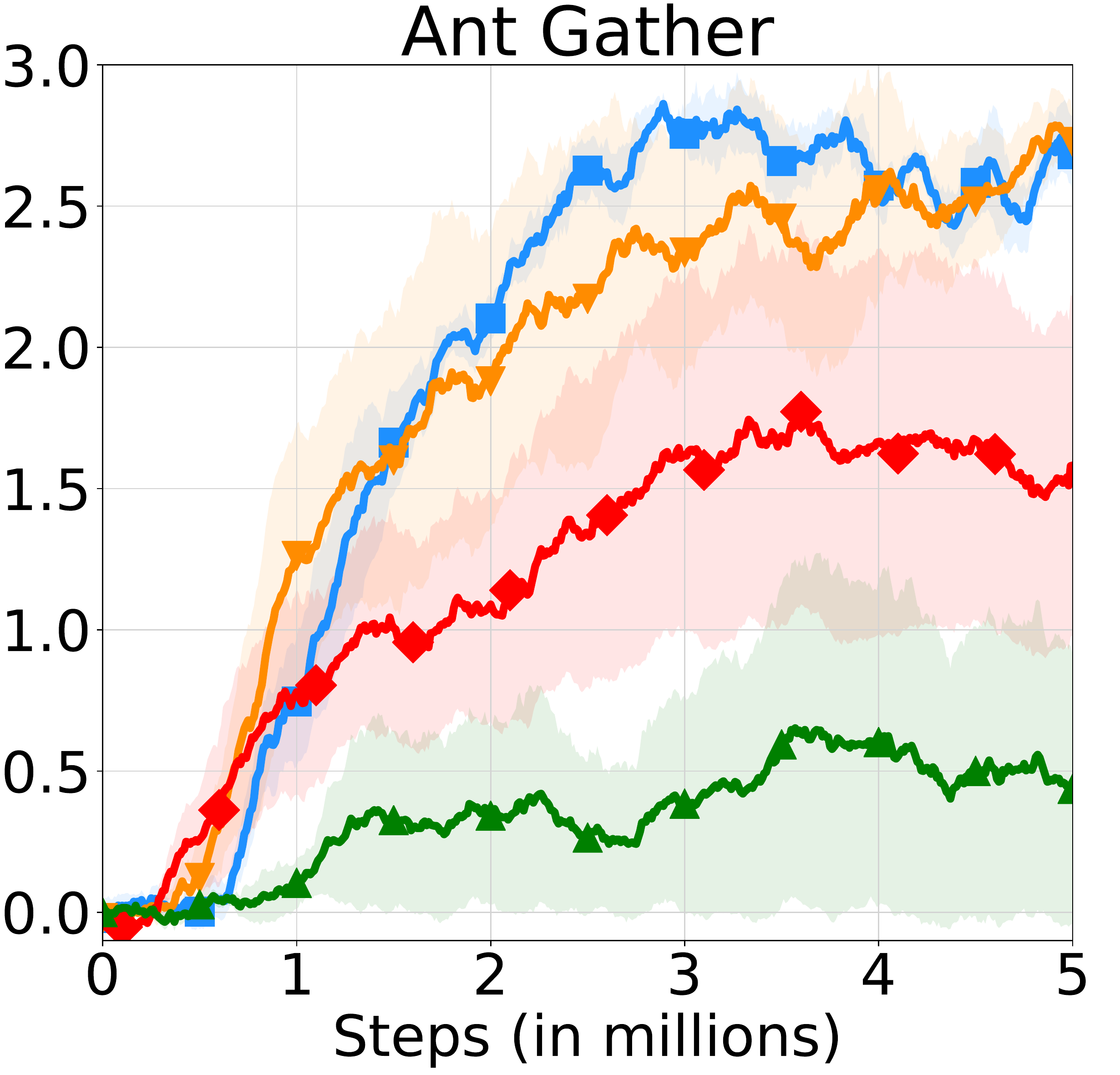} \\
\vspace{0.5em}
\includegraphics[width=0.85\linewidth]{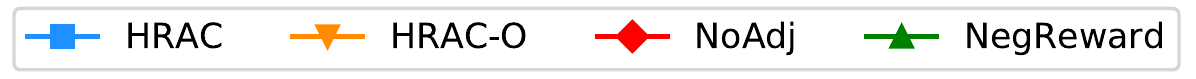}
\caption{Learning curves in the ablation study.}
\label{fig:ablation}
\end{figure}

\begin{figure}[t]
\centering
\includegraphics[width=0.4\linewidth]{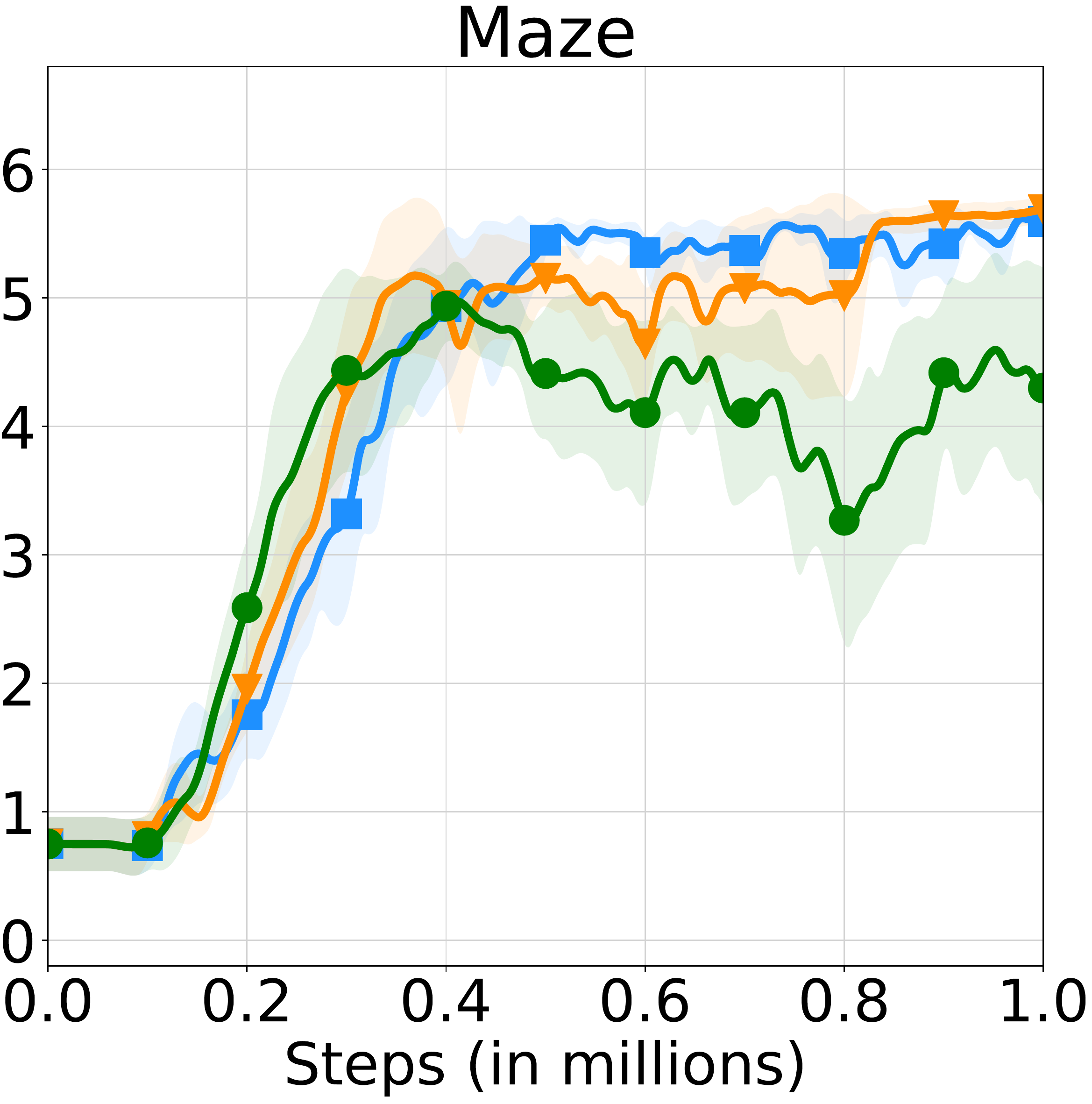}
\hspace{0.5em}
\includegraphics[width=0.41\linewidth]{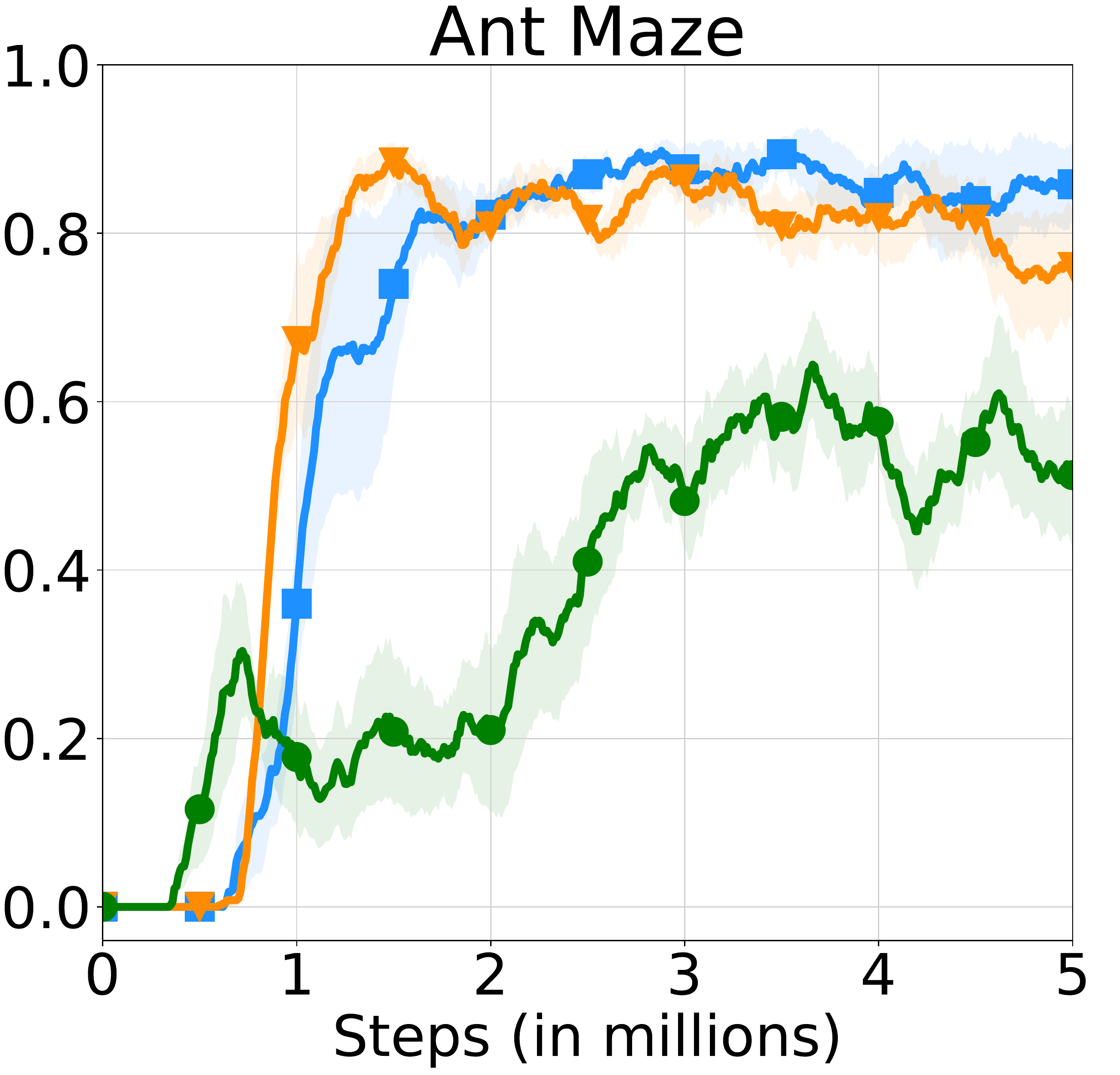} \\
\vspace{0.5em}
\includegraphics[width=0.6\linewidth]{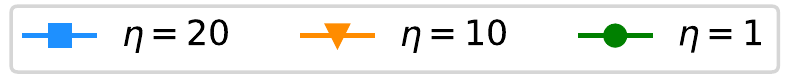}
\caption{Learning curves with different balancing coefficients.}
\label{fig:param}
\end{figure}

The learning curves of HRAC and baselines across all tasks are plotted in Fig.~\ref{fig:fetch_results}. Each curve and its shaded region represent mean success rate and standard error of the mean respectively, averaged over 5 independent trials. All curves have been smoothed equally for visual clarity. In both tasks, HRAC consistently surpasses HIRO and DDPG. Note that HIRO also outperforms DDPG, which indicates the superiority of HRL in better exploration due to the introduced temporal abstraction.

\subsection{Ablation Study}
\label{subsec:ablation}
We also compared HRAC with several variants to investigate the effectiveness of each component in our method.

\textit{HRAC-O:} An oracle variant that uses a perfect adjacency matrix directly obtained from the environment. We note that compared to other methods, this variant uses additional information that is not available in many applications

\textit{NoAdj:} A variant that uses an adjacency training method analogous to the work of Savinov et al.~\cite{savinov_semi-parametric_2018,savinov_episodic_2019}, where no adjacency matrix is maintained. The adjacency network is trained using state pairs directly sampled from stored trajectories, under the same training budget as HRAC.

\textit{NegReward:} This variant implements the $k$-step adjacency constraint by penalizing the high-level with a negative reward when it generates non-adjacent subgoals, which is used by HAC~\cite{levy_learning_2019}.

We provide learning curves of HRAC and these variants in Fig.~\ref{fig:ablation}. In all tasks, HRAC yields similar performance with the oracle variant HRAC-O while surpassing the NoAdj variant by a large margin, exhibiting the effectiveness of our adjacency learning method. Meanwhile, HRAC achieves better performance than the NegReward variant, suggesting the superiority of implementing the adjacency constraint using a differentiable adjacency loss, which provides stronger supervision than a reward-based penalty. We also empirically studied the effect of different balancing coefficients $\eta$. Results are shown in Fig.~\ref{fig:param}, which suggests that generally, a large $\eta$ can lead to better and more stable performance.

\subsection{Empirical Study in Stochastic Environments with Additive State Noise}

To empirically verify the robustness of HRAC, we applied it to a set of stochastic tasks, including stochastic Ant Gather, stochastic Ant Maze, and stochastic Ant Maze Sparse. These tasks are modified from the original ant tasks respectively. Although in Section~\ref{subsec:setup} we have mentioned that discrete tasks also have inherent stochasticity implemented as random actions, the reason for additionally conducting an empirical analysis here is that the discrete environments we used are generally toy domains, which tend to be insufficient to accurately reflect the performance of the algorithm in the face of stochasticity. In contrast, the testbeds that we consider in this section are adapted from widely-used, challenging locomotion tasks, which we believe are more suitable for benchmarking.

Concretely, we added Gaussian noise with different standard deviations $\sigma$ to the $(x,y)$ position of the ant robot at every step, including $\sigma=0.01,\sigma=0.05$ and $\sigma=0.1$, representing increasing environmental stochasticity. This simulates some common real-world scenarios such as in robotics, where the stochasticity in state transitions is often induced by the noise in the environment or in the sensors and actuators of the robots~\citep{peng_sim--real_2018}. In these tasks, we compare HRAC with the baseline HIRO, which has exhibited generally better performance than other baselines, in the noisiest scenario when $\sigma=0.1$. As displayed in Fig.~\ref{fig:stochastic}, HRAC achieves similar asymptotic performances with different noise magnitudes in stochastic Ant Gather and Ant Maze tasks and consistently outperforms HIRO, exhibiting robustness to stochastic environments.

\subsection{Subgoal and Adjacency Metric Visualization}

We visualize the subgoals generated by the high-level policy and the adjacency heatmaps in Fig.~\ref{fig:visualization_maze}. The visualization indicates that the agent does learn to generate adjacent and interpretable subgoals, and verifies the effectiveness of our adjacency metric learning strategy. Additionally, in Fig.~\ref{fig:visualization_antmaze} we visualize state and subgoal distributions on the Ant Maze task by randomly sampling transitions from the replay buffer at different training stages, illustrating the impact of the adjacency constraint by comparing the subgoals generated by HRAC and HIRO. As shown in the figure, the subgoal distribution generated by HRAC is ``closer'' to the state distribution compared to HIRO without the $k$-step adjacency constraint, significantly reducing undesirable subgoals (e.g., subgoals that correspond to unreachable wall states) that hinders efficient exploration. The results suggest that thanks to the adjacency constraint, the subgoals generated by HRAC exhibit higher quality compared to HIRO and lead to better exploration and more efficient learning.

\begin{figure*}
\centering
\includegraphics[width=0.145\linewidth]{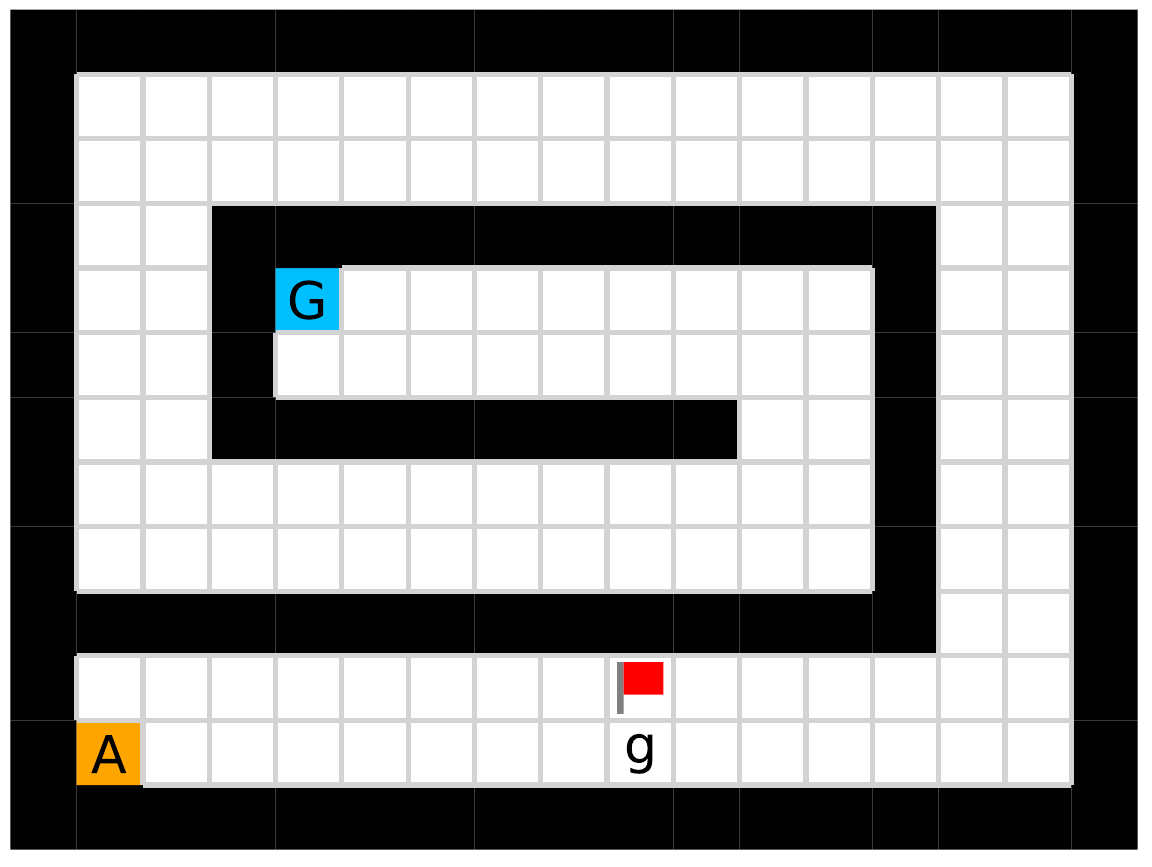}
\hspace{0.4em}
\vspace{0.5em}
\includegraphics[width=0.145\linewidth]{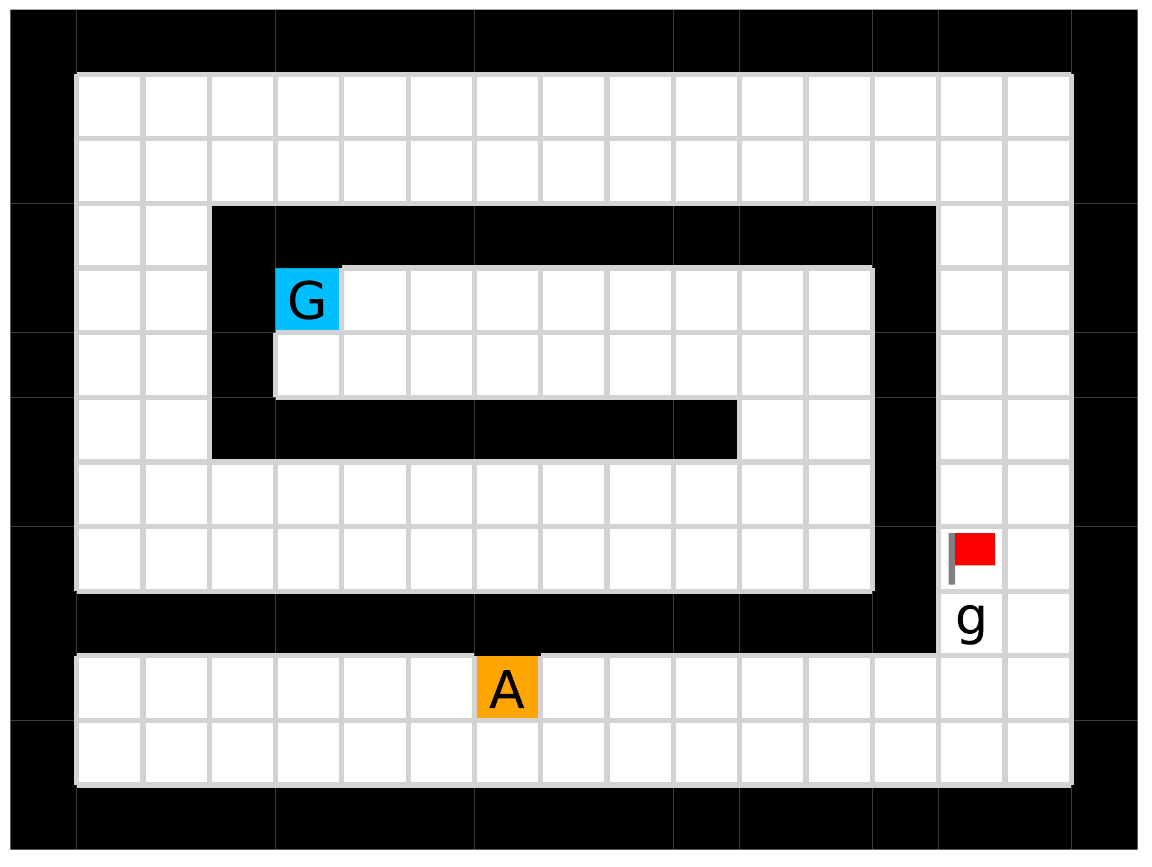}
\hspace{0.4em}
\includegraphics[width=0.145\linewidth]{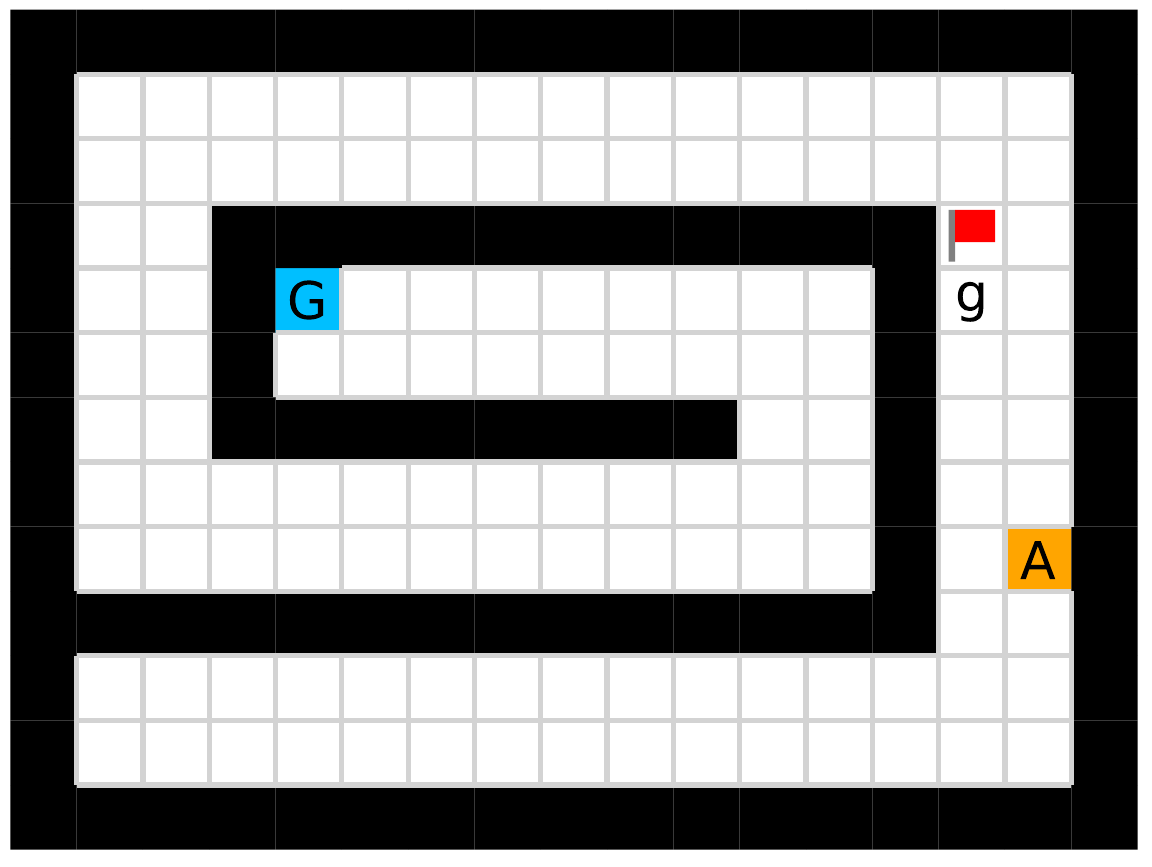}
\hspace{0.4em}
\includegraphics[width=0.145\linewidth]{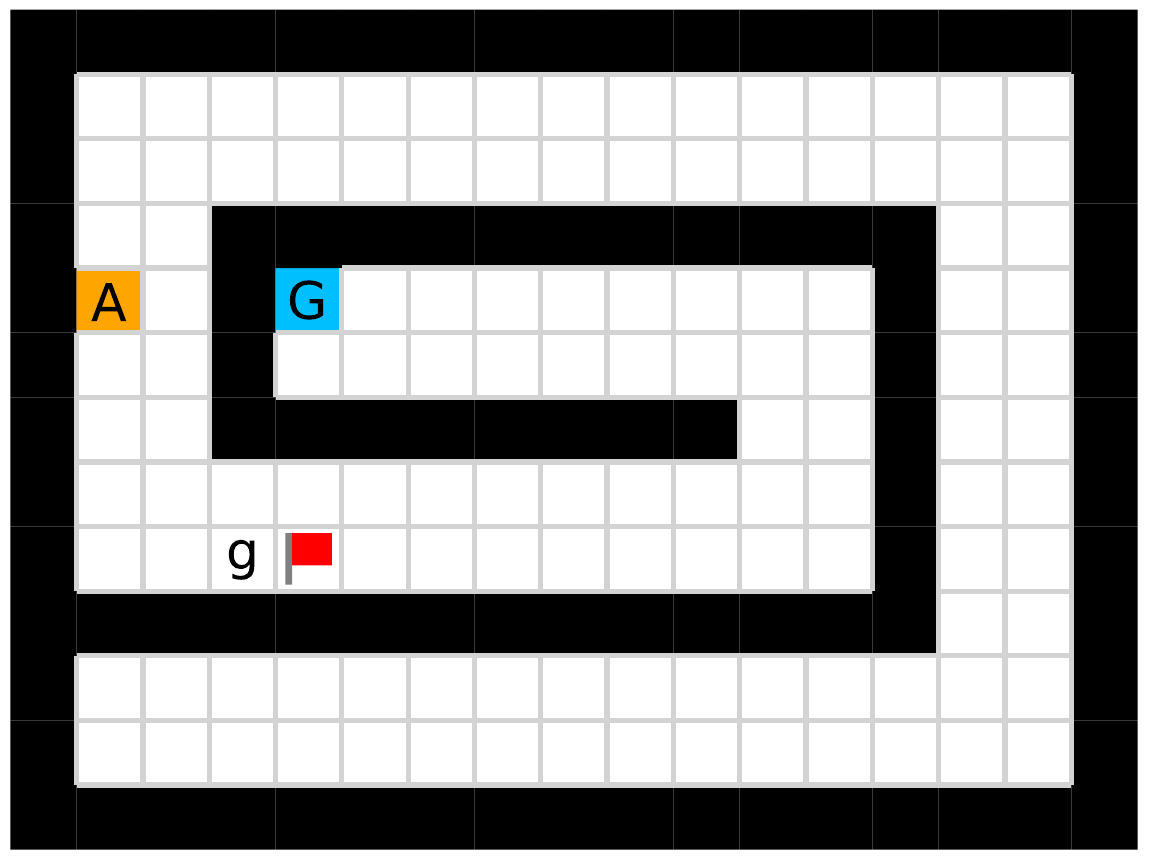}
\hspace{0.4em}
\includegraphics[width=0.145\linewidth]{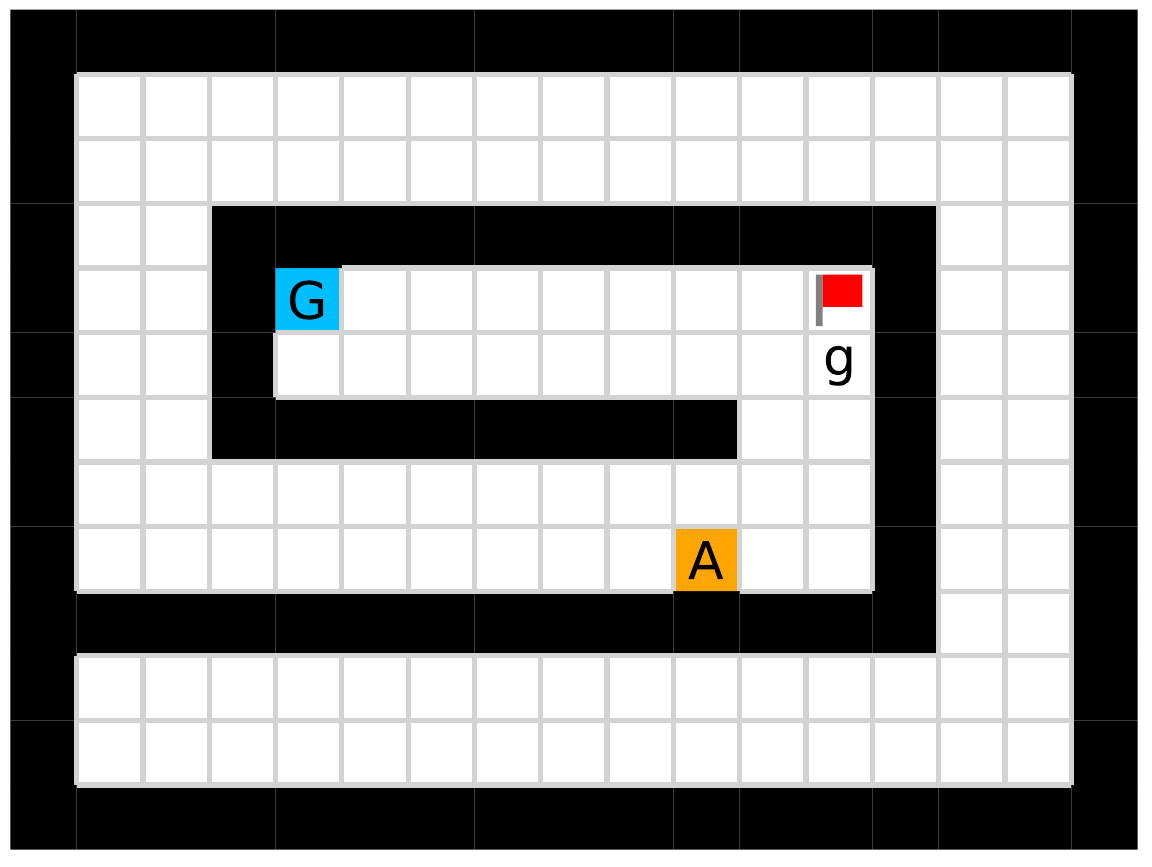}\\
\includegraphics[width=0.145\linewidth]{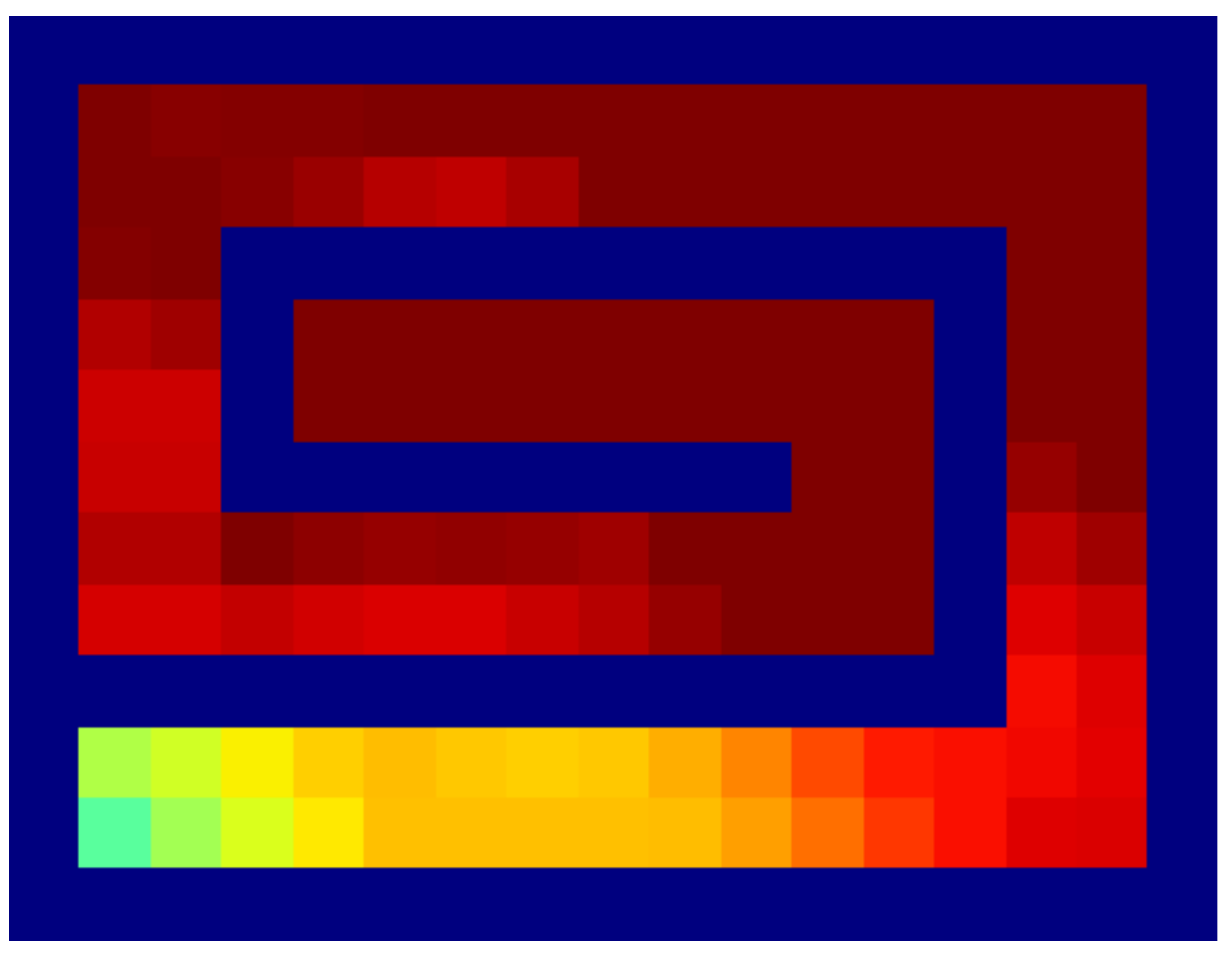}
\hspace{0.4em}
\vspace{2.5em}
\includegraphics[width=0.145\linewidth]{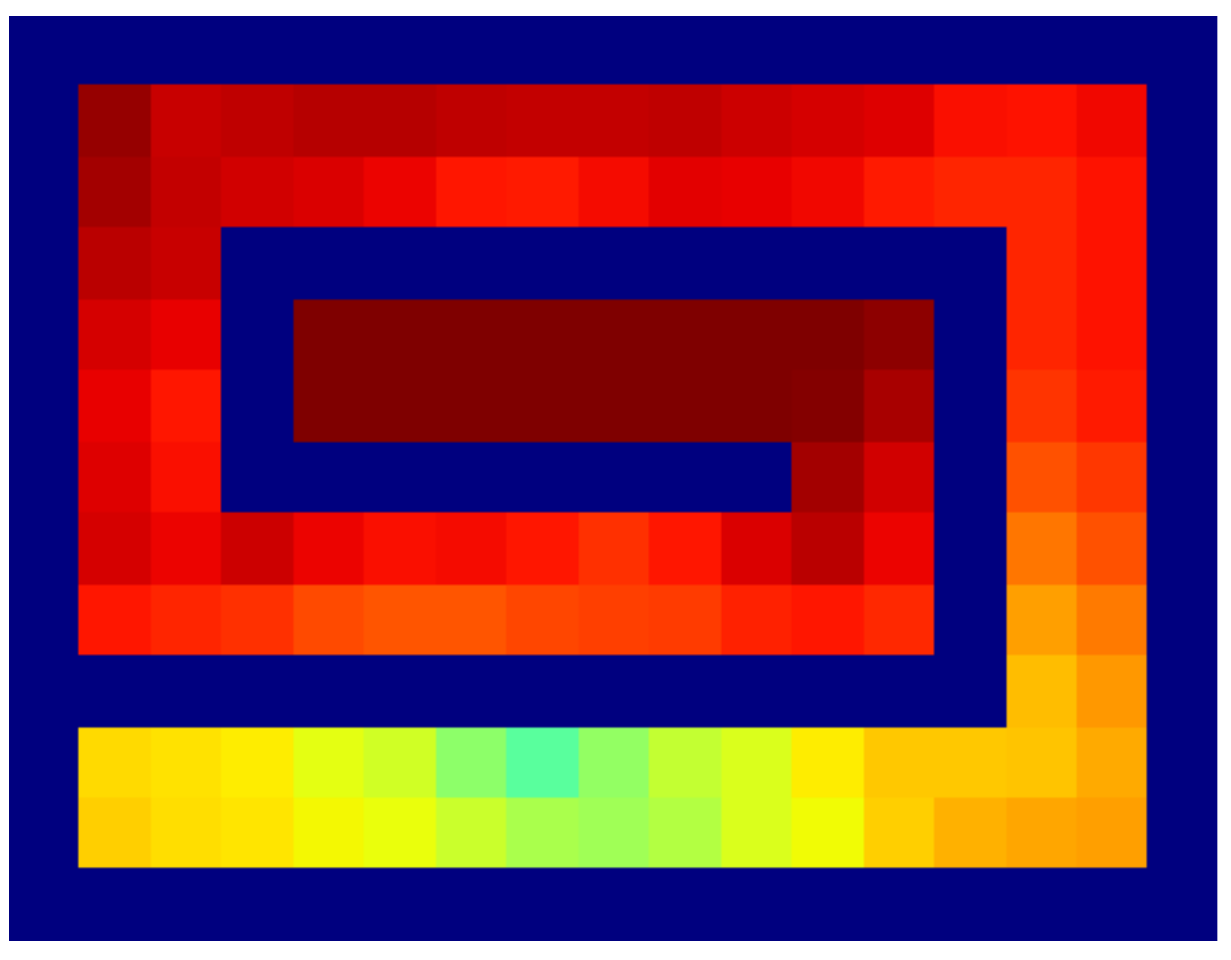}
\hspace{0.4em}
\includegraphics[width=0.145\linewidth]{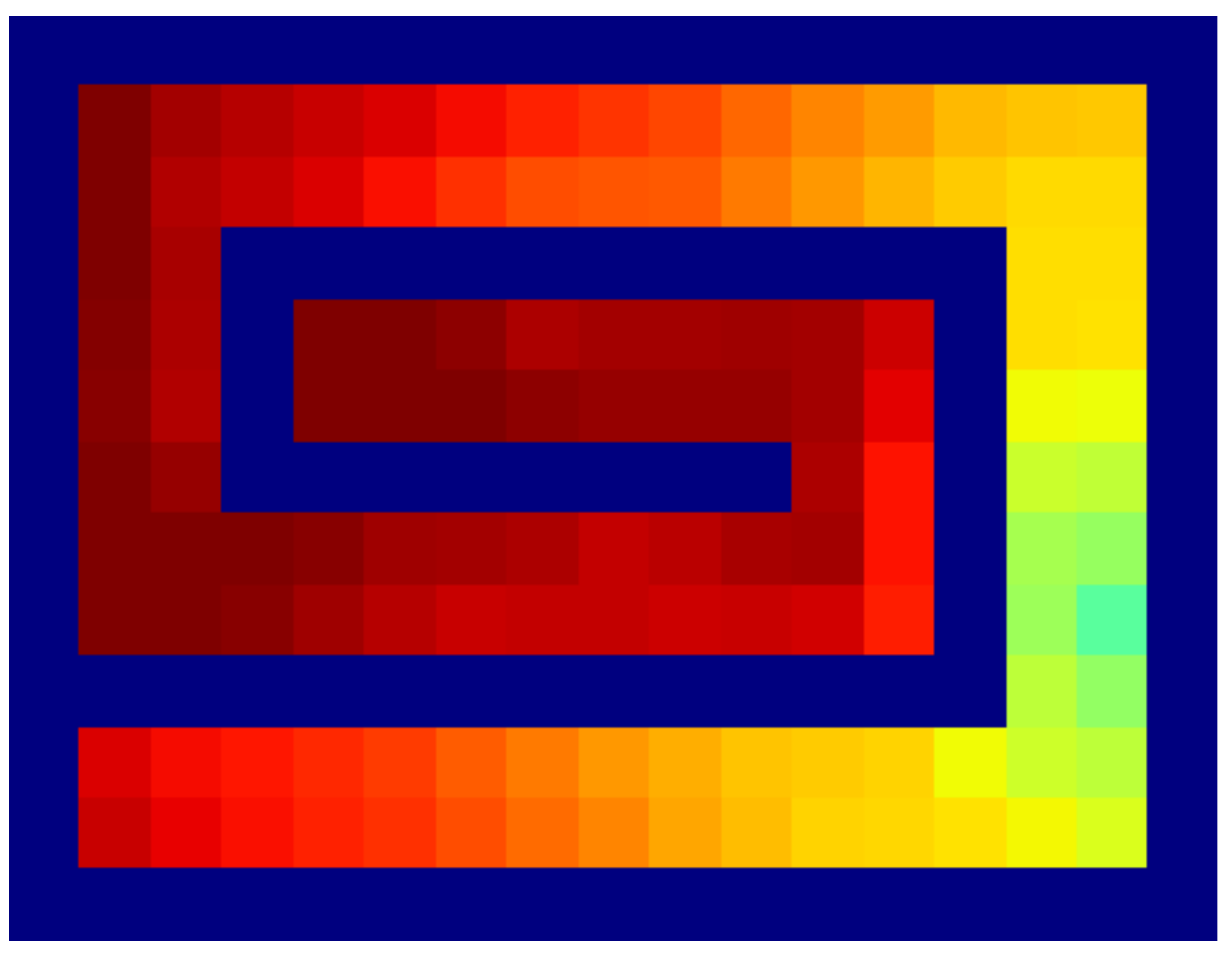}
\hspace{0.4em}
\includegraphics[width=0.145\linewidth]{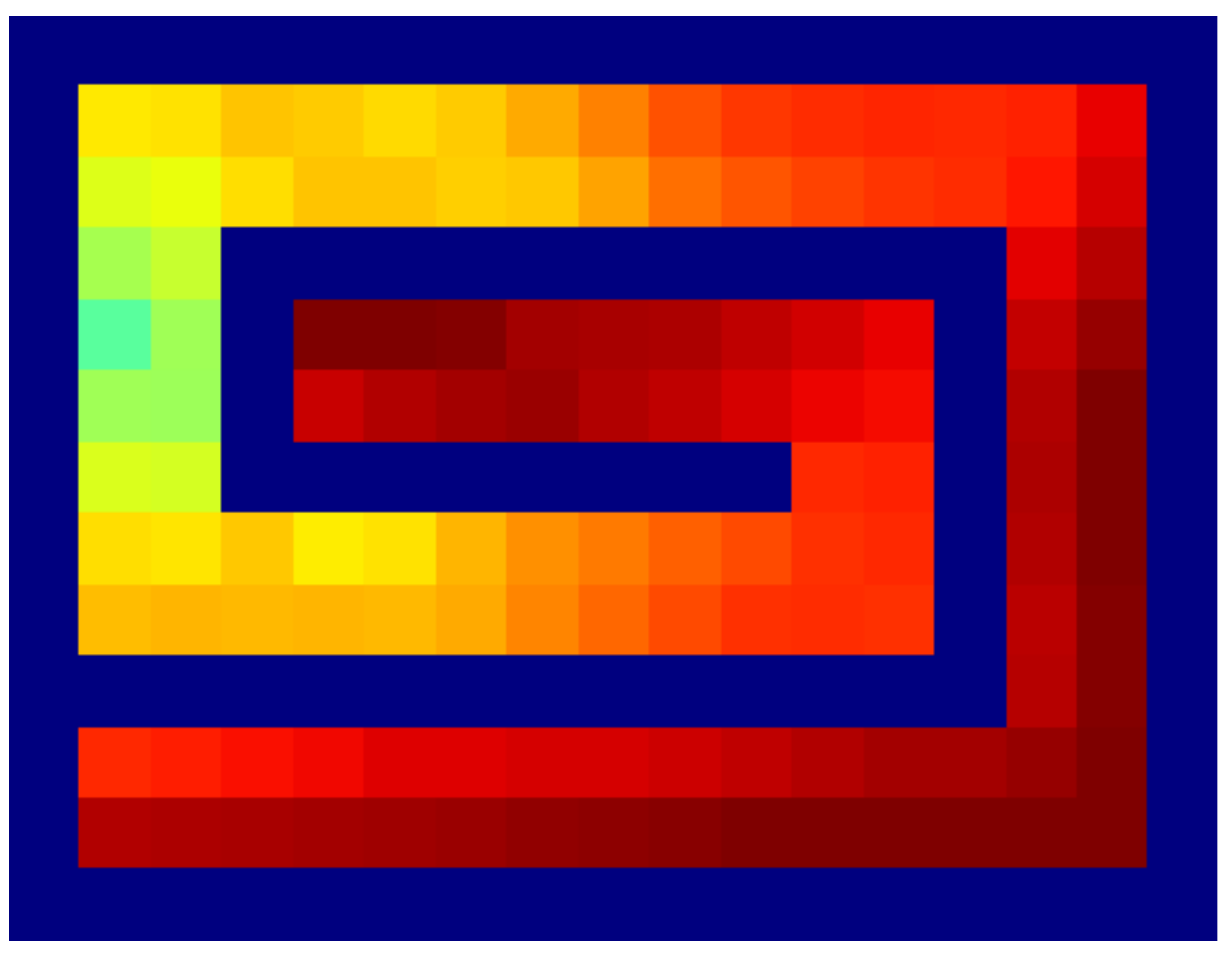}
\hspace{0.4em}
\includegraphics[width=0.145\linewidth]{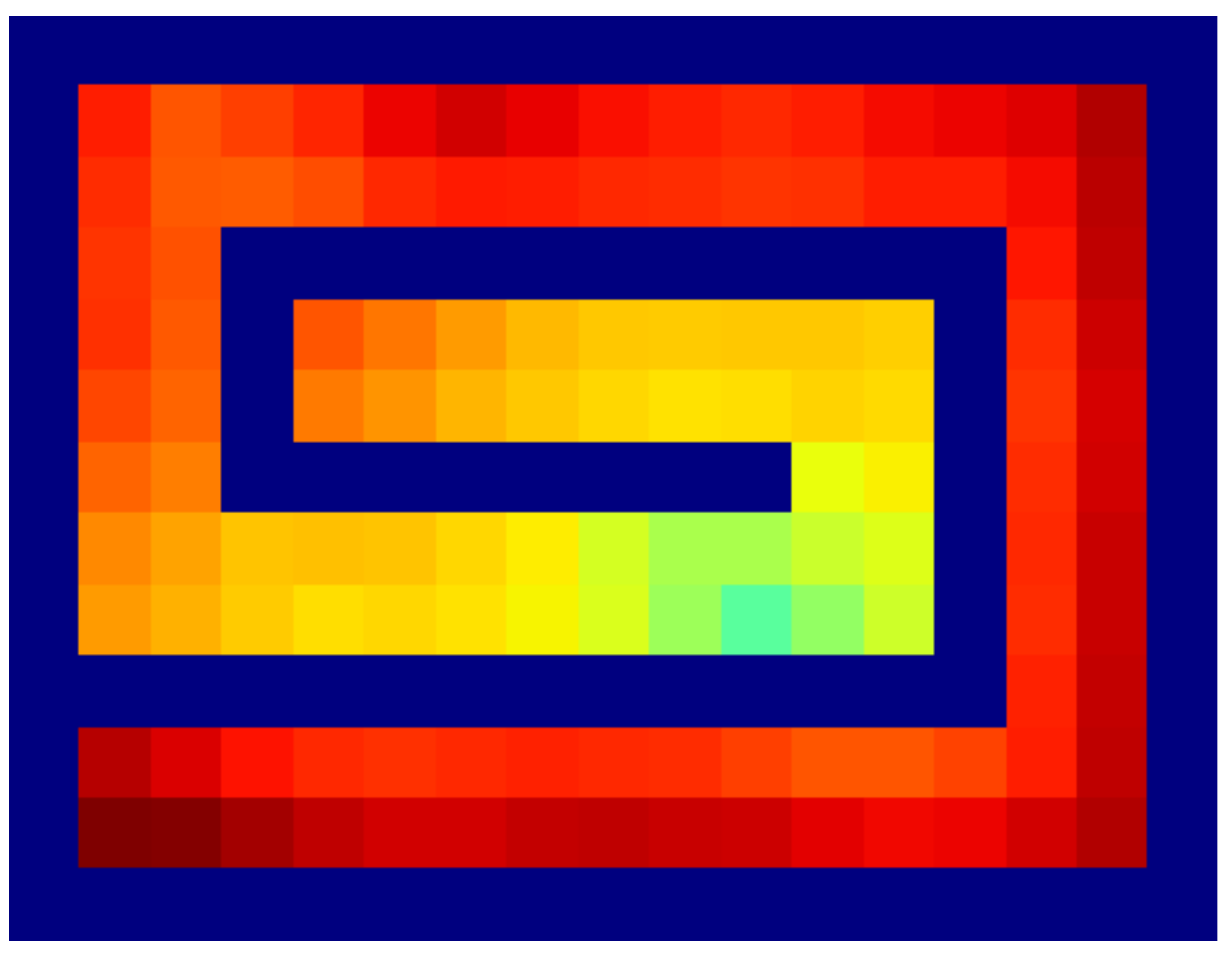}\\
\vspace{-2em}
\hspace{1.2em}\includegraphics[width=0.88\linewidth]{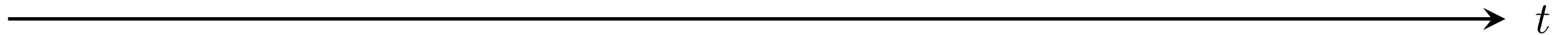}
\caption{Subgoal and adjacency heatmap visualizations of the Maze task, based on a single evaluation run. The agent (A), goal (G), and subgoal (g) at different time steps in one episode are plotted.  Colder colors in the adjacency heatmaps represent smaller shortest transition distances.}
\label{fig:visualization_maze}
\end{figure*}

\begin{figure*}
    \centering
    \includegraphics[width=0.8\linewidth]{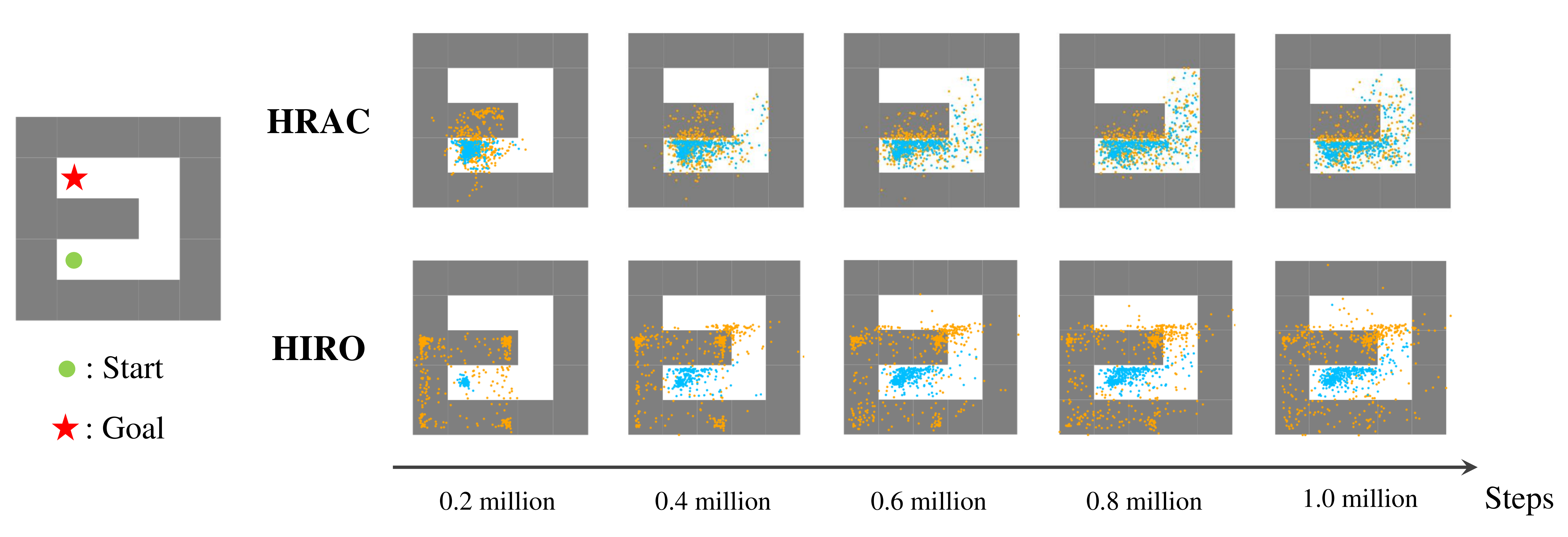}
    \caption{Subgoal Visualizations on the Ant Maze task. The state (blue) and subgoal (orange) distributions at different training periods are plotted. States and subgoals are uniformly sampled from the replay buffer.}
    \label{fig:visualization_antmaze}
\end{figure*}

\section{Related Work}
\label{sec:rel}

Effectively learning policies with multiple hierarchies has been a long-standing problem in RL.
Goal-conditioned HRL~\cite{dayan_feudal_1992,schmidhuber_planning_1993,kulkarni_hierarchical_2016,vezhnevets_feudal_2017,nachum_data-efficient_2018,levy_learning_2019} aims to resolve this problem with a framework that separates high-level planning and low-level control using subgoals.
Recent advances in goal-conditioned HRL mainly focus on improving the learning efficiency of this framework. Nachum et al.~\cite{nachum_data-efficient_2018,nachum_near-optimal_2019} proposed an off-policy correction technique to stabilize training, and addressed the problem of goal space representation learning using a mutual-information-based objective. However, the subgoal generation process in their approaches is unconstrained and supervised only by the external rewards, and thus these methods may still suffer from training inefficiency. Levy et al.~\cite{levy_learning_2019} used hindsight techniques~\cite{andrychowicz_hindsight_2017} to train multi-level policies in parallel and also penalized the high-level for generating subgoals that the low-level failed to reach. However, their method has no theoretical guarantee, and they directly obtain the reachability measure from the environment, using the environmental information that is not available in many scenarios.
There are also prior works focusing on the unsupervised acquisition of subgoals based on potentially pivotal states~\cite{mcgovern_automatic_2001,simsek_identifying_2005,kulkarni_deep_2016,savinov_semi-parametric_2018,rafati_unsupervised_2019,huang_mapping_2019}. However, these subgoals are not guaranteed to be well-aligned with the downstream tasks and thus are often sub-optimal.

Several prior works have constructed an environmental graph for high-level planning used search nearby graph nodes as reachable subgoals for the low-level~\cite{savinov_semi-parametric_2018,eysenbach_search_2019,huang_mapping_2019,zhang_composable_2018}. However, these approaches hard-coded the high-level planning process based on domain-specific knowledge, e.g., treat the planning process as solving a shortest-path problem in the graph instead of a learning problem, and thus are limited in scalability. Nasiriany et al.~\cite{nasiriany_planning_2019} used goal-conditioned value functions to measure the feasibility of subgoals, but a pre-trained goal-conditioned policy is required. A more general topic of goal generation in RL has also been studied in the literature~\cite{florensa_automatic_2018,nair_visual_2018,ren_exploration_2019}. However, these methods only have a flat architecture and therefore often struggle on tasks that require complex high-level planning.

Our method relates to prior research that studied transition distance or reachability~\cite{pong_temporal_2018,savinov_semi-parametric_2018,savinov_episodic_2019,florensa_self-supervised_2019,hartikainen_dynamical_2020}. Most of these works learn the transition distance based on RL~\cite{pong_temporal_2018,florensa_self-supervised_2019,hartikainen_dynamical_2020} with a high cost. Savinov et al.~\cite{savinov_semi-parametric_2018,savinov_episodic_2019} proposed a supervised learning approach for reachability learning. However, the metric they learned depends on a certain policy used for interaction and thus could be sub-optimal compared to our method and we provide detailed comparison and discussion in Appendix~\ref{appsec:comp}. There are also other metrics that can reflect state similarities in MDPs, such as successor represention~\cite{dayan_improving_1993,kulkarni_deep_2016} that depends on both the environmental dynamics and a specific policy, and bisimulation metrics~\cite{ferns_metrics_2004,castro_scalable_2020} that depend on both the dynamics and the rewards. Compared to these metrics, the shortest transition distance depends only on the dynamics and therefore may be seamlessly applied to multi-task settings.

\section{Conclusion}
\label{sec:con}

We present a novel $k$-step adjacency constraint for goal-conditioned HRL framework to mitigate the issue of training inefficiency, with the theoretical guarantee of bounded suboptimality in both deterministic and stochastic MDPs. We show that the proposed adjacency constraint can be practically implemented with an adjacency network. Experiments on several testbeds with discrete and continuous state and action spaces demonstrate the effectiveness and robustness of our method.

As one of the most promising directions for scaling up RL, goal-conditioned HRL provides an appealing paradigm for handling large-scale problems. However, some key issues involving how to devise effective and interpretable hierarchies remain to be solved, such as how to empower the high-level policy to learn and explore in a more semantically meaningful action space~\cite{nachum_why_2019}, and how to enable the subgoals to be shared and reused in multi-task settings. Other future work includes extending our method to tasks with high-dimensional state spaces, e.g., by encompassing modern representation learning schemes~\cite{higgins_beta-vae:_2017,nachum_near-optimal_2019,srinivas_curl_2020}, and leveraging the adjacency network to improve the learning efficiency in more general scenarios.

{
  \bibliography{ref}
  \bibliographystyle{IEEEtranN}
}

%

\begin{IEEEbiography}[{\includegraphics[width=1in,height=1.25in,clip,keepaspectratio]{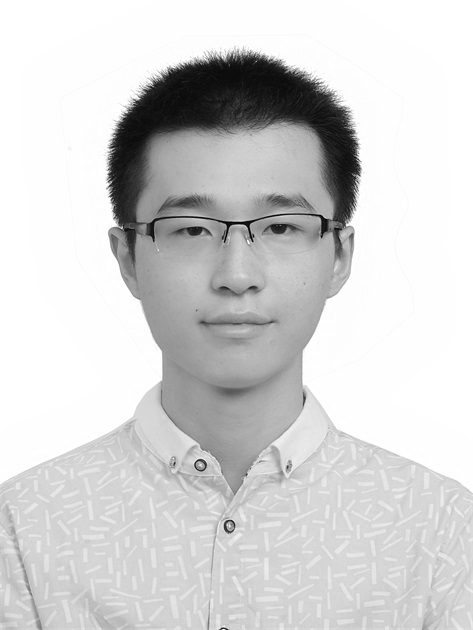}}]{Tianren Zhang}
received his B.S. degree from the department of automation, Tsinghua University, China, in 2019. He is currently pursuing his Ph.D. degree at the department of automation, Tsinghua University. His current research interests include generalization in machine learning, deep reinforcement learning, and learning theory.
\end{IEEEbiography}

\begin{IEEEbiography}[{\includegraphics[width=1in,height=1.25in,clip,keepaspectratio]{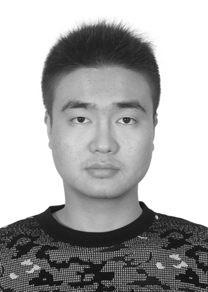}}]{Shangqi Guo} 
received a B.S. degree in mathematics and physics basic science from the University of Electronic Science and Technology of China, Chengdu, China, in 2015, and a Ph.D. degree at the Department of Automation, Tsinghua University, Beijing, China, in 2021. His current research interests include inference in artificial intelligence, brain-inspired computing, and reinforcement learning.
\end{IEEEbiography}

\begin{IEEEbiography}[{\includegraphics[width=1in,height=1.25in,clip,keepaspectratio]{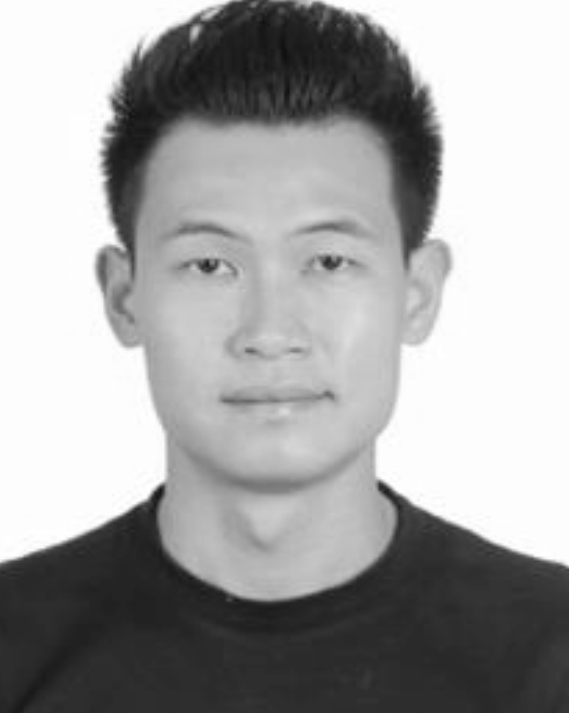}}]{Tian Tan}
Tian Tan received his B.S. degree in Telecommunications Engineering from Beijing University of Posts and Telecommunications, his M.S. degree and Ph.D. degree in Engineering and Ph.D. minor in Computer Science from Stanford University. His research interests broadly include topics in machine learning and algorithms, such as reinforcement learning, multi-task learning, gradient boosting decision trees, deep learning, and statistical learning theory.
\end{IEEEbiography}

\begin{IEEEbiography}[{\includegraphics[width=1in,height=1.25in]{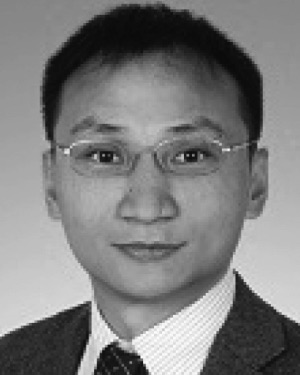}}]{Xiaolin Hu} (S'01-M'08-SM'13) received B.E. and M.E. degrees in automotive engineering from the Wuhan University of Technology, Wuhan, China, in 2001 and 2004, respectively, and a Ph.D. degree in automation and computer-aided engineering from the Chinese University of Hong Kong, Hong Kong, in 2007. He is currently an Associate Professor at the Department of Computer Science and Technology, Tsinghua University, Beijing, China. His current research interests include deep learning and computational neuroscience. At present, he is an Associate Editor of the IEEE Transactions on Image Processing. Previously he was an Associate Editor of the IEEE Transactions on Neural Networks and Learning Systems.
\end{IEEEbiography}

\begin{IEEEbiography}[{\includegraphics[width=1in,height=1.25in,clip,keepaspectratio]{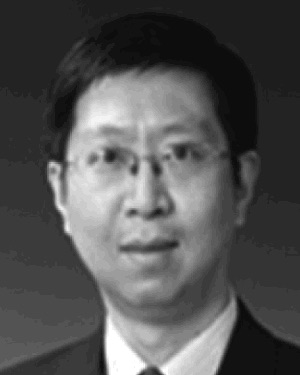}}]{Feng Chen} (M'06) received B.S. and M.S. degrees in automation from Saint-Petersburg Polytechnic University, Saint Petersburg, Russia, in 1994 and 1996, respectively, and a Ph.D. degree from the Automation Department, Tsinghua University, Beijing, China, in 2000. He is currently a Professor at Tsinghua University. His current research interests include computer vision, brain-inspired computing, and inference in graphical models.
\end{IEEEbiography}

\vfill





\vfill



%

\newpage
\setcounter{page}{1}
\appendices

\section{Proofs of Theorems}
\label{app:proof}

\subsection{Proof of Lemma~\ref{theo:low}}
\label{app:proof1}

\begin{proof}
Under the assumption that the MDP is deterministic and all states are strongly connected, there exists at least one shortest state trajectory from $s$ to $g$. Without loss of generality, we consider one shortest state trajectory $\tau^* = (s_0,s_1,s_2,\cdots,s_{n-1},s_n)$, where $s_0 = s,\,s_n = \varphi^{-1}(g)$ and $d_{\mathrm{st}}\left(s,\varphi^{-1}(g)\right) = n$. For all $k\in\mathbb{N}_+$ and $k\le d_{\mathrm{st}}\left(s,\varphi^{-1}(g)\right) = n$, let $\tilde{g} = \varphi(s_k)$, and let $\tau = (s_0,s_1,s_2,\cdots,s_k)$ be the $k$-step sub-trajectory of $\tau^*$ from $s_0$ to $s_k$. Since $s_0$ and $s_k$ is connected by $\tau$ in $k$ steps, we have that $d_{\mathrm{st}}\left(s_0,\varphi^{-1}(\tilde{g})\right) = d_{\mathrm{st}}\left(s_0,s_k \right) \le k$, i.e., $\tilde{g}\in\mathcal{G}_A(s,k)$. In the following, we will prove that $\pi^*(s_i,\tilde{g}) = \pi^*(s_i,g),\ \forall\, s_i\in \tau\,(i\ne k)$.

We first prove that the shortest transition distance $d_{\mathrm{st}}$ satisfies the triangle inequality, i.e., consider three arbitrary states $s_1,s_2,s_3\in\mathcal{S}$, then $d_{\mathrm{st}}(s_1,s_3) \le d_{\mathrm{st}}(s_1,s_2) + d_{\mathrm{st}}(s_2,s_3)$: let $\tau^*_{12}$ be one shortest state trajectory between $s_1$ and $s_2$ and let $\tau^*_{23}$ be one shortest state trajectory between $s_2$ and $s_3$. We can concatenate $\tau^*_{12}$ and $\tau^*_{23}$ to form a trajectory $\tau_{13} = (\tau^*_{12},\tau^*_{23})$ that connects $s_1$ and $s_3$. Then, by Definition~\ref{def:distance} we have $d_{\mathrm{st}}(s_1,s_3)\le d_{\mathrm{st}}(s_1,s_2) + d_{\mathrm{st}}(s_2,s_3)$.

Using the triangle inequality, we can prove that the sub-trajectory $\tau = (s_0,s_1,s_2,\cdots,s_k)$ is also a shortest trajectory from $s_0 = s$ to $s_k$: assume that this is not true and there exists a shorter trajectory from $s_0$ to $s_k$. Then, by Definition~\ref{def:distance} we have $d_{\mathrm{st}}(s_0,s_k) < k$. Since $(s_k,s_{k+1},\cdots,s_n)$ is a valid trajectory from $s_k$ to $s_n$, we have $d_{\mathrm{st}}(s_k,s_n) \le n-k$. Applying the triangle inequality, we have $d_{\mathrm{st}}(s_0,s_n) \le d_{\mathrm{st}}(s_0,s_k) + d_{\mathrm{st}}(s_k,s_n) < k + n - k = n$, which is in contradiction with $d_{\mathrm{st}}\left(s,\varphi^{-1}(g)\right) = d_{\mathrm{st}}(s_0,s_n) = n$. Thus, our original assumption must be false, and the trajectory $\tau = (s_0,s_1,s_2,\cdots,s_k)$ is a shortest trajectory from $s_0$ to $s_k$.

Finally, let $\alpha:\mathcal{S}\times\mathcal{S}\rightarrow\mathcal{A}$ be an inverse dynamics model, i.e., given state $s_t$ and the next state $s_{t+1}$, $\alpha(s_t,s_{t+1})$ outputs the action $a_t$ that is performed at $s_t$ to reach $s_{t+1}$. Then, employing Equation~\eqref{equ:goal_conditioned}, for $i=0,1,\cdots,k-1$ we have $\pi^*(s_i,g) = \alpha(s_i,s_{i+1})$ given that $\tau^*$ is a shortest trajectory from $s_0$ to $\varphi^{-1}(g)$, and $\pi^*(s_i,\tilde{g}) = \alpha(s_i,s_{i+1})$ given that $\tau$ is a shortest trajectory from $s_0$ to $\varphi^{-1}(\tilde{g})$. This indicates that $\pi^*(s_i,\tilde{g}) = \pi^*(s_i,g),\ \forall\, s_i\in \tau\,(i\ne k)$.
\end{proof}

\subsection{Proof of Theorem~\ref{theo:high}}
\label{app:proof2}

\begin{proof}
For each subgoal $g_{kt},\,t=0,1,\cdots,T-1$, if $k > d_\mathrm{st}(s_{kt},\varphi^{-1}(g_{kt}))$, then we have that $g_{kt}\in\mathcal{G}_A(s_{kt},k)$ and thus directly set $\tilde{g}_{kt} = g_{kt}$; otherwise, by using Lemma~\ref{theo:low}, we have that for each subgoal $g_{kt}$, there exists a subgoal $\tilde{g}_{kt}\in\mathcal{G}_A(s_{kt},k)$ that can induce the same low-level $k$-step action sequence as $g_{kt}$. This indicates that the agent's trajectory and the high-level reward $r_{kt}^h$ defined by Equation~\eqref{equ:r_high} remain the same for all $t$ when replacing $g_{kt}$ with $\tilde{g}_{kt}$. Then, using the high-level Bellman optimality equation for the optimal state-action value function
\begin{equation}
\begin{aligned}
Q^*(s_{kt},g_{kt}) &= r_{kt}^\mathrm{hi} + \gamma\max_{g\in\mathcal{G}} Q^*(s_{k(t+1)},g)\\
&= r_{kt}^\mathrm{hi} + \gamma Q^*(s_{k(t+1)},g_{k(t+1)}),\\
&\hspace{10em} t = 0,\,1\cdots,\,T-1
\end{aligned}
\end{equation}
and $Q^*(s_{kT},g) = 0,\ \forall\, g\in\mathcal{G}$ (since $s_{kT}$ is the final state in the trajectory $\tau^*$), we have $Q^*(s_{kt},\tilde{g}_{kt}) = Q^*(s_{kt}, g_{kt}),\,t=0,1,\cdots,T-1$.
\end{proof}

\subsection{Proof of Theorem~\ref{theo:subopt}}
\label{app:proof3}

\begin{proof}
For ease of exposition, we first introduce some shorthand: given a state $s\in\mathcal{S}$, let $P (s,g)\in [0,1]^{|\mathcal{S}|}$ with $\forall s\in\mathcal{S}, g\in\mathcal{G}, \lVert P(s,g)\rVert_1 = 1$ denote the probablistic distribution of the $k$-step successor state starting from $s$ under the subgoal $g$; let $\widetilde{r}(s)\in[0,kR_\mathrm{max}]^{|\mathcal{S}|}$ denote the $k$-step high-level reward vector for every successor state. We then rewrite the state value function of $\pi_\mathrm{hi}^*$ and $\pi_\mathrm{hi}^\mathrm{adj}$ respectively:
\begin{equation}
\begin{aligned}
V^{\pi_\mathrm{hi}^*}(s) &= \sum_{g\in\mathcal{G}} \pi_\mathrm{hi}^*(g\,|\, s) \sum_{s'\in\mathcal{S}} P^k(s'\,|\, s,g) \Big(R_\mathrm{hi}(s,s') \,+ \\
&\qquad \gamma V^{\pi_\mathrm{hi}^*}(s')\Big) \\
&= \left\langle \sum_{g\in\mathcal{G}} \pi_\mathrm{hi}^*(g\,|\, s)P(s,g), R(s) + \gamma V^{\pi_\mathrm{hi}^*} \right\rangle;
\end{aligned}
\end{equation}
\begin{equation}
\begin{aligned}
V^{\pi_\mathrm{hi}^\mathrm{adj}}(s) &= \sum_{g\in\mathcal{G}_\mathrm{AM}(s,k)}\sum_{g'\in\mathcal{G}}P(g\,|\, s,g')\pi_\mathrm{hi}^*(g'\,|\, s) \\
&\qquad\sum_{s'\in\mathcal{S}} P^k(s'\,|\, s,g) \left(R_\mathrm{hi}(s,s') + \gamma V^{\pi_\mathrm{hi}^\mathrm{adj}}(s')\right) \\
&= \sum_{g\in\mathcal{G}}\pi_\mathrm{hi}^*(g\,|\, s)\sum_{g'\in\mathcal{G}_\mathrm{AM}(s,k)} P(g'\,|\, s,g) \\
&\qquad\sum_{s'\in\mathcal{S}} P^k(s'\,|\, s,g') \left(R_\mathrm{hi}(s,s') + \gamma V^{\pi_\mathrm{hi}^\mathrm{adj}}(s')\right) \\
&= \Bigg\langle \sum_{g\in\mathcal{G}}\pi_\mathrm{hi}^*(g\,|\, s)\sum_{g'\in\mathcal{G}_\mathrm{AM}(s,k)} P(g'\,|\, s,g) P(s,g'), \\
&\qquad R(s) + \gamma V^{\pi_\mathrm{hi}^\mathrm{adj}} \Bigg\rangle \\
&= \Bigg\langle \sum_{g\in\mathcal{G}} \pi_\mathrm{hi}^*(g\,|\, s)P_\mathrm{AM}(s,g), R(s) + \gamma V^{\pi_\mathrm{hi}^\mathrm{adj}} \Bigg\rangle,
\end{aligned}
\end{equation}
where $\langle \cdot\,,\cdot \rangle$ denotes the inner product of vectors, and
\begin{equation}
P_\mathrm{AM}(s,g) \vcentcolon= \sum_{g'\in\mathcal{G}_\mathrm{AM}(s,k)} P(g'\,|\, s,g) P(s,g').
\end{equation}
The transition mismatch rate $\mu_k$~\eqref{equ:distribution_mismatch} can be rewritten as
\begin{equation}
\mu_k = \max_{s,g} \left\lVert P(s,g) - P_\mathrm{AM}(s,g) \right\rVert_\infty,
\end{equation}
which is presented as the infinite norm of the difference of the probabilistic vectors. Then, for every $s\in\mathcal{S}$, we have
\begin{equation}
\begin{aligned}
&\left| V^{\pi_\mathrm{hi}^*}(s) - V^{\pi_\mathrm{hi}^\mathrm{adj}}(s) \right| \\
&\quad=\Bigg| \left\langle \sum_{g\in\mathcal{G}} \pi_\mathrm{hi}^*(g\,|\, s)P(s,g), R(s) + \gamma V^{\pi_\mathrm{hi}^*} \right\rangle - \\
&\qquad\quad\left\langle \sum_{g\in\mathcal{G}}\pi_\mathrm{hi}^*(g\,|\, s)P_\mathrm{AM}(s,g), R(s) + \gamma V^{\pi_\mathrm{hi}^\mathrm{adj}} \right\rangle \Bigg| \\
&\quad\le\Bigg| \left\langle \sum_{g\in\mathcal{G}} \pi_\mathrm{hi}^*(g\,|\, s)P(s,g), R(s)\right\rangle - \\
&\qquad\quad\left\langle \sum_{g\in\mathcal{G}} \pi_\mathrm{hi}^*(g\,|\, s)P_\mathrm{AM}(s,g), R(s)\right\rangle \Bigg| \\
&\qquad+ \gamma\,\Bigg| \left\langle \sum_{g\in\mathcal{G}} \pi_\mathrm{hi}^*(g\,|\, s)P(s,g), V^{\pi_\mathrm{hi}^*} \right\rangle - \\
&\qquad\quad\left\langle \sum_{g\in\mathcal{G}}\pi_\mathrm{hi}^*(g\,|\, s)P_\mathrm{AM}(s,g), V^{\pi_\mathrm{hi}^\mathrm{adj}} \right\rangle \Bigg|.
\end{aligned}
\label{eq:decomposition}
\end{equation}
In the sequel we will bound both terms in the right hand side of Inequation~\eqref{eq:decomposition} respectively.
For the first term in the right hand side, we have

\begin{equation}
\begin{aligned}
&\Bigg| \left\langle \sum_{g\in\mathcal{G}} \pi_\mathrm{hi}^*(g\,|\, s)P(s,g), R(s)\right\rangle -\\
&\qquad\left\langle \sum_{g\in\mathcal{G}} \pi_\mathrm{hi}^*(g\,|\, s)P_\mathrm{AM}(s,g), R(s)\right\rangle \Bigg| \\
&\quad= \Bigg| \left\langle \sum_{g\in\mathcal{G}} \pi_\mathrm{hi}^*(g\,|\, s)P(s,g) - \sum_{g\in\mathcal{G}} \pi_\mathrm{hi}^*(g\,|\, s)P_\mathrm{AM}(s,g), R(s)\right\rangle \Bigg| \\
&\quad= \Bigg| \Bigg\langle \sum_{g\in\mathcal{G}} \pi_\mathrm{hi}^*(g\,|\, s)P(s,g) - \sum_{g\in\mathcal{G}} \pi_\mathrm{hi}^*(g\,|\, s)P_\mathrm{AM}(s,g), \\
&\qquad\qquad R(s) - \frac{kR_\mathrm{max}}{2}\cdot\bm{1}\Bigg\rangle \Bigg|\\
&\quad\le\left\lVert \sum_{g\in\mathcal{G}} \pi_\mathrm{hi}^*(g\,|\, s)P(s,g) - \sum_{g\in\mathcal{G}} \pi_\mathrm{hi}^*(g\,|\, s)P_\mathrm{AM}(s,g) \right\rVert_1\cdot \\
&\qquad\qquad\left\lVert R(s) - \frac{kR_\mathrm{max}}{2}\cdot\bm{1} \right\rVert_\infty \\
&\quad\le \max_{s,g}\left\lVert P(s,g) - P_\mathrm{AM}(s,g) \right\rVert_\infty \left\lVert R(s) - \frac{kR_\mathrm{max}}{2}\cdot\bm{1} \right\rVert_\infty \\
&\quad= \frac{\mu_k k R_\mathrm{max}}{2},
\end{aligned}
\label{eq:bound_term_1}
\end{equation}
where we leverage the fact that $\pi_\mathrm{hi}^*(g\,|\, s)$, $\sum_{g\in\mathcal{G}} \pi_\mathrm{hi}^*(g\,|\, s)P(s,g)$ and $\sum_{g\in\mathcal{G}} \pi_\mathrm{hi}^*(g\,|\, s)P_\mathrm{AM}(s,g)$ are all probability distributions and thus sum up to 1; $\bm{1}\in \mathbb{R}^{|\mathcal{S}|}$ is the all-one vector used to center the range of $R(s)$ around the origin.

The second term in the right hand side can be similarly bounded by
\begin{equation}
\begin{aligned}
&\gamma\,\Bigg| \left\langle \sum_{g\in\mathcal{G}} \pi_\mathrm{hi}^*(g\,|\, s)P(s,g), V^{\pi_\mathrm{hi}^*} \right\rangle -\\
&\qquad\left\langle \sum_{g\in\mathcal{G}}\pi_\mathrm{hi}^*(g\,|\, s)P_\mathrm{AM}(s,g), V^{\pi_\mathrm{hi}^\mathrm{adj}} \right\rangle \Bigg| \\
&\quad \le \gamma\, \Bigg| \left\langle \sum_{g\in\mathcal{G}} \pi_\mathrm{hi}^*(g\,|\, s)P(s,g), V^{\pi_\mathrm{hi}^*} \right\rangle - \\
&\qquad\qquad\quad\left\langle \sum_{g\in\mathcal{G}}\pi_\mathrm{hi}^*(g\,|\, s)P(s,g), V^{\pi_\mathrm{hi}^\mathrm{adj}} \right\rangle \Bigg| \\
&\qquad + \gamma\, \Bigg| \left\langle \sum_{g\in\mathcal{G}} \pi_\mathrm{hi}^*(g\,|\, s)P(s,g), V^{\pi_\mathrm{hi}^\mathrm{adj}} \right\rangle - \\
&\qquad\qquad\quad\left\langle \sum_{g\in\mathcal{G}}\pi_\mathrm{hi}^*(g\,|\, s)P_\mathrm{AM}(s,g), V^{\pi_\mathrm{hi}^\mathrm{adj}} \right\rangle \Bigg|,
\end{aligned}
\label{eq:bound_term_2}
\end{equation}
where both terms of the right hand side of Inequality~\eqref{eq:bound_term_2} can be respectively bounded by
\begin{equation}
\begin{aligned}
&\gamma\, \Bigg| \left\langle \sum_{g\in\mathcal{G}} \pi_\mathrm{hi}^*(g\,|\, s)P(s,g), V^{\pi_\mathrm{hi}^*} \right\rangle - \\
&\qquad\left\langle \sum_{g\in\mathcal{G}}\pi_\mathrm{hi}^*(g\,|\, s)P(s,g), V^{\pi_\mathrm{hi}^\mathrm{adj}} \right\rangle \Bigg| \\
&\quad\le \gamma \left\lVert V^{\pi^*_\mathrm{hi}} - V^{\pi^\mathrm{adj}_\mathrm{hi}} \right\rVert_\infty 
\end{aligned}
\end{equation}
\begin{equation}
\begin{aligned}
& \gamma\, \Bigg| \left\langle \sum_{g\in\mathcal{G}} \pi_\mathrm{hi}^*(g\,|\, s)P(s,g), V^{\pi_\mathrm{hi}^\mathrm{adj}} \right\rangle - \\
&\qquad\left\langle \sum_{g\in\mathcal{G}}\pi_\mathrm{hi}^*(g\,|\, s)P_\mathrm{AM}(s,g), V^{\pi_\mathrm{hi}^\mathrm{adj}} \right\rangle \Bigg|\\
&= \gamma\, \Bigg| \Bigg\langle \sum_{g\in\mathcal{G}} \pi_\mathrm{hi}^*(g\,|\, s)P(s,g) - \sum_{g\in\mathcal{G}} \pi_\mathrm{hi}^*(g\,|\, s)P_\mathrm{AM}(s,g), \\
&\qquad V^{\pi_\mathrm{hi}^\mathrm{adj}} - \frac{kR_\mathrm{max}}{2(1-\gamma)}\cdot\bm{1} \Bigg\rangle \Bigg| \\
&\le \gamma \left\lVert V^{\pi^*_\mathrm{hi}} - V^{\pi^\mathrm{adj}_\mathrm{hi}} \right\rVert_\infty +\\
&\qquad\gamma \left\lVert \sum_{g\in\mathcal{G}} \pi_\mathrm{hi}^*(g\,|\, s)P(s,g) - \sum_{g\in\mathcal{G}} \pi_\mathrm{hi}^*(g\,|\, s)P_\mathrm{AM}(s,g)\right\rVert_1\cdot\\
&\qquad\ \left\lVert V^{\pi_\mathrm{hi}^\mathrm{adj}} - \frac{kR_\mathrm{max}}{2(1-\gamma)}\cdot\bm{1} \right\rVert_\infty \\
&\le \gamma \left\lVert V^{\pi^*_\mathrm{hi}} - V^{\pi^\mathrm{adj}_\mathrm{hi}} \right\rVert_\infty +\\
&\qquad\gamma \max_{s,g}\left\lVert P(s,g) - P_\mathrm{AM}(s,g) \right\rVert_\infty \left\lVert V^{\pi_\mathrm{hi}^\mathrm{adj}} - \frac{kR_\mathrm{max}}{2(1-\gamma)}\cdot\bm{1} \right\rVert_\infty \\
&\le \gamma \left\lVert V^{\pi^*_\mathrm{hi}} - V^{\pi^\mathrm{adj}_\mathrm{hi}} \right\rVert_\infty + \frac{\gamma\mu_k k R_\mathrm{max}}{2(1-\gamma)}.
\end{aligned}
\end{equation}
Combining the bound in~\eqref{eq:bound_term_1} and~\eqref{eq:bound_term_2} yields
\begin{equation}
\begin{aligned}
&\left| V^{\pi_\mathrm{hi}^*}(s) - V^{\pi_\mathrm{hi}^\mathrm{adj}}(s) \right|\\
&\quad\le \frac{\mu_k k R_\mathrm{max}}{2} + \frac{\gamma\mu_k k R_\mathrm{max}}{2(1-\gamma)} + \gamma \left\lVert V^{\pi^*_\mathrm{hi}} - V^{\pi^\mathrm{adj}_\mathrm{hi}} \right\rVert_\infty.
\label{eq:bound}
\end{aligned}
\end{equation}
Since the bound~\eqref{eq:bound} holds for all $s\in\mathcal{S}$, we can also take the infinite norm on the left hand side, which gives
\begin{equation}
\begin{aligned}
&\left\lVert V^{\pi^*_\mathrm{hi}} - V^{\pi^\mathrm{adj}_\mathrm{hi}} \right\rVert_\infty\\
&\quad\le \frac{\mu_k k R_\mathrm{max}}{2} + \frac{\gamma\mu_k k R_\mathrm{max}}{2(1-\gamma)} + \gamma \left\lVert V^{\pi^*_\mathrm{hi}} - V^{\pi^\mathrm{adj}_\mathrm{hi}} \right\rVert_\infty.
\end{aligned}
\end{equation}
A simple transformation of the above inequality completes the proof.
\end{proof}

\section{Comparison with the Work of Savinov et al.}
\label{appsec:comp}
Savinov et al.~\cite{savinov_semi-parametric_2018,savinov_episodic_2019} also propose a supervised learning approach for learning the adjacency between states. The main differences between our method and theirs are 1) We use trajectories sampled by multiple policies to construct training samples, while they only use trajectories sampled by one specific policy; 2) We use an adjacency matrix to explicitly aggregate the adjacency information and sample training pairs based on the adjacency matrix, while they directly sample training pairs from trajectories. These differences lead to two advantages of our method: 1) By using multiple policies, we achieve a more accurate adjacency approximation, as shown by Equation~\eqref{equ:approximation}; 2) By maintaining an adjacency matrix, we can uniformly sample from the set of all explored states and realize a nearly unbiased estimation of the expectation in Equation~\eqref{equ:contrastive}, while the estimation by sampling state-pairs from trajectories is biased. As an example, consider a simple grid world in Fig.~\ref{subfig:gridworld}, where states are represented by their $(x,y)$ positions. In this environment, states $s_1$ and $s_2$ are non-adjacent since they are separated by a wall. However, it is hard for the method by Savinov et al. to handle this situation as these two states rarely emerge in the same trajectory due to the large distance, and thus the loss induced by this state-pair is very likely to be dominated by the loss of other nearer state-pairs. Meanwhile, our method treats the loss of all state pairs equally, and can therefore alleviate this phenomenon.
Empirically, we employed a random agent (since the random policy is stochastic, it can be viewed as multiple deterministic policies, and is enough for adjacency learning in this simple environment) to interact with the environment for $20,000$ steps, and trained the adjacency network with collected samples using both methods.
We visualize the LLE of state embeddings and two adjacency distance heatmaps by both methods respectively in Fig.~\ref{subfig:lle_ours} and~\ref{subfig:lle_savinov}. Visualizations validate our analysis, showing that our method does learn a better adjacency measure in this scenario.

\begin{figure}
\centering
\subfigure[]{
  \includegraphics[width=0.23\linewidth]{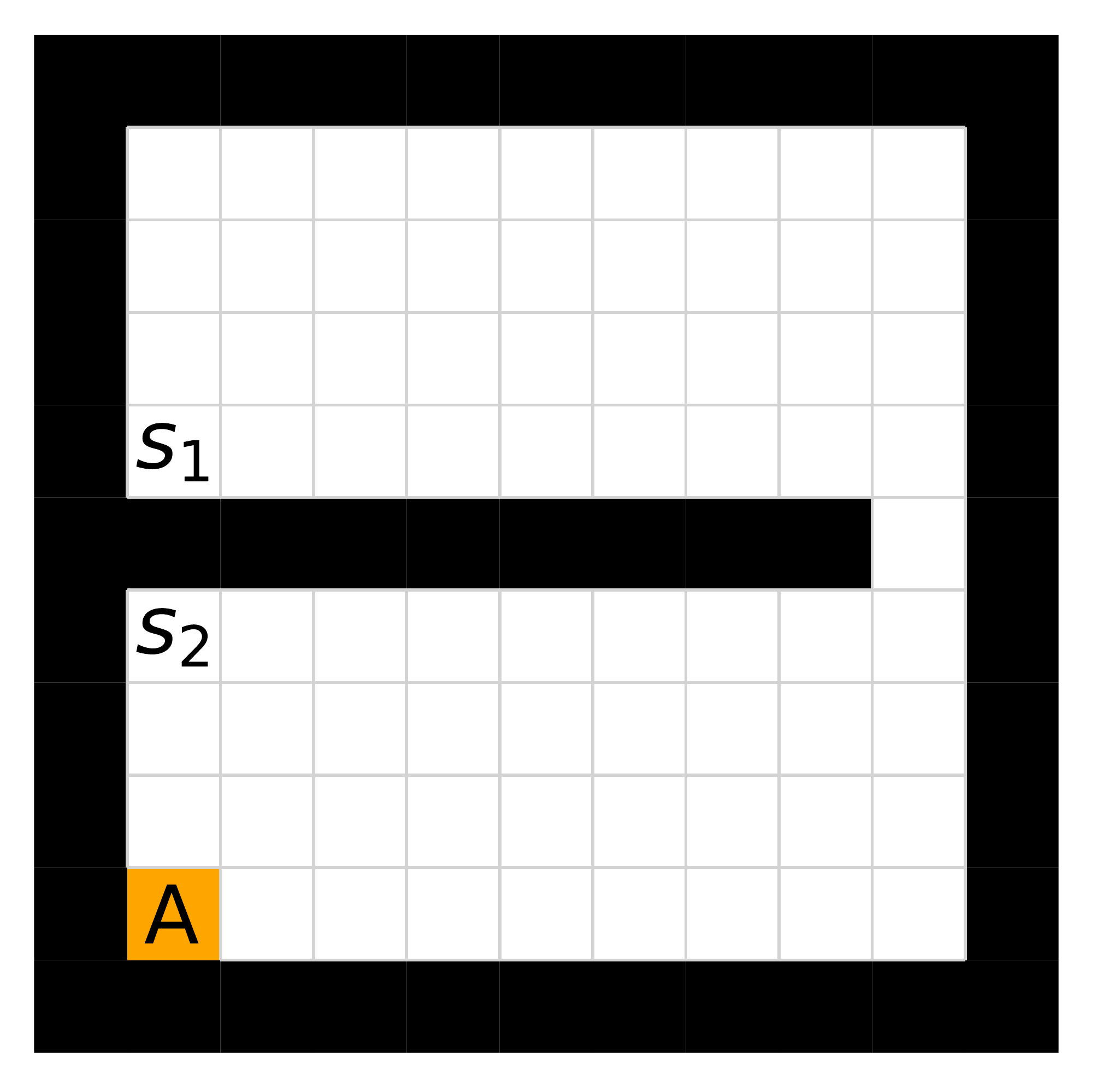}
  \label{subfig:gridworld}
}
\subfigure[]{
  \includegraphics[width=0.21\linewidth]{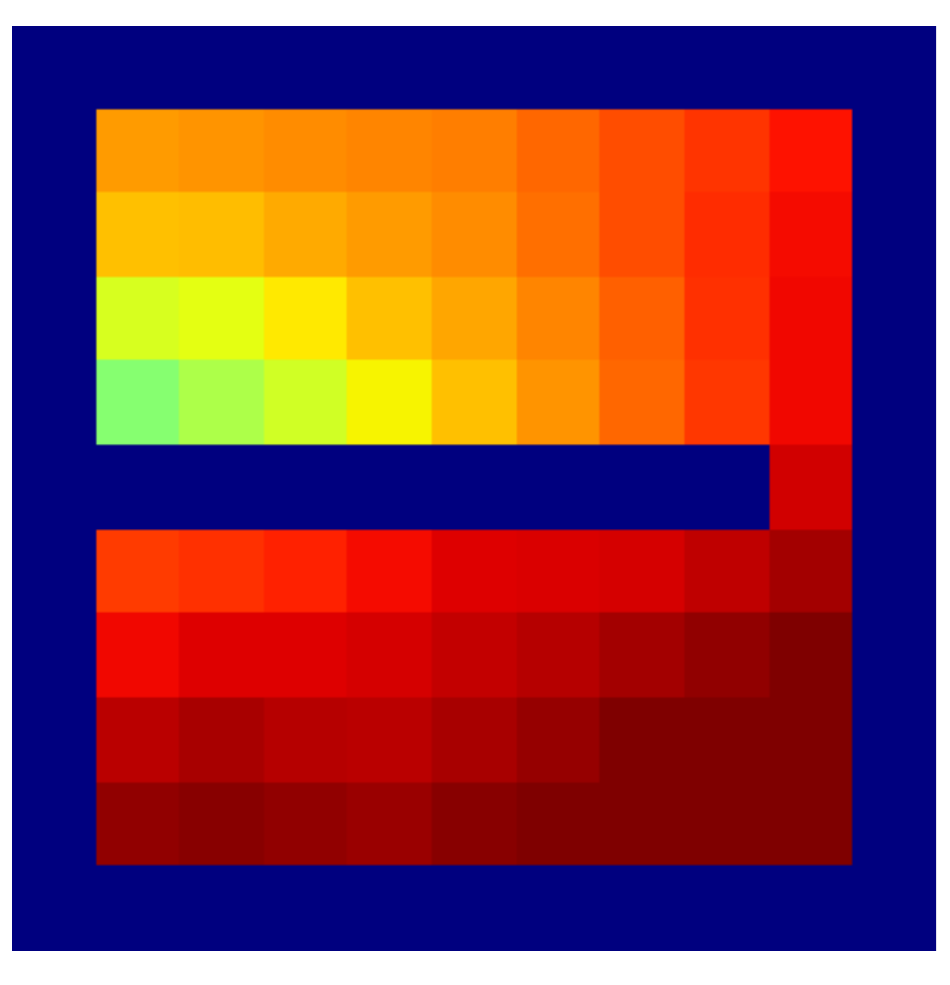}
  \hspace{0.5em}
  \includegraphics[width=0.21\linewidth]{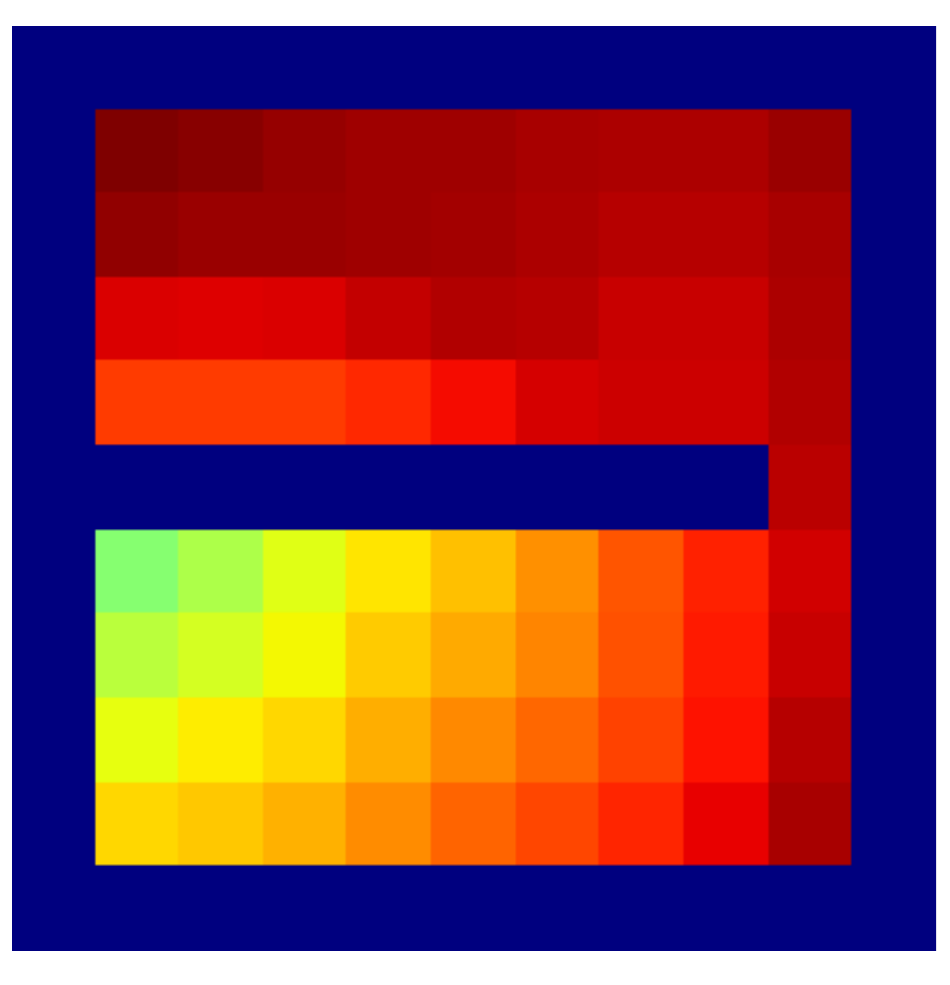}
  \hspace{0.5em}
  \includegraphics[width=0.21\linewidth]{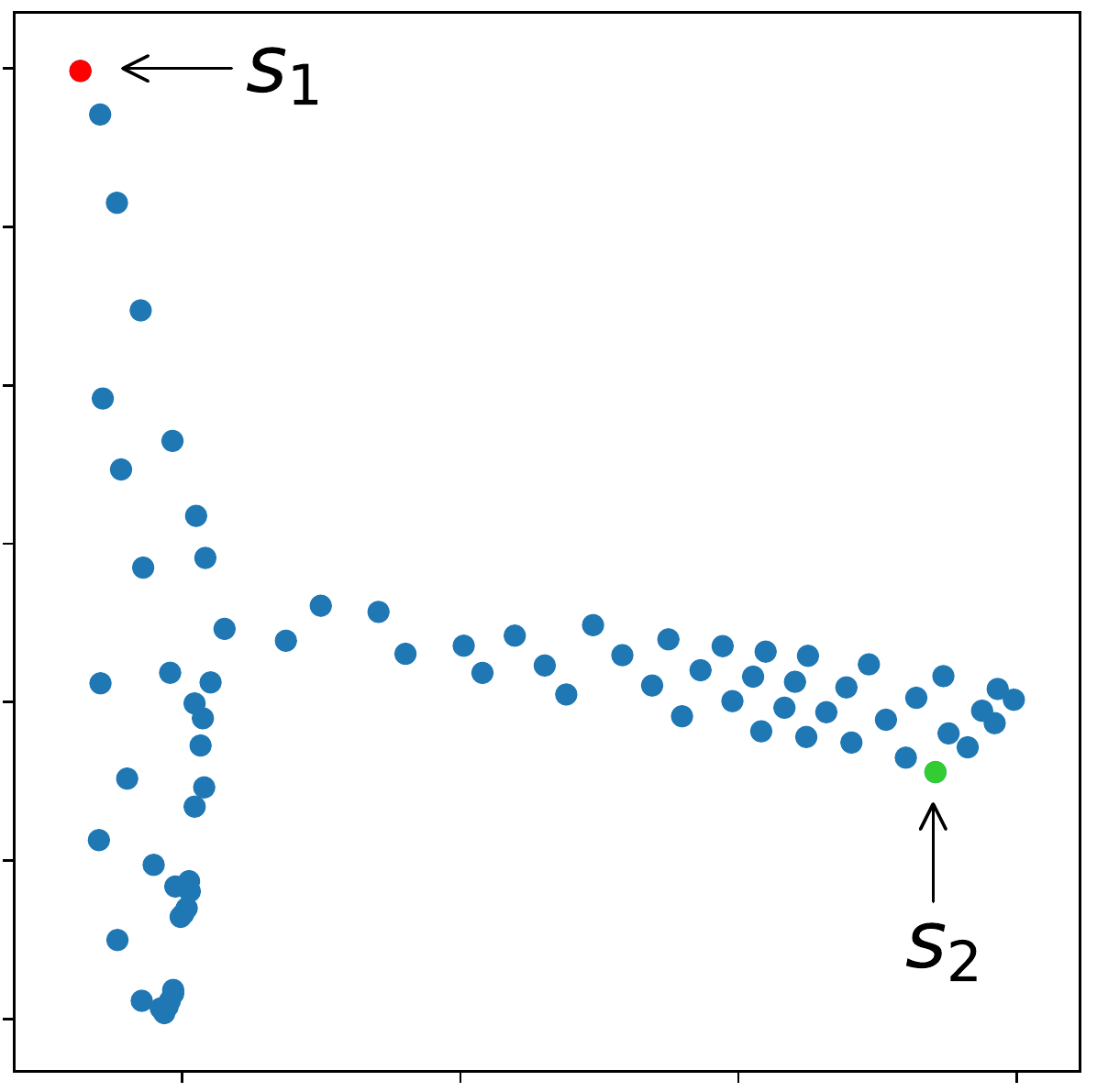}
  \label{subfig:lle_ours}
  }\\
\subfigure[]{
  \hspace{0.1em}
  \includegraphics[width=0.21\linewidth]{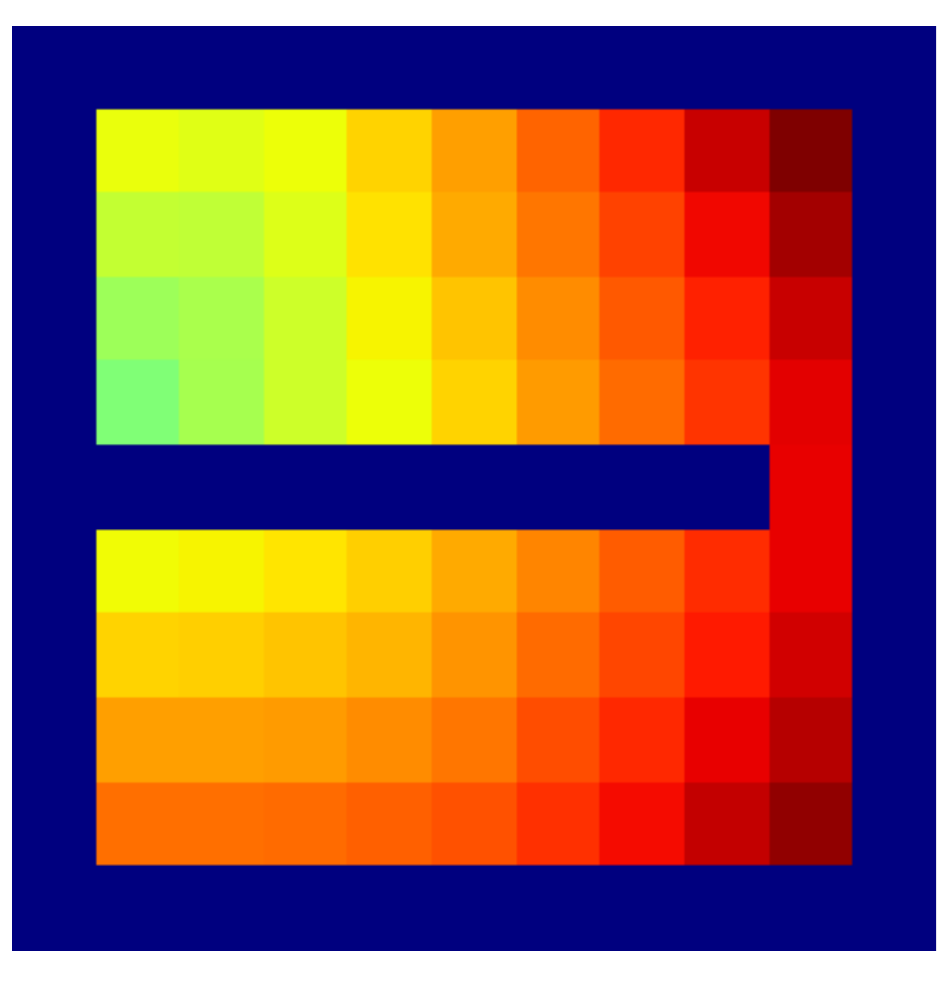}
  \hspace{1em}
  \includegraphics[width=0.21\linewidth]{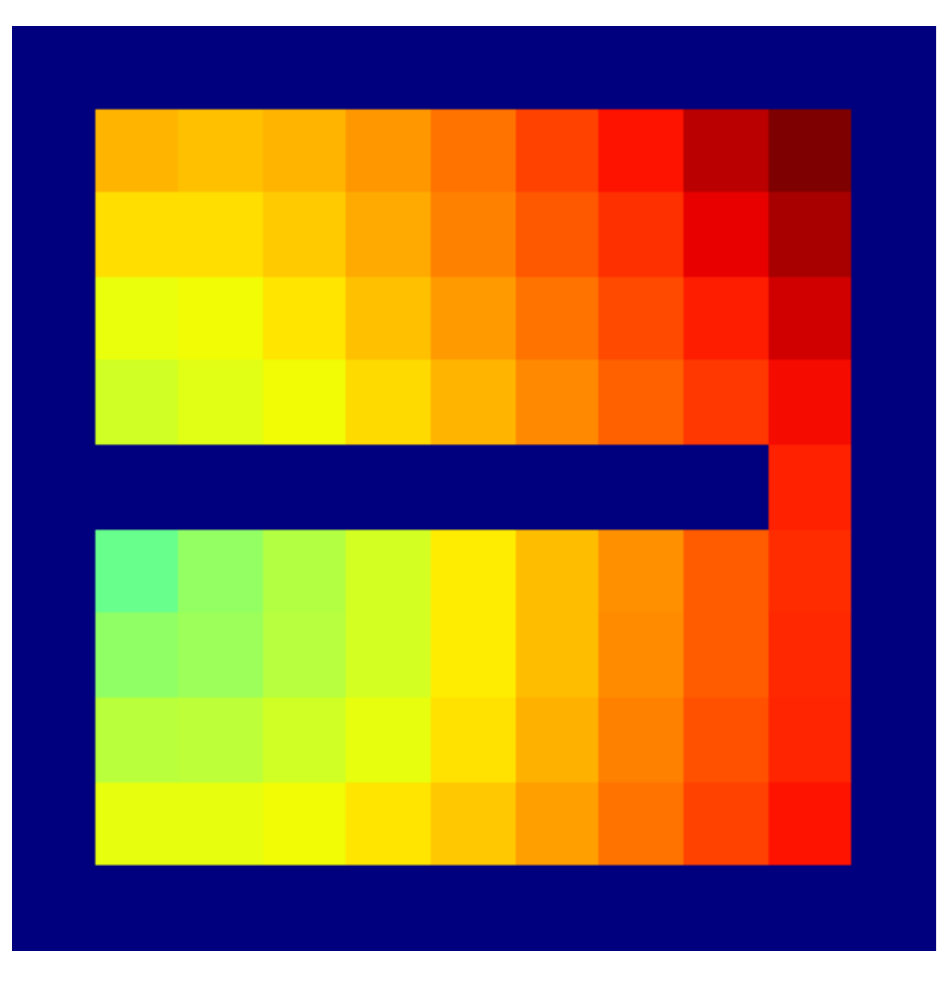}
  \hspace{1em}
  \includegraphics[width=0.21\linewidth]{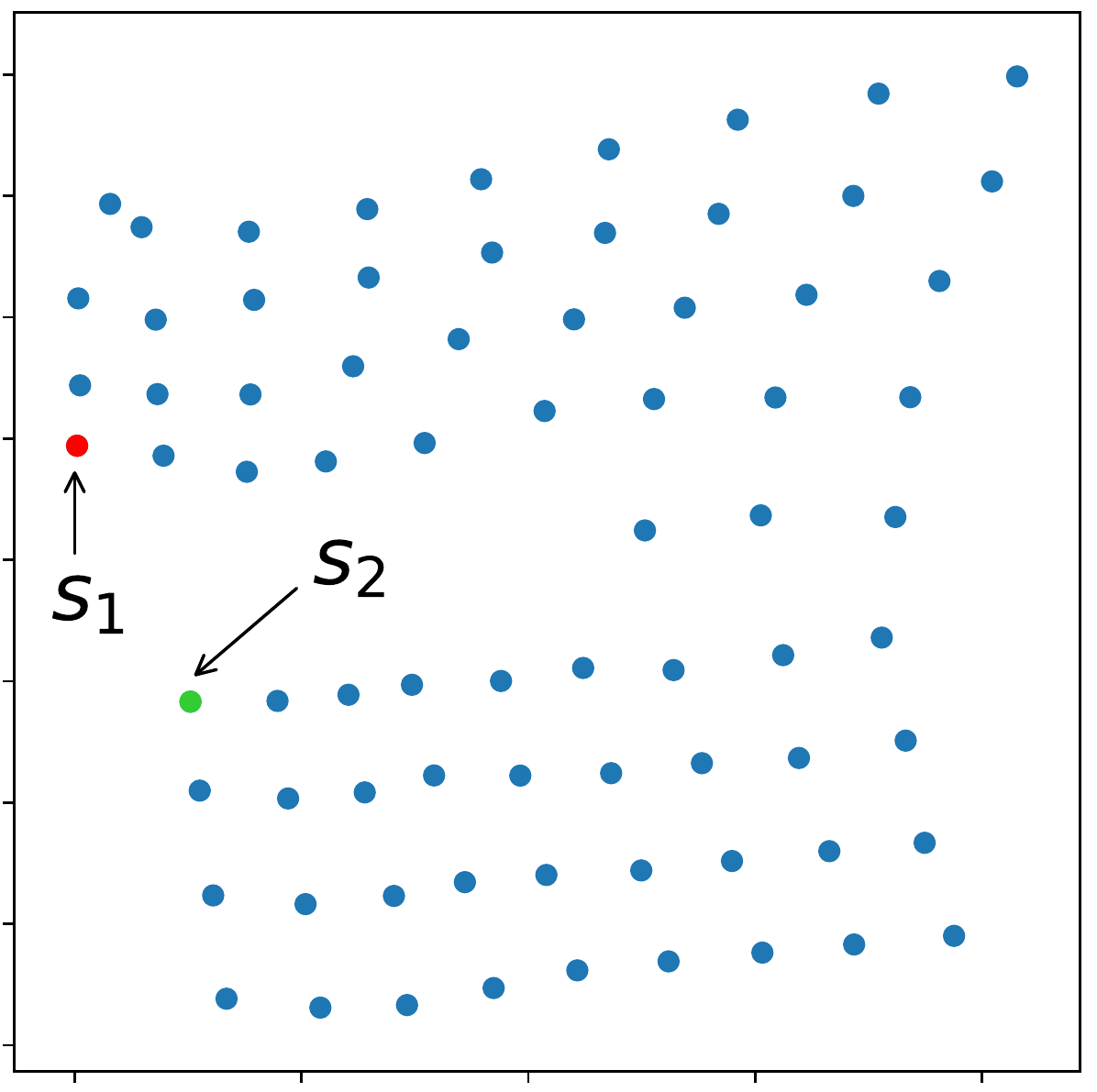}
  \label{subfig:lle_savinov}
  }
\caption{Qualitative comparison of different adjacency learning methods. (a) Environment layout. The agent starts from the grid A. (b) Results of our method, including the adjacency heatmaps from states $s_1$, $s_2$, and the LLE visualization of state embeddings. (c) Results of the method proposed by Savinov et al., including the adjacency heatmaps from states $s_1$, $s_2$, and the LLE visualization of state embeddings.}
\label{fig:r_compare}
\end{figure}

\section{Implementation Details}
\label{appsec:imp}

\subsection{HRAC and Baseline Details}

We use PyTorch~\cite{paszke_pytorch_2019} to implement our method HRAC and all the baselines.

\textit{HRAC:} for discrete control tasks, we adopt a binary intrinsic reward setting: we set the intrinsic reward to 1 when $|s_x - g_x| \le 0.5$ and $|s_y - g_y| \le 0.5$, where $(s_x,\,s_y)$ is the position of the agent and $(g_x,\,g_y)$ is the position of the desired subgoal. For continuous control tasks, we adopt a dense intrinsic reward setting based on the negative Euclidean distances $-\lVert s - g\rVert_2$ between states and subgoals.

\textit{HIRO:} following Nachum et al.~\cite{nachum_data-efficient_2018}, we restrict the output of high-level to $(\pm10,\,\pm10)$, representing the desired shift of the agent's $(x,y)$ position. By limiting the range of directional subgoals generated by the high-level, HIRO can roughly control the Euclidean distance between the absolute subgoal and the current state in the raw goal space rather than the learned adjacency space. 

\textit{HRL-HER:} as HER cannot be applied to the on-policy training scheme in a straightforward manner, in discrete control tasks where the low-level policy is trained using A2C, we modify its implementation so that it can be incorporated into the on-policy setting. For this on-policy variant, during the training phase, we maintain an additional episodic state memory. This memory stores states that the agent has visited from the beginning of each episode. When the high-level generates a new subgoal, the agent randomly samples a subgoal mapped from a stored state with a fixed probability of 0.2 to substitute the generated subgoal for the low-level to reach. This implementation resembles the ``episode'' strategy introduced in the original HER. We still use the original HER in continuous control tasks.

\textit{NoAdj:} we follow the training pipeline proposed by Savinov et al.~\cite{savinov_semi-parametric_2018,savinov_episodic_2019}, where no adjacency matrix is maintained. Training pairs are constructed by randomly sampling state-pairs $(s_i,\,s_j)$ from the stored trajectories. The samples with $|i-j|\le k$ are labeled as positive with $l=1$, and the samples with $|i-j| \ge M k$ are negative ones with $l=0$. The hyperparameter $M$ is used to create a gap between the two types of samples. In practice, we use $M = 4$.

\textit{NegReward:} in this variant, every time the high-level generates a subgoal, we use the adjacency network to judge whether it is $k$-step adjacent. If the subgoal is non-adjacent, the high-level will be penalized with a negative reward $-1$.

\subsection{Network Architecture}
For the hierarchical policy network, we employ the same architecture as HIRO~\cite{nachum_data-efficient_2018} in continuous control tasks, where both the high-level and the low-level use TD3~\cite{fujimoto_addressing_2018} algorithm for training. In discrete control tasks, we use two networks consisting of 3 fully-connected layers with ReLU nonlinearities as the low-level actor and critic networks of A2C (our preliminary results show that the performances using on-policy and off-policy methods for the low-level training are similar in the discrete control tasks we consider), and use the same high-level TD3 network architecture as in the continuous control task. The size of the hidden layers of both the low-level actor and the low-level critic is $(300,\,300)$. The output of the high-level actor is activated using the \texttt{tanh} function and scaled to fit the size of the environments.

For the adjacency network, we use a network consisting of 4 fully-connected layers with ReLU nonlinearities in all tasks. Each hidden layer of the adjacency network has the size of $(128,\,128)$. The dimension of the output embedding is 32.

We use Adam~\cite{kingma_adam:_2015} as the optimizer for all networks.

\subsection{Hyperparameters}
\label{appsubsec:hyperparam}

We list all hyperparameters we use in discrete control tasks and quadrupedal robot locomotion tasks respectively in Table~\ref{tab:param_disc} and Table~\ref{tab:param_cont}, and list the hyperparameters used for adjacency network training in Table~\ref{tab:param_adj}. Hyperparameters of the robot arm manipulation tasks are the same as~\cite{plappert_multi-goal_2018}, with a high-level action frequency of 5 and adjacency loss coefficient of 0.01. ``Ranges'' in the tables show the ranges of hyperparameters considered when we performed parameter tuning, and the hyperparameters without ranges were not specifically tuned.

\subsection{Evaluation Procedure}
We evaluate the performance of the agent every 5000 training steps by the average episodic reward over 5 independent evaluation episodes. The performance curves are generated using a rolling window with size 20; rewards in the rolling window are averaged.

\begin{table}[h]
    \centering
    \caption{Hyperparameters used in adjacency network training.}
    \begin{tabular}{lcc}
        \toprule
        \textbf{Hyperparameters} & \textbf{Values} & \textbf{Ranges} \\
        \midrule
        \midrule
        Adjacency Network & & \\
        \midrule
        Learning rate & \multicolumn{1}{c}{0.0002} & - \\
        Batch size & \multicolumn{1}{c}{64} & - \\
        $\epsilon_k$ & \multicolumn{1}{c}{1.0} & - \\
        $\delta$ & \multicolumn{1}{c}{0.2} & - \\
        Steps for pre-training & 50000 & - \\
        Pre-training epochs & 50 & - \\
        Online training frequency (steps) & 50000 & - \\
        Online training epochs & 25 & - \\
        \bottomrule
    \end{tabular}
    \label{tab:param_adj}
\end{table}

\begin{table}
    \centering
    \caption{Hyperparameters used in discrete control tasks. ``KC'' in the table refers to ``Key-Chest''.}
    \begin{tabular}{lcc}
        \toprule
        \textbf{Hyperparameters} & \textbf{Values} & \textbf{Ranges} \\
        \midrule
        \midrule
        High-level TD3 & & \\
        \midrule
        Actor learning rate & \multicolumn{1}{c}{0.0001}& - \\
        Critic learning rate & \multicolumn{1}{c}{0.001} & - \\
        \multirow{2}{*}{Replay buffer size} & 10000/20000 & \multirow{2}{*}{\{10000, 20000\}} \\
        & for Maze/KC & \\
        Batch size & \multicolumn{1}{c}{64} & - \\
        Soft update rate & \multicolumn{1}{c}{0.001} & - \\
        Policy update frequency & \multicolumn{1}{c}{2} & \{1, 2\} \\
        Discounting &\multicolumn{1}{c}{0.99} & - \\
        High-level action frequency & \multicolumn{1}{c}{10} & - \\
        Reward scaling & \multicolumn{1}{c}{1.0} & - \\
        \multirow{3}{*}{Exploration strategy} & Gaussian & \multirow{3}{*}{\{3.0, 5.0\}} \\
        & ($\sigma=3.0/5.0$ & \\
        & for Maze/KC) & \\
        Adjacency loss coefficient & \multicolumn{1}{c}{20} & \{1, 5, 10, 20\} \\
        \midrule
        \midrule
        Low-level A2C & & \\
        \midrule
        Actor learning rate & \multicolumn{1}{c}{0.0001} & - \\
        Critic learning rate & \multicolumn{1}{c}{0.0001} & - \\
        Entropy weight & \multicolumn{1}{c}{0.01} & - \\
        Discounting & \multicolumn{1}{c}{0.99} & - \\
        Reward scaling & \multicolumn{1}{c}{1.0} & - \\
        \bottomrule
    \end{tabular}
    \label{tab:param_disc}
\end{table}

\begin{table}
    \centering
    \caption{Hyperparameters used in continuous control tasks. ``AM'' and ``AP'' in the table refer to ``Ant Maze'' and ``Ant Push'' respectively.}
    \begin{tabular}{lcc}
        \toprule
        \textbf{Hyperparameters} & \textbf{Values} & \textbf{Ranges} \\
        \midrule
        \midrule
        High-level TD3 & & \\
        \midrule
        Actor learning rate & 0.0001 & - \\
        Critic learning rate & 0.001 & - \\
        Replay buffer size & 200000 & - \\
        Batch size & 128 & - \\
        Soft update rate & 0.005 & - \\
        Policy update frequency & 1 & \{1, 2\} \\
        Discounting & 0.99 & - \\
        High-level action frequency & 10 & - \\
        \multirow{2}{*}{Reward scaling} & 0.1/1.0 & \multirow{2}{*}{\{0.05, 0.1, 1.0\}} \\
        & for AM, AP/others & \\
        \multirow{3}{*}{Exploration strategy} & Gaussian & \multirow{3}{*}{\{1.0, 2.0, 5.0\}} \\
        & ($\sigma=5.0/1.0$ & \\
        & for AP/others) & \\
        Adjacency loss coefficient & 20 & \{1, 5, 10, 20\} \\
        \midrule
        \midrule
        Low-level TD3 & & \\
        \midrule
        Actor learning rate & 0.0001 & - \\
        Critic learning rate & 0.001 & - \\
        Replay buffer size & 200000 & - \\
        Batch size & 128 & - \\
        Soft update rate & 0.005 & - \\
        Policy update frequency & 1 & - \\
        Discounting & 0.95 & - \\
        Reward scaling & 1.0 & - \\
        Exploration strategy & Gaussian ($\sigma=1.0$) & - \\
        \bottomrule
    \end{tabular}
    \label{tab:param_cont}
\end{table}


\ifCLASSOPTIONcaptionsoff
  \newpage
\fi



%

\end{document}